\pdfoutput=1

\documentclass[twoside,11pt]{article}
\usepackage{jmlr2e}

\usepackage{times}
\usepackage{helvet}
\usepackage{courier}

\usepackage{ifthen}
\usepackage{amssymb}
\usepackage{amsmath}
\usepackage{ulem}
\usepackage{graphicx} % more modern
\usepackage{subfigure}

\usepackage{natbib}
\usepackage{algorithm}
\usepackage{algorithmic}

\usepackage{hyperref}

\usepackage{epstopdf}

\usepackage{verbatim}
\usepackage{float}

\usepackage{caption}
\ShortHeadings{Structure Learning in BNs of Moderate Size by Efficient Sampling}{He, Tian and Wu}

\firstpageno{1}

\begin{document}
%\bibliographystyle{../../aaai-named}
%\bibliographystyle{plain}

%\maketitle

%Put it at the end
%\bibliographystyle{icml2013}
%\bibliography{./csl_cites_HR}

%\twocolumn[
%\icmltitle{Bayesian Learning in Bayesian Networks of Moderate Size by Sampling}
%\icmltitle{Bayesian Learning for Bayesian Networks of Moderate Size}

% unless you've provided the [accepted] option to the icml2013
%\icmlauthor{Your Name}{email@yourdomain.edu}
%\icmladdress{Your Fantastic Institute,
%            314159 Pi St., Palo Alto, CA 94306 USA}
%\icmlauthor{Your CoAuthor's Name}{email@coauthordomain.edu}
%\icmladdress{Their Fantastic Institute,
%            27182 Exp St., Toronto, ON M6H 2T1 CANADA}

%\icmlkeywords{boring formatting information, machine learning, ICML}

%\vskip 0.3in
%]

%\title{Learning with Mixtures of Trees}
%\title{Structure Learning by Efficiently Sampling Bayesian Networks of Moderate Size }
\title{Structure Learning in Bayesian Networks of Moderate Size by Efficient Sampling}

%\author{\name Marina Meil\u{a} \email mmp@stat.washington.edu \\
%       \addr Department of Statistics\\
%       University of Washington\\
%       Seattle, WA 98195-4322, USA
%       \AND
%       \name Michael I.\ Jordan \email jordan@cs.berkeley.edu \\
%       \addr Division of Computer Science and Department of Statistics\\
%       University of California\\
%       Berkeley, CA 94720-1776, USA}

       %\author{ {\bf Ru He}, {\bf Jin Tian} and {\bf Huaiqing Wu}\\
%Department of Computer Science and Department of Statistics \\
%Iowa State University\\
%Ames, IA 50011\\
%\textit{\{rhe, jtian, isuhwu\}@iastate.edu}\\}

\author{
       \name Ru He      \email rhe@iastate.edu \\
       \addr Department of Computer Science and Department of Statistics\\
       Iowa State University\\
       Ames, IA 50011, USA
       \AND
       \name Jin Tian   \email jtian@iastate.edu \\
       \addr Department of Computer Science\\
       Iowa State University\\
       Ames, IA 50011, USA
       \AND
       \name Huaiqing Wu \email isuhwu@iastate.edu \\
       \addr Department of Statistics\\
       Iowa State University\\
       Ames, IA 50011, USA
       }

%\editor{Leslie Pack Kaelbling}
\editor{}

\maketitle

%\pdfoutput=1

\begin{abstract}
We study the Bayesian model averaging approach to learning Bayesian network structures
%which are directed acyclic graphs
(DAGs) from data.
We develop new algorithms including the first algorithm that is able to efficiently sample DAGs
according to the exact structure posterior.
%The network samples can then be used to efficiently estimate
%the posterior of any feature.
%Performance guarantee can be provided on the quality of the estimators, unlike the existing MCMC-based algorithms.
%We empirically show that our algorithms considerably outperform previous state-of-the-art methods.
The DAG samples can then be used to construct estimators for the posterior of any feature.
We theoretically prove good properties of our estimators and
empirically show that our estimators considerably outperform
the estimators from the previous state-of-the-art methods.
\end{abstract}


%\section{Introduction \label{sec-intro}}
%\section{INTRODUCTION \label{sec-intro}}
\section{Introduction \label{sec-intro}}
Bayesian networks are graphical representations of multivariate joint probability distributions
and have been widely used in
various data mining tasks for probabilistic inference and causal modeling \citep{pearl:2k,spirtes:etal2001}.
The core of a Bayesian network (BN) representation is its Bayesian network structure.
A Bayesian network structure is a DAG (directed acyclic graph)
whose nodes represent the random variables $X_1, X_2, \cdots, X_n$ in the problem domain
and whose edges correspond to the direct probabilistic dependencies.
Semantically, a Bayesian network structure $G$ encodes a set of conditional independence assumptions:
for each variable (node) $X_i$, $X_i$ is conditionally independent of its non-descendants given its parents.
With the above semantics,
%By explicitly encoding the conditional independence relations,
a Bayesian network structure
provides a compact representation for joint distributions and
supports efficient algorithms for answering probabilistic queries.
Furthermore,
with its semantics,
a Bayesian network structure can often provide a deep insight into the problem domain
and open the door to the cause-and-effect analysis.
%For example, in causal discovery,
%we are interested in the causal relations among variables,
%represented by the edges in the Bayesian network structure \cite{heckerman:etal99}.

%Learning the structures of Bayesian networks  from data has been an active research problem.
In the last two decades, there have been a large number of researches focusing
on the problem of learning Bayesian network structure(s) from the data.
These researches deal with a common real situation where
the underlying Bayesian network is typically unknown
so that it has to be learned from the observed data.
One motivation for the structure learning
%is that we are able to use the learned structure
is to use the learned structure for inference or decision making.
For example, we can use the learned model to predict or classify a new instance of data.
Another structure-learning motivation,
which is more closely related to the semantics of Bayesian network structures,
is for discovering the structure of the problem domain.
For example, in the context of biological expression data,
the discovery of the causal and dependence relation among different genes is often of primary interests.
With the semantics of a Bayesian network structure $G$,
the existence of an edge from node $X$ to node $Y$ in $G$ can be interpreted as the fact
that variable $X$ directly influences variable $Y$;
and the existence of a directed path from node $X$ to node $Y$ can be interpreted as the fact
that $X$ eventually influences $Y$.
Furthermore, under certain assumptions \citep{heckerman:etal99,spirtes:etal2001},
%the existence of an edge from node $X$ to node $Y$  shows that $X$ directly causes variable $Y$;
%and
the existence of a directed path from node $X$ to node $Y$
indicates that $X$ causes $Y$.
Thus, with the learned Bayesian network structure,
we can answer interesting questions
such as whether gene $X$
%causes
controls
gene $Y$ which in turn controls %causes
gene $Z$
by examining whether
there is a directed path from node $X$ via node $Y$ to node $Z$
in the learned structure.
Just as mentioned
%in \citep{friedman:kol03},
by \citet{friedman:kol03},
the extraction of these kinds of interesting structural features is
often the primary goal in the discovery task.

%As we have said, since a BN structure can potentially distinguish direct and indirect dependencies,
%it is able to give the users a deep understanding of domain knowledge.
%For example,
%in biology, the learned structure can help the biologist to
%reveal the underlying casual-end-effect relations between the expression levels of different genes.
%In these situations, the exaction of some structural features is often the primary goal in the discovery task.

There are several general approaches to learning BN structures.
One approach is to treat it as a model selection problem.
This approach defines a scoring criterion that measures
%how good a network is
how well a BN structure (DAG) fits the data
and finds the DAG (or a set of equivalent DAGs) with the optimal score
%\cite{silander:myl06}.
\citep{silander:myl06,Jaakkola:AISTATS2010,Yuan:IJCAI2011,Malone:AAAI2011,Malone:UAI2011,Cussens:UAI2011,Yuan:UAI2012,Malone:UAI2013,Cussens:UAI2013,Yuan:JAIR2013}.
(In Bayesian approach, the score of a DAG $G$ is simply the posterior $p(G|D)$ of $G$ given data $D$.)
However, when the data size is small relative to the number of variables,
the posterior $p(G|D)$ often gives significant
support to a number of DAGs,
and using a single maximum-a-posteriori (MAP) model could lead to unwarranted conclusions
%\cite{friedman:kol03}.
\citep{friedman:kol03}.
It is therefore desirable to use the Bayesian model averaging approach
by which the posterior probability of any feature of interest is computed
by averaging over all the possible DAGs
%\cite{heckerman:etal99}.
\citep{heckerman:etal99}.

Bayesian model averaging is, however, computationally challenging
%since
because
the number of possible network structures is
%at least $2^{n(n-1)/2}$,
between $2^{n(n-1)/2}$ and $n! 2^{n(n-1)/2}$,
super-exponential  in the number of variables $n$.
Tractable algorithms have been developed for special cases of averaging over trees
%\cite{meila:jaa06}
\citep{meila:jaa06}
and averaging over DAGs given a node ordering
%\cite{dash:coo04}.
\citep{dash:coo04}.
%Recently,
Since 2004,
dynamic programming (DP) algorithms have been developed for computing exact posterior
probabilities of structural features such as edges or subnetworks
%\cite{Koivisto:soo04,Koivisto:06,tian:he2009}.
\citep{Koivisto:soo04,Koivisto:06,tian:he2009}.
These algorithms have
exponential time and space complexity and are capable of handling
Bayesian networks of moderate size with up to around $25$ variables (mainly due to their space cost $O(n 2^n)$).
A big limitation of these
algorithms is that they can only compute posteriors of modular
features such as an edge but can not compute non-modular
features such as a path
(``is there a path from node $X$ to node $Y$''),
a combined path
(``is there a path from node $X$ via node $Y$ to node $Z$''
or ``is there a path from node $X$ to node $Y$ and no path from node $X$ to node $Z$''),
%or ``is there a path from node $X$ to node $Y$ and a path from node $X$ to node $Z$''),
or a limited-length path
(``is there a path of length at most 3 from node $X$ to node $Y$'').
Recently,
%Parviainen and Koivisto
%\cite{Parviainen:ECML2011}
%\yrcite{Parviainen:ECML2011}
%\citeyearpar{Parviainen:ECML2011}
\citet{Parviainen:ECML2011}
have developed a DP algorithm that can compute the exact posterior
probability of a path feature under a certain assumption.
(The assumption, called order-modular assumption,
will be discussed in details soon.)
This DP algorithm has (even higher) exponential time and space complexity and
can only handle a Bayesian network with fewer than $20$ variables (mainly due to its space cost $O(3^n)$).
Since this DP algorithm can only deal with a path feature,
all the other non-modular
features (such as a combined path)
which various users would be interested in
for their corresponding problems
still can not be computed by any DP algorithm proposed so far.
Note that generally
the posterior $p(f | D)$ of a combined feature $f = (f_1, f_2, \ldots, f_J)$
can not be obtained only from the posterior of each individual feature $p(f_j | D)$
($j \in \{1, 2, \ldots, J\}$)
because the independence among these features does not hold generally.
Actually,
by comparing $p(f_2 | f_1, D)$ with $p(f_2 | D)$,
a user can know  the effect of the feature $f_1$  upon the feature $f_2$;
%the conditional posterior $p(f_2 | f_1, D)$
%needs to be computed and then
%compared with $p(f_2)$.
but to obtain $p(f_2 | f_1, D)$ $(= p(f_1, f_2 | D) / p(f_1 | D))$,
the user typically needs to obtain $p(f_1, f_2 | D)$ first.
%some conditional posterior $p(f_2 | f_1, D)$,
%which is not necessary equal to $p(f_2 | D)$,
%generally he has to compute both $p(f_1 | D)$ and $p(f_1, f_2 | D)$ first
%(and then apply $p(f_2 | f_1, D) = $ $p(f_1, f_2 | D) / p(f_1 | D)$ ).
Another limitation of all these DP algorithms
is that it is very expensive for
%these algorithms
them
to perform data
prediction tasks. They can compute the exact posterior of a new
observational data case $p(x|D)$ but the algorithms have to be re-run for
each new data case $x$.

One solution %approach
to computing the posterior of an arbitrary non-modular feature is drawing
DAG
%(directed acyclic graph)
samples $\{G_1, \ldots, G_T\}$
%of
%networks
from the posterior $p(G|D)$,
which can then be used to approximate the full Bayesian model averaging
%by estimating
by estimating
the posterior of an arbitrary feature $f$ as $p(f|D) $ $ \approx \frac{1}{T}\sum_{i=1}^T f(G_i)$, or the posterior predictive distribution as $p(x|D) $ $ \approx \frac{1}{T}\sum_{i=1}^T p(x|G_i)$.
A number of algorithms have been developed for drawing sample
%networks
DAGs
using the bootstrap technique
%\cite{friedman:etal99}
\citep{friedman:etal99}
or  the Markov Chain Monte Carlo (MCMC) techniques
%\cite{madigan:yor95,friedman:kol03,eaton:mur07b,Marco:ML08,niinimaki:etal11}.
\citep{madigan:yor95,friedman:kol03,eaton:mur07b,Marco:ML08,niinimaki:etal11,Niinimaki:IJCAI2013}.
%Madigan and York
%\cite{madigan:yor95} %(1995)
%\citeyear{madigan:yor95}
%\citeyearpar{madigan:yor95}
\citet{madigan:yor95}
developed the \emph{Structure MCMC} algorithm that
uses the Metropolis-Hastings algorithm in the space of DAGs.
%Friedman and Koller
%\cite{friedman:kol03} %(2003) \nocite{friedman:kol03}
%\citeyear{friedman:kol03}
%\citeyearpar{friedman:kol03}
\citet{friedman:kol03}
developed the \emph{Order MCMC} procedure that operates
in the space of orders.  %node  orderings. %HR: order is better word than ordering. See \cite{friedman:kol03}.
The Order MCMC was shown to be able to considerably improve over the Structure MCMC the mixing and convergence of the Markov chain and  to outperform the bootstrap approach of
%\cite{friedman:etal99}
%\citep{friedman:etal99}
\citet{friedman:etal99}
as well.
%Eaton and Murphy
%\cite{eaton:mur07b}  %(2007) \nocite{eaton:mur07b}
%\citeyear{eaton:mur07b}
%\citeyearpar{eaton:mur07b}
\citet{eaton:mur07b}
developed the \emph{Hybrid MCMC} method (i.e., DP+MCMC method) that
first runs the DP algorithm of
%\cite{Koivisto:06}
\citet{Koivisto:06}
to develop a global proposal distribution and then runs the MCMC phase in the
DAG space.
Their experiments showed that the Hybrid MCMC
converged faster than both the Structure MCMC and the Order MCMC
so that the Hybrid MCMC resulted in more accurate structure learning performance.
An improved MCMC algorithm (often denoted as \emph{REV-MCMC}) traversing in the DAG space
with the addition of a new edge reversal move
was developed by
%\citep{Marco:ML08}
\citet{Marco:ML08}
%\cite{grzegorczyk:hus08} %no citation found in Jin
and was shown to be superior to the Structure MCMC and nearly
as efficient as the Order MCMC in the mixing and convergence.
Recently,
%Niinimaki et al.
%\cite{niinimaki:etal11}
%\citeyear{niinimaki:etal11}
%\citeyearpar{niinimaki:etal11}
\citet{niinimaki:etal11}
have proposed the \emph{Partial Order MCMC} method %,
%an algorithm that
which
operates in the space of partial orders.
The Partial Order MCMC includes the Order MCMC as its special case (by setting the parameter bucket size $b$ to be 1)
and has been shown to be superior to the Order MCMC in terms of the mixing and the structural learning performance
when a more appropriate bucket size $b > 1$ is set.
One common drawback of these MCMC algorithms is that there is no guarantee on the quality of the
approximation in finite runs. The approach to approximating full Bayesian model
averaging using the \emph{$K$-best} Bayesian network structures was studied
by
%\cite{tian:he2010}  %\cite{tian:etal10} %no citation found in Jin
%\citep{tian:he2010}
\citet{tian:he2010}
and was shown to be
%at least as good as Hybrid MCMC.
competitive with the Hybrid MCMC.
%no worse than Hybrid MCMC in structural learning.

%One common drawback of the algorithms that work in the order space,
%including the exact algorithms in \cite{Koivisto:soo04,Koivisto:06}
%and the approximate algorithms Order MCMC and Partial Order MCMC,
%is that they  require a special form of the structure prior, termed as \emph{order-modular} prior \cite{friedman:kol03,Koivisto:soo04}.
%As a consequence, the resulting prior  $p(G)$ cannot represent some desirable priors, such as ones that are uniform, and the computed posterior probabilities are biased as DAGs that have a larger number of topological orderings could have a larger prior probability. One method for correcting this bias was proposed by Ellis and Wong  \cite{ellis:won08}. %(2008) %\nocite{ellis:won08}  \cite{ellis:won08}

Several of these state-of-the-art algorithms work in the order space,
including the exact algorithms
%in
%\cite{Koivisto:soo04,Koivisto:06,Parviainen:ECML2011}
\citep{Koivisto:soo04,Koivisto:06,Parviainen:ECML2011}
and the approximate algorithms:
the Order MCMC
\citep{friedman:kol03}
and the Partial Order MCMC
\citep{niinimaki:etal11}.
%in \citep{friedman:kol03,niinimaki:etal11}.
They all assume a special form of the structure prior,
termed as \emph{order-modular} prior
%\cite{friedman:kol03,Koivisto:soo04},
\citep{friedman:kol03,Koivisto:soo04},
for computational convenience.
(Please refer to the beginning of Section \ref{sec-subsec-DPs} for the definition of the order-modular prior.)
%With such an assumption, the order-based computation-saving strategies can be %correspondingly developed
%applied so that these algorithms will typically have computational advantages over their counterparts that do not have this assumption.
However, the assumption of order-modular prior has the consequence that  %  also has its drawback:
the
%resulting
corresponding
prior $p(G)$ cannot represent some desirable priors %,
%such as ones that are uniform over the DAG space;
such as a uniform prior over the DAG space;
and the computed posterior probabilities are biased
%since DAGs that have a larger number of topological orders
since a DAG that has a larger number of topological orders
%could have
will be assigned
a larger prior probability.
%Whether the bias from the order-modular prior can be tolerated or not
Whether a computed posterior with the bias from the order-modular prior is inferior to its counterpart without such a bias
depends on the application scenario
and is beyond the scope of this paper.
For the detailed discussion about this issue,
please see the related papers
%can be found in
%\cite{friedman:kol03,Marco:ML08,Parviainen:ECML2011}.
%\citep{friedman:kol03,Marco:ML08,Parviainen:ECML2011}.
%http://jmlr.org/format/natbib.pdf
%\citealt{friedman:kol03,Marco:ML08,Parviainen:ECML2011}.
\citep{friedman:kol03,Marco:ML08,Parviainen:ECML2011}.
%Whether the bias from the order-modular prior can be allowed/tolerated or not depends on the application scenario
%and is beyond the scope of this paper.
%Detailed discussions about this issue can be found in \cite{friedman:kol03,Marco:ML08}.
%One method that helps Order MCMC \cite{friedman:kol03} to correct this bias was proposed by Ellis and Wong
%\cite{ellis:won08}. %(2008) %\nocite{ellis:won08}  \cite{ellis:won08}
%\citeyear{ellis:won08}. %when this bias is not allowed.
%To allow arbitrary priors over DAGs, Ellis and Wong (2008) \nocite{ellis:won08} propose to run Order MCMC first but followed by a correction step.
One method that helps the Order MCMC
%\cite{friedman:kol03}
\citep{friedman:kol03}
to correct this bias was proposed by
%Ellis and Wong
%\cite{ellis:won08}. %(2008) %\nocite{ellis:won08}  \cite{ellis:won08}
%\citeyear{ellis:won08}. %when this bias is not allowed.
%\citeyearpar{ellis:won08}
\citet{ellis:won08}.

In this paper,
first we develop a new algorithm that
uses
the results of the DP algorithm of
\citet{Koivisto:soo04}
to
efficiently sample orders according to the \emph{exact} order posterior
under the assumption of order-modular prior.
Next, we develop a time-saving strategy for the process of sampling DAGs consistent with given orders.
(Such a DAG-sampling process is
based on sampling parents for each node
as described by
\citet{friedman:kol03}
by assuming a bounded node in-degree.)
The resulting algorithm (called DDS)  is the first algorithm  that is
able to sample DAGs according to the \emph{exact} DAG posterior with the same order-modular prior assumption.
We empirically show that our DDS algorithm is both considerably more accurate and considerably more efficient than the Order MCMC and the
Partial Order MCMC
when $n$ is moderate so that our DDS algorithm is applicable.
%since a DAG consistent with a given order can be %efficiently
%sampled as described by
%\citet{friedman:kol03}
%(with the assumed maximum node-indegree),
%our order sampling algorithm
%leads to the first algorithm (called DDS)
%that
%can be proved to
%be able to
%sample DAGs according to the \emph{exact} DAG posterior with the same order-modular prior assumption.
%With our time-saving strategy for sampling DAGs,
%our DDS algorithm
%is shown to be both considerably more accurate and considerably more efficient than the Order MCMC and the Partial Order MCMC
%when $n$ is moderate so that our DDS algorithm is applicable.
Moreover,
the estimator based on our DDS algorithm has several desirable properties;
for example,
unlike the existing MCMC algorithms,
the quality of our estimator can be guaranteed by controlling the number of DAGs sampled by our DDS algorithm.
% for which it is difficult to estimate the quality of approximation,
%we can provide performance guarantee on the quality of the estimate from our DDS algorithm with the given number of DAG samples
 The main application of our DDS algorithm is to address the limitation
 of the exact DP algorithms
 %\cite{Koivisto:soo04,Koivisto:06,Parviainen:ECML2011}
 \citep{Koivisto:soo04,Koivisto:06,Parviainen:ECML2011}
 (whose usage is restricted to modular features or path features) %for moderate $n$)
 in order to
 estimate
 the posteriors of
%arbitrary non-modular
%features such as combined paths or limited-length paths,
various \emph{non-modular} features
arbitrarily specified by users.
Additionally our DDS algorithm can also be used
to efficiently perform data
prediction tasks in
%computing
estimating
$p(x|D)$ for a large number of data cases (while the exact DP algorithm has to be re-run for each case $x$).
%when the bias from the order-modular prior is allowed/tolerated.
%We also develop an estimation strategy that can directly use the DAGs sampled from our sampling algorithm but can efficiently correct the bias stemmed from the used order-modular prior.
%Finally, we develop an algorithm (called IW-DDS) that allows arbitrary structure priors $p(G)$ by adding a correction step to the DDS algorithm (to correct for the  use of the order-modular prior)
%using an idea refined from the one in
%\cite{ellis:won08}.
%\citeyear{ellis:won08}.
%We show that our algorithm is much more efficient than the algorithms with similar computational complexity $K$-best and Hybrid MCMC.
%Similar computational complexity is not true for k-best.
Finally, we develop an algorithm (called IW-DDS)
to correct the bias (due to the order-modular prior) in the DDS algorithm
%by addressing the drawbacks of
%%using an idea refined from
%the strategy of
by extending the idea of
\citet{ellis:won08}.
%We show that our algorithm is much more efficient than the algorithms with similar computational complexity $K$-best and Hybrid MCMC.
%Similar computational complexity is not true for k-best.
We theoretically prove that the estimator based on our IW-DDS
%is consistent and asymptotically unbiased;
has several desirable properties;
then we
empirically show that our estimator is
superior to the estimators based on the Hybrid MCMC method
%\cite{eaton:mur07b}
\citep{eaton:mur07b}
and the $K$-best algorithm
%\cite{tian:he2010},
\citep{tian:he2010},
%the state-of-the-art algorithms  that also draw network samples to estimate the posteriors of any features  %$K$-best does not draw samples
two state-of-the-art algorithms that can estimate the posterior of any feature
without
%the bias from
the order-modular prior assumption.
Analogously,
%the main application of our IW-DDS is to address the the limitation
%of the exact DP algorithm of \citet{tian:he2009} (whose usage is restricted to modular features)
%in order to estimate the posteriors of arbitrary \emph{non-modular} features,
%and additionally our IW-DDS can be used to efficiently perform data prediction tasks.
our IW-DDS algorithm
mainly addresses the limitation of the exact DP algorithm of
 \citet{tian:he2009} (whose usage is restricted to modular features)
 in order to estimate the posteriors of arbitrary \emph{non-modular} features
 and can additionally be used to efficiently perform data prediction tasks
 when an application situation prefers to avoid the bias from order-modular prior.

The rest of the paper is organized as follows.
In Section~\ref{sec-back} we briefly review the Bayesian approach to learning Bayesian networks from data,
the related DP algorithms %\cite{Koivisto:soo04,Koivisto:06}
\citep{Koivisto:soo04,Koivisto:06}
and the Order MCMC algorithm
%\cite{friedman:kol03}.
\citep{friedman:kol03}.
In Section~\ref{sec-sam}
%we present our algorithms for drawing order and DAG samples, and our strategy in sampling DAGs for correcting the bias from the order-modular prior.
we present our order sampling algorithm, DDS algorithm, and IW-DDS algorithm;
and prove good properties of the estimators based on our algorithms.
We empirically demonstrate the advantages of our algorithms in Section~\ref{sec-exp}.
Section~\ref{sec-con} concludes the paper.
Finally,
Appendix~\ref{sec-appendix-proofs} provides
the proofs of
all the conclusions
including the propositions, theorems and corollary referenced in the
paper.
%Finally,
%Appendix~\ref{sec-appendix-hist-DDS} and Appendix~\ref{sec-appendix-hist-IWDDS} show the supplementary experimental results.

%\section{Bayesian Learning of Bayesian Networks \label{sec-back}}
%\section{BAYESIAN LEARNING OF BAYESIAN NETWORKS \label{sec-back}}
\section{Bayesian Learning of Bayesian Networks \label{sec-back}}
A Bayesian network is a DAG
%(directed acyclic graph)
$G$ that encodes a joint probability distribution over a set $X= \{X_1,\ldots,X_n\}$ of random variables with each node of the DAG representing a variable in $X$. %We will focus on discrete random variables and assume that each variable $X_i$ can take values from a finite domain, $Dm(X_i) = \{x_{i1}, \ldots, x_{ir_i}\}$, where $r_i$ is the number of states of $X_i$.
For convenience we %will
typically work on the index set $V=\{1, \ldots, n\}$ and represent a variable $X_i$ by its index $i$.
We use $X_{Pa_i}\subseteq X$ to represent the parent set of $X_i$ in a DAG $G$
and use $Pa_i \subseteq V$ to represent the corresponding index set.
(A pair $(i, Pa_i)$ is often called a family.)
Thus, a DAG $G$ can be represented as a vector $(Pa_1, \ldots, Pa_n)$. %add to let user to understand the def: G is consistent with an order.
% $Dm(Pa_i)$ to represent the set of states of $Pa_i$.

Assume that we are given a training data set $D=\{x^1, x^2, \ldots, x^m\}$, where each $x^i$ is a particular instantiation over the set of variables $X$. We only consider situations where the data are complete, that is, every variable in $X$ is assigned a value.
In the Bayesian approach to learning Bayesian networks from the training data $D$,
we compute the posterior probability of a DAG $G$ as
\begin{equation*}
%p_{\nprec}(G|D) = \frac{p(D|G)p(G)}{p_{\nprec}(D)} = \frac{p(D|G)p(G)}{\sum_{G} p(D|G)p(G) }.
p(G|D) = \frac{p(D|G)p(G)}{p(D)} = \frac{p(D|G)p(G)}{\sum_{G} p(D|G)p(G) }.
\end{equation*}
Assuming global and local parameter independence, and parameter modularity,
$p(D|G)$ can be decomposed into a product of local marginal likelihoods (often called local scores) as
%\cite{cooper:her92, heckerman:etal95a}
\citep{cooper:her92, heckerman:etal95a}
\begin{align}\label{eq-p_D_given_G}
%p(D|G) = \prod_{i=1}^n p(x_i|x_{Pa_i}:D)\equiv \prod_{i=1}^n score_i(Pa_i:D), % \equiv means "always =" in maths so that := is used instead.
%p(D|G) = \prod_{i=1}^n p(x_i|x_{Pa_i}:D) :=  \prod_{i=1}^n score_i(Pa_i:D),
p(D|G) = \prod_{i=1}^n p(X_i|X_{Pa_i}:D) :=  \prod_{i=1}^n score_i(Pa_i:D),
\end{align}
where, with appropriate parameter priors,
$score_i(Pa_i:D)$
(the local score for a family $(i, Pa_i)$)
has a closed form solution. In this paper we will assume that these local scores can be computed efficiently from data. The standard assumption for structure prior $p(G)$ is %
%that it is \emph{modular}
%structure modularity \cite{friedman:kol03}:
\emph{structure-modular} prior
%\cite{friedman:kol03}:
\citep{friedman:kol03}:
\begin{align}\label{eq-smodularity-1}
p(G) =  \prod_{i=1}^n p_i(Pa_i),
\end{align}
where $p_i$ is some nonnegative function over the subsets of $V - \{i\}$.

Combing Eq.~(\ref{eq-p_D_given_G}) and Eq.~(\ref{eq-smodularity-1}), we have
\begin{align}\label{eq-p_G_D_struc}
p_{\nprec}(G, D) = p(D|G)p(G) =  \prod_{i=1}^n score_i(Pa_i:D) p_i(Pa_i).
\end{align}
Note that the subscript $\nprec$ is intentionally added by us to mean
that the corresponding probability is the one obtained
without order-modular prior assumption.
%This notation is used to distinguish the probability computed with order-modular prior assumption,
This is different from the probability computed with order-modular prior assumption,
which will be marked by the subscript $\prec$ for the distinction.

%where, if further assuming Dirichlet parameter priors, we have
%\begin{align}
%score_i(Pa_i:D)=\prod_{pa_i\in Dm(Pa_i)}\frac{\Gamma(\alpha_{pa_i})}{\Gamma(\alpha_{pa_i}+N_{pa_i})} \prod_{x_i\in Dm(X_i)}\frac{\Gamma(\alpha_{x_i;pa_i}+N_{x_i,pa_i})}{\Gamma(\alpha_{x_i;pa_i})},
%\end{align}
%where $\Gamma(\cdot)$ is the Gamma function, $N_{x_i, pa_i}$ is the number of cases in data set $D$ for which $x_i$ takes the value $x_i$ and its parents $Pa_i$ takes the value $pa_i$,  $\alpha_{x_i;pa_i}$ are Dirichlet hyper parameters,  and
%$\alpha_{pa_i} = \sum_{x_i} \alpha_{x_i;pa_i}$, $N_{pa_i} = \sum_{x_i} N_{x_i,pa_i}$.
%We use $\prod_{v_i}$ as a shorthand for $\prod_{v_i \in Dm(V_i)}$ and $\prod_{pa_i}$ for $\prod_{pa_i \in Dm(Pa_i)}$.

We can compute the posterior probability of any hypothesis of interest
by averaging over all the possible DAGs.
For example, we are often interested in computing the posteriors of structural features.
Let $f$ be a structural feature represented by an indicator function such that $f(G)$ is 1 if the feature is present in $G$ and 0 otherwise.
By the full Bayesian model averaging, we have the posterior of %feature
$f$ as
\begin{align}
%p_{\nprec }(f|D) = \sum_G f(G) p_{\nprec }(G|D), \label{eq-HR-pre1}
p(f|D) = \sum_G f(G) p(G|D). \label{eq-HR-pre1}
\end{align}

%Here the subscript $\nprec | \prec$ means that the corresponding posterior probability is obtained
%either without the order-modular prior assumption
%or with the assumption:
Note that
$p_{\nprec}(f|D)$ will be obtained if $p(G|D)$ in Eq.~(\ref{eq-HR-pre1}) is $p_{\nprec }(G|D)$;
$p_{\prec}(f|D) $ will be obtained if $p(G|D)$ in Eq.~(\ref{eq-HR-pre1}) is $p_{ \prec}(G|D)$.
This difference is the key to understanding the bias issue which will be described in details later.

Since summing over all the possible DAGs is generally infeasible for any problem with $n > 6$
using a contemporary computer,
one approach to computing the posterior of $f$ is to draw DAG samples $\{G_1, \ldots, G_T\}$
from the posterior
%$p_{\nprec | \prec}(G|D)$,
$p_{\nprec }(G|D)$ or $p_{\prec}(G|D)$,
which can then be used to estimate the posterior
%$p_{\nprec | \prec}(f|D)$ as
$p_{\nprec }(f|D)$ or $p_{\prec}(f|D)$ as
\begin{align} \label{eq-2-HR}
%\hat{p}_{\nprec | \prec}(f|D)=  \frac{1}{T}\sum_{i=1}^T p(f|G_i,D) = \frac{1}{T}\sum_{i=1}^T f(G_i).
%\hat{p}_{\nprec | \prec}(f|D)=  \frac{1}{T}\sum_{i=1}^T f(G_i).
\hat{p}(f|D)=  \frac{1}{T}\sum_{i=1}^T f(G_i).
\end{align}

%\subsection{The DP Algorithms}
%\subsection{THE DP ALGORITHMS}
\subsection{The DP Algorithms}
\label{sec-subsec-DPs}
The DP algorithms
%\cite{Koivisto:soo04,Koivisto:06}
\citep{Koivisto:soo04,Koivisto:06}
%(and order MCMC)
%Some state-of-the-art algorithms
work in the order space rather than the DAG space.
We define an \emph{order} $\prec$ of variables
as a total order (a linear order) on $V$ represented as a vector $(U_1, \ldots, U_n)$,
where $U_i$ is the set of predecessors of $i$ in the order $\prec$.
To be more clear we may use $U^{\prec}_i$.
%add the def: G is consistent with an order
We say that a DAG $G = (Pa_1, \ldots, Pa_n)$ is consistent with an order $(U_1, \ldots, U_n)$, denoted by $G \subseteq \prec$,
 if $Pa_i \subseteq U_i$ for each $i$.
% Let $\mathcal{G}$ denote the set of all DAGs over $V$. %$\mathcal{G}$ is never used and $\mathcal{G}$ is redefined in IW-DDS.
%Let $\mathcal{G}_{\prec}$ denote the set of DAGs consistent with $\prec$.
%notation $G \subseteq \prec$ is enough to replace $\mathcal{G}_{\prec}$.
If $S$ is a subset of $V$, we let $\mathcal{L}(S)$ denote the set of linear orders on $S$.
 In the following we will largely follow the notation from
%Koivisto
%\cite{Koivisto:06}.
%\citeyear{Koivisto:06}.
%\citeyearpar{Koivisto:06}.
\citet{Koivisto:06}.

The algorithms working in the order space assume \emph{order-modular} prior defined as follows:
if $G$ is consistent with $\prec$, then
\begin{align}\label{eq-smodularity}
p(\prec, G) =  \prod_{i=1}^n q_i(U_i)\rho_i(Pa_i),
\end{align}
where each $q_i$ and $\rho_i$ is some function from the subsets of $V - \{i\}$ to the nonnegative real numbers.
(If $G$ is not consistent with $\prec$, then $p(\prec, G) = 0$.)

%The DP algorithms can compute the exact posteriors $p(f|D)$ for \emph{modular features} defined as:
A \emph{modular feature} is defined as:
\begin{align*}
f(G) = \prod_{i=1}^n f_i(Pa_i),
\end{align*}
where $f_i(Pa_i)$ is an indicator function returning a $0/1$ value.
For example, an edge feature $j\rightarrow i$ can be represented by setting $f_i(Pa_i)=1$ if and only if $j\in Pa_i$ and setting $f_l(Pa_l)=1$ for all $l\neq i$.
%while a path feature (``is there a path between node $X$ and node $Y$'') is not modular.

With the order-modular prior, we are interested in the posterior $p_{\prec}(f|D) = p_{\prec}(f,D) / p_{\prec}(D)$.
$p_{\prec}(f|D)$ can be obtained if the joint probability $p_{\prec}(f, D)$ can be computed
(since $p_{\prec}(D)$
%$ = p_{\prec}(f=1,D)$
$ = p_{\prec}(f \equiv 1,D)$
%can be computed by setting each $f_i(Pa_i)$ to be the constant 1).
where $f \equiv 1$, meaning that $f$ always equals 1, can be easily achieved by setting each $f_i(Pa_i)$ to be the constant 1).
Koivisto and Sood (2004) show that
\begin{align}\label{eq-p_f_prec_D}
p(f, \prec,D)
& =  \prod_{i=1}^n \alpha_i(U^{\prec}_i),
\end{align}
and
\begin{align} \label{eq-p_prec(f, D)}
p_{\prec}(f, D)
& =  \sum_{\prec} \prod_{i=1}^n \alpha_i(U^{\prec}_i),
\end{align}
where the function $\alpha_i$ is defined for each $i\in V$ and each $S\subseteq V-\{i\}$ as
\begin{align*}%\label{eq-ai}
%\alpha_i(S) \equiv q_i(S)\sum_{Pa_i\subseteq S} \beta_i(Pa_i),
\alpha_i(S)  =  q_i(S)\sum_{Pa_i\subseteq S} \beta_i(Pa_i),
\end{align*}
in which
the function $\beta_i$ is defined for each $i\in V$ and each $Pa_i \subseteq V-\{i\}$ as
\begin{align*}
%\beta_i(Pa_i)\equiv f_i(Pa_i)\rho_i(Pa_i)score_i(Pa_i:D).
\beta_i(Pa_i) =  f_i(Pa_i)\rho_i(Pa_i)score_i(Pa_i:D).
\end{align*}
Accordingly, the DP algorithm of \citet{Koivisto:soo04} consists of the following three steps.
The first step computes $\beta_i(Pa_i)$ for each $i\in V$ and each $Pa_i \subseteq V-\{i\}$.
The time complexity of this step is $O(n^{k+1}C(m))$ under the assumption of the maximum in-degree $k$,
where
$n$ is the number of variables,
and $C(m)$ is the cost of computing a single local marginal %conditional
likelihood $score_i(Pa_i:D)$ for $m$ data instances.
The second step computes $\alpha_i(S)$ for each $i\in V$ and each $S\subseteq V-\{i\}$.
With the assumed maximum in-degree $k$, this step takes $O(k n 2^n)$ time
by using the truncated M\"{o}bius transform technique \citep{Koivisto:soo04}
which is extended from the standard fast M\"{o}bius transform algorithm \citep{kennes:sme90}.
The third step computes $p_{\prec}(f, D)$
%Now the summation over the order space is computed %recursively
by defining the following function (called forward contribution) for each $S\subseteq V$:
\begin{align}\label{eq-L_S}
L(S) = \sum_{\prec\in\mathcal{L}(S)} \prod_{i\in S} \alpha_i(U^{\prec}_i),
\end{align}
where $U^{\prec}_i$ is the set of variables in $S$ ahead of $i$ in the order $\prec\in\mathcal{L}(S)$.
It can be shown that for every $S\subseteq V$ the corresponding $L(S)$ can be computed recursively
using the DP technique according to the following equation
%\cite{Koivisto:soo04,Koivisto:06}:
\citep{Koivisto:soo04,Koivisto:06}:
\begin{align}\label{eq-L_S_recursive}
L(S) = \sum_{i\in S} \alpha_i(S-\{i\})L(S-\{i\}),
\end{align}
starting with $L(\emptyset)=1$ and ending with $L(V)$.
From Eq.~(\ref{eq-p_prec(f, D)}) and Eq.~(\ref{eq-L_S}),
we have
\begin{align}\label{eq-L_V}
p_{\prec}(f, D) = L(V).
\end{align}
The third step takes $O(n 2^n)$ time when $L(V)$ is computed using the above DP technique.
%Therefore,
In summary, the whole DP algorithm of
%\cite{Koivisto:soo04}
\citet{Koivisto:soo04}
can compute the posterior of any modular feature (such as an edge feature)
in $O(n^{k+1}C(m) + kn2^n)$ time and $O(n 2^n)$ space.

As the extended work of
%\citep{Koivisto:soo04},
\citet{Koivisto:soo04},
%Koivisto
%\citeyearpar{Koivisto:06}
\citet{Koivisto:06}
includes the DP algorithm of
\citet{Koivisto:soo04}
as its first three steps and appends two additional steps
so that
all the $n(n-1)$ edges
can be computed in $O(n^{k+1}C(m) + kn2^n)$ time and $O(n 2^n)$ space.
The foundation of the two additional steps is
the introduction of the following function (called backward contribution) for each $T \subseteq V$:
\begin{align}\label{eq-backward-contr}
%L(S) = \sum_{\prec  \in \mathcal{L}(S)} \prod_{i\in S} \alpha_i( U^{\prec}_i ),
 R(T) = \sum_{\prec' \in \mathcal{L}(T)} \prod_{i\in T} \alpha_i( (V-T) \cup  U^{\prec'}_i ).
\end{align}
Like $L(S)$,
$R(T)$ can also be computed recursively using some DP technique. % similar to the DP technique used for $L(S)$.
Please refer to
%\citep{Koivisto:06}
the paper of
\citet{Koivisto:06}
for further details of the two additional steps.

While the introduced DP algorithms
\citep{Koivisto:soo04,Koivisto:06}
make significant contributions to the structure learning of Bayesian networks,
they have one fundamental limitation:
they can only compute the posteriors of modular features.
In the next section, we will show how to use the results of the DP algorithm of
%\cite{Koivisto:soo04}
\citet{Koivisto:soo04}
to efficiently draw DAG samples,
which can then be used to compute the posteriors of arbitrary features.

%\subsection{Order MCMC}
%\subsection{ORDER MCMC}
\subsection{Order MCMC}
The idea of the Order MCMC is
to use the Metropolis-Hastings algorithm to draw order samples $\{\prec_1, \ldots, \prec_{N_{o}}\}$
that have $p(\prec|D)$ as the invariant distribution,
where $N_o$ is the number of sampled orders.
For this purpose we need to be able to compute $p(\prec,D)$,
which can be obtained from Eq.~(\ref{eq-p_f_prec_D}) by setting
%$f=1$.
$f \equiv 1$.
Let $\beta'_i(Pa_i)$ denote $\beta_i(Pa_i)$ resulted from setting each $f_i(Pa_i)$ to be the constant 1.
%We define $\alpha'_i(S)$ and $L'(S)$ similarly.
Similarly,
we define $\alpha'_i(S)$ and $L'(S)$ as the special cases of $\alpha_i(S)$ and $L(S)$
respectively by setting each $f_i(Pa_i)$ to be the constant 1.
Then from Eq.~(\ref{eq-p_f_prec_D}) and (\ref{eq-L_V}) we have
\begin{align}\label{eq-p_prec_D}
p(\prec,D)
& =  \prod_{i=1}^n \alpha'_i(U^{\prec}_i),
\end{align}
and
\begin{align}\label{eq-120330237}
p_{\prec}(D) = L'(V).
\end{align}
%emphasize Order MCMC can only approximate (invariant distribution) $p(\prec|D)$ but can not get exact $p(\prec|D)$.
%So our algorithm is more accurate

The Order MCMC can estimate the posterior of a modular feature as
\begin{align} \label{eq-3-HR}
\hat{p}_{\prec}(f|D) = \frac{1}{N_{o}}\sum_{i=1}^{N_{o}} p(f|\prec_i,D).
\end{align}
For example, from Propositions 3.1 and 3.2
stated by
%\cite{friedman:kol03}
%\citep{friedman:kol03}
\citet{friedman:kol03}
as well as the definitions of $\beta'_i$ and $\alpha'_i$,
the posterior of a particular choice of parent set $Pa_i \subseteq U^{\prec}_i$ for node $i$ given an order %$\prec$
is
\begin{align}\label{eq-120330246}
p((i, Pa_i)|\prec,D) = \frac{\beta'_i(Pa_i)}{\alpha'_i(U^{\prec}_i)/q_i(U^{\prec}_i)},
\end{align}
%and the posterior of the edge feature $j\rightarrow i$ is given by (see Proposition 3.2 in \cite{friedman:kol03})
and the posterior of the edge feature $j\rightarrow i$ given an order %$\prec$
is
\begin{align} \label{eq-1-HR}
p(j\rightarrow i|\prec,D)= 1- \frac{\alpha'_i(U^{\prec}_i-\{j\})/q_i(U^{\prec}_i-\{j\})}{\alpha'_i(U^{\prec}_i)/q_i(U^{\prec}_i)}.
\end{align}

In order to compute arbitrary non-modular features, we further draw DAG samples after drawing $N_o$ order samples.
Given an order, a DAG can be sampled by drawing parents for each node according to Eq.~(\ref{eq-120330246}).
%Assuming that $N_d$ DAG samples have been drawn from each sampled order, we have $T (= N_o \cdot N_d)$ DAG samples so that
Given DAG samples $\{G_1, \ldots, G_T\}$, we can then estimate any feature posterior $p_{\prec}(f | D)$ using $\hat{p}_{\prec}(f | D)$
shown in  Eq.~(\ref{eq-2-HR}).

%\section{Sampling Network Structures\label{sec-sam}}
%\section{Order Sampling Algorithm and DAG Sampling Algorithms \label{sec-sam}}
%\section{ORDER SAMPLING ALGORITHM AND DAG SAMPLING ALGORITHMS \label{sec-sam}}
\section{Order Sampling Algorithm and DAG Sampling Algorithms \label{sec-sam}}

In this section we present our order sampling algorithm, DDS algorithm and IW-DDS algorithm.
We also prove good properties of the estimators based on our algorithms.

%\subsection{Order Sampling Algorithm \label{sec-order-samp-algo}}
%\subsection{ORDER SAMPLING ALGORITHM \label{sec-order-samp-algo}}
\subsection{Order Sampling Algorithm \label{sec-order-samp-algo}}
In this subsection,
we show that using the results including $\alpha'_i(S)$ (for each $i \in V$ and each $S\subseteq V-\{i\}$) and $L'(S)$
%(for each $S\subseteq V-\{i\}$)
(for each $S\subseteq V$)
computed from the DP algorithm of
%\cite{Koivisto:soo04},
\citet{Koivisto:soo04},
we can draw an order sample efficiently by drawing each element in the order one by one.
Let an order $\prec$ be represented as ($\sigma_1, \ldots, \sigma_n$) %($\sigma_1 < \ldots < \sigma_n$)
where $\sigma_i$ is the $i$th element in the order.
\begin{proposition}\label{prop-1}
The conditional probability that the $k$th ($1 \leq k \leq n$) element in the order is $\sigma_k$
given that the $n - k$ elements after it along the order are $\sigma_{k+1}, \ldots, \sigma_n$ respectively is as follows:
\begin{align}\label{eq-120330240}
p(\sigma_k| \sigma_{k+1}, \ldots, \sigma_n, D) = \frac{L'(U^{\prec}_{\sigma_k})\alpha'_{\sigma_k}(U^{\prec}_{\sigma_k})}{L'(U^{\prec}_{\sigma_{k+1}})},
\end{align}
where $\sigma_k \in V - \{\sigma_{k+1}, \ldots, \sigma_n \} $,
and $U^{\prec}_{\sigma_i}$
$= V - \{\sigma_i, \sigma_{i+1}, \ldots, \sigma_n\}$
so that $U^{\prec}_{\sigma_i}$ denotes the set of predecessors of $\sigma_i$ in the order $\prec$.
\\ Specifically for $k=n$, we essentially have
\begin{align}\label{eq-120330240-init}
p(\sigma_n=i|D)= \frac{L'(V-\{i\})\alpha'_i(V-\{i\})}{L'(V)},
\end{align}
where $i \in V$.
\end{proposition}

%Note that the proofs of Propositions~\ref{prop-1}, \ref{prop-2} and Theorem~\ref{thm-DDS} are
%given in the supplementary material due to the space limit.
%Words in the paper of Kewei Tu: We omit all the proofs in this paper because of  space constraints.
%Note that
Note that all the proofs in this paper are provided in Appendix \ref{sec-appendix-proofs}.
\begin{comment}
Note that all the proofs in this paper are omitted due to the space constraint.
\footnote{
All the material
that is omitted in this paper due to the space constraint
has been provided in the supplementary material by the link
%$https://www.dropbox.com/s/czn99dw7zjl79m4/$ $BySamplingSupplement.pdf$.
%https://www.dropbox.com/s/p3mc7gjqwff2eip/$ $orderSamplingSupplementary.pdf.
%http://www.cs.iastate.edu/~rhe/BySampling/BySamplingSupplement.pdf.
%\smalltt{http://www.cs.iastate.edu/$\sim$rhe/probupdate/}.
\texttt{http://www.cs.iastate.edu/$\sim$rhe/BySampling/}.
}
\end{comment}

It is clear that for each $k \in \{1, \ldots, n\},$ $\sum_{i\in U^{\prec}_{\sigma_{k+1}}} p(\sigma_k=i| \sigma_{k+1}, \ldots, \sigma_n, D)=1$
because of Eq.~(\ref{eq-L_S_recursive}) and $U^{\prec}_{\sigma_k} = U^{\prec}_{\sigma_{k+1}} - \{\sigma_k\}$.
Thus, $p(\sigma_k | \sigma_{k+1}, \ldots, \sigma_n, D)$ is a %wrong: multinomial
probability mass function (pmf) with $k$ possible $\sigma_k$  values from $U^{\prec}_{\sigma_{k+1}}$.

%Based on Proposition~\ref{prop-1}, we can draw an order sample by first drawing the last element $\sigma_n$,
%then drawing the second to last element $\sigma_{n-1}$, and so on.
%The following proposition guarantees that
%the sampled order $\prec$ has the \emph{exact} posterior $p(\prec|D)$.

Based on Proposition~\ref{prop-1}, we propose the following order sampling algorithm to sample an order $\prec$:
\begin{itemize}
\item Sample $\sigma_n$, the last element of the order $\prec$, according to Eq.~(\ref{eq-120330240-init}). %, i.e., .
\item For each $k$ from $n-1$ down to $1$: given the sampled $(\sigma_{k+1}, \ldots, \sigma_n)$,
      sample $\sigma_k$, the $k$th element of the order $\prec$, according to Eq.~(\ref{eq-120330240}).
\end{itemize}

Sampling an order using the above algorithm takes only $O(n^2)$ time
since sampling each element $\sigma_k$ ($k \in \{1, \ldots,n \}$) in the order takes $O(n)$ time.

%The following proposition guarantees that
%the sampled order $\prec$ according to our order sampling algorithm has the \emph{exact} posterior $p(\prec|D)$.

The following proposition guarantees the correctness of our order sampling algorithm.
\begin{proposition}\label{prop-2}
An order $\prec$ sampled according to our order sampling algorithm has its pmf equal to the exact posterior $p(\prec|D)$
%by assuming the order-modular prior,
under the order-modular prior,
because
\begin{align}\label{eq-120330235}
\prod_{k=1}^n p(\sigma_k| \sigma_{k+1}, \ldots, \sigma_n, D) = p(\prec|D).
\end{align}
\end{proposition}

The key to our order sampling algorithm is that
we realize that
the results
including $\alpha'_i(S)$ %(for each $i \in V$ and each $S\subseteq V-\{i\}$)
and $L'(S)$ %(for each $S\subseteq V$)
computed from the DP algorithm
%\citep{Koivisto:soo04}
of \citet{Koivisto:soo04}
are already sufficient to guide the order sampling process.
In an abstract point of view,
the results computed from the DP algorithm
%\citep{Koivisto:soo04}
of \citet{Koivisto:soo04}
%correspond to
are analogous to
%the information that can directly construct
the answers provided by
the $\#$P-oracle stated in Theorem 3.3
%in \citep{Jerrum:TSC1986}.
of
\citet{Jerrum:TSC1986}.
Theorem 3.3
%in \citep{Jerrum:TSC1986}
of \citet{Jerrum:TSC1986}
states that
with the aid of a $\#$P-oracle
that is always able to
%answer
provide
the exact counting information
(the exact number) of accepting configurations
from a currently given configuration,
a probabilistic Turing Machine can serve as a uniform generator
so that every accepting configuration
will be reached with an equal positive probability.
In our situation, instead of providing the exact counting information,
the results
computed from the DP algorithm
of \citet{Koivisto:soo04}
is able to provide the exact joint probability $p(\sigma_k, \sigma_{k+1}, \ldots, \sigma_n, D)$
for a subsequence $(\sigma_k, \sigma_{k+1}, \ldots, \sigma_n)$ of any order $\prec $  %$ = (\sigma_1, \ldots, \sigma_n)$
for any $k \in \{1, 2, \ldots, n \}$, which is shown in the proof for Proposition~\ref{prop-1} in Appendix \ref{sec-appendix-proof-prop1}.
As a result, the order sampling can be efficiently performed
%using the chain rule which represents a joint probability distribution as the product of conditional probability distributions.
based on the definition of conditional probability distribution.

%In our situation we intend to sample each order according to the exact order posterior distribution $p(\prec|D)$
%instead of the uniform probability $1 / (n !)$.
%Thus, we use the information including $\alpha'_i(S)$ and $L'(S)$,
%which actually contains all the information necessary to guide the order sampling process.

%\subsection{DDS Algorithm \label{sec-DDS-algo}}
%\subsection{DDS ALGORITHM \label{sec-DDS-algo}}
\subsection{DDS Algorithm \label{sec-DDS-algo}}
After drawing an order sample,
we can then easily sample a DAG  by drawing parents for each node according to Eq.~(\ref{eq-120330246}) as described
%in
%\citep{friedman:kol03}
by
\citet{friedman:kol03}
(assuming a maximum in-degree $k$).
This naturally leads to our algorithm, termed Direct DAG Sampling (DDS), as follows:
%The resulting DAG $G$ is guaranteed to have the correct posterior $p(G|D)$ by assuming the order-modular prior.
%Thus, our algorithm, termed Direct DAG Sampling (DDS), is as follows:
%\begin{enumerate}
\begin{itemize}
\item Step 1: Run the DP algorithm
              %\cite{Koivisto:soo04}
              %\citep{Koivisto:soo04}
              of \citet{Koivisto:soo04}
              %(i.e., the first three steps of the DP algorithm
              %%\cite{Koivisto:06})
              %%\citep{Koivisto:06}
              %of \citealt{Koivisto:06}) %why no comma?
              with each $f_i(Pa_i)$ set to be the constant 1.
%\item Step 2 (Order Sampling Step): Sample $N_o$ orders such that each order $\prec$ is drawn  by first drawing the last element $\sigma_n$, then second to last element $\sigma_{n-1}$, and so on according to Eq.~(\ref{eq-120330240}).
\item Step 2 (Order Sampling Step): Sample $N_o$ orders such that each order $\prec$ is independently sampled according to our order sampling algorithm.
%\item Step 3 (DAG Sampling Step): For each sampled order $\prec$, sample $N_d$ DAGs.
%                Each DAG is sampled by drawing parent set for each node of the DAG according to Eq.~(\ref{eq-120330246}).
\item Step 3 (DAG Sampling Step): For each sampled order $\prec$,
                one DAG is independently sampled by drawing a parent set for each node of the DAG according to Eq.~(\ref{eq-120330246}).
%\end{enumerate}
\end{itemize}

%The following proposition guarantees that the sampled $N_o$ DAGs according to our DDS algorithm
%$G$ is independent and identically distributed with pmf
%is guaranteed to have the correct posterior $p(G|D)$ by assuming the order-modular prior.

The correctness of our DDS algorithm is guaranteed by
the following
theorem.

\begin{theorem}\label{thm-DDS}
The $N_o$ DAGs sampled according to %our
the DDS algorithm
%$G$ is independent and identically distributed with pmf
are independent and identically distributed (iid) with the pmf equal to
the exact posterior $p_{\prec}(G|D)$
%by assuming the order-modular prior.
%with the assumption of order-modular prior.
under the order-modular prior.

\end{theorem}

The time complexity of the DDS algorithm is as follows.
%\begin{itemize}
%\item
Step 1 takes $O(n^{k+1}C(m) + kn2^n)$ time
%\cite{Koivisto:soo04},
\citep{Koivisto:soo04},
which has been discussed in Section \ref{sec-subsec-DPs}.
%where $n$ is the number of nodes,
%$k$ is the assumed constant maximum in-degree,
%and $C(m)$ is the cost of computing a single local marginal %conditional
%likelihood $score_i(Pa_i:D)$ for $m$ data instances.
In Step 2, sampling each order takes $O(n^2)$ time.
% (since sampling each element $\sigma_k$ in an order takes $O(n)$ time).
In Step 3, sampling each DAG  takes $O(n^{k+1})$ time.
Thus, the overall time complexity of our DDS algorithm is $O(n^{k+1}C(m) + kn2^n + n^{2} N_o + n^{k+1} N_o)$.
Since typically we assume $k \geq 1$,
the order sampling process (Step 2) does not affect the overall time complexity of the DDS algorithm
because of its efficiency.

The time complexity of our DDS algorithm depends on the assumption of the maximum in-degree $k$.
Such an assumption is fairly innocuous, as discussed on page 101 in
%\citep{friedman:kol03},
the article of
\citet{friedman:kol03},
because DAGs with very large families tend to have low scores.
(The maximum-in-degree assumption is also justified
in the context of biological expression data
on page 270 in
%\citep{Marco:ML08}.
the article of \citealt{Marco:ML08}.)
Accordingly, this assumption has been widely used
in the literature
\citep{friedman:kol03,Koivisto:soo04,ellis:won08,Marco:ML08,niinimaki:etal11,Parviainen:ECML2011}
and
%the assumed maximum in-degree $k$ is set to be no greater than $6$ for their algorithms in all their experiments.
the maximum in-degree $k$ has been set to be no greater than $6$ in all of their experiments.

Note that the DAG sampling step of the DDS algorithm takes $O(n^{k+1} N_o)$ time.
This will actually dominate the overall running time of the DDS algorithm (even if $k$ is assumed to be $3$ or $4$),
when $n$ is moderate ($n \leq 25$) and the sample size $N_o$ reaches %a number around
several thousands.
Therefore,
for the efficiency of our DDS algorithm,
we have developed a time-saving strategy for the DAG sampling step, which will be described in details in
Remark \ref{rm:time_DAG_sampling}.
%we have developed a time-saving strategy for the
%DAG sampling step. The basic idea is that ... The experimental results show that our  strategy is very effective. Please refer to Appendix B %for detailed description of the strategy.

%(%We have also used some strategy for Step 3 which can often greatly reduce its real running time when $m$ is not small.
%Please also refer to Remarks \ref{rm:max_indegree_k} \& \ref{rm:time_DAG_sampling} for the maximum-in-degree assumption
%and our strategy for reducing the time cost of Step 3.)

%Given DAG samples, the posterior of any arbitrary feature $p_{\prec}(f|D)$ can be estimated by  Eq.~(\ref{eq-2-HR}).
Given DAG samples,
$\hat{p}_{\prec}(f|D)$,
an estimator for the exact posterior of any arbitrary feature $f$, % $p_{\prec}(f|D)$,
can be constructed by  Eq.~(\ref{eq-2-HR}).
Letting $C_{n, f}$ denote
the time cost of determining the structural feature $f$ in a DAG of $n$ nodes,
constructing $\hat{p}_{\prec}(f|D)$ takes $O(C_{n,f} N_o )$ time.
(For example, $C_{n, f_e}$ $= O(1)$ for an edge feature $f_e$; and $C_{n, f_p}$ $= O(n^2)$ for a path feature $f_p$.)
If we only need order samples, the algorithm consisting of Steps 1 and 2 will be called Direct Order Sampling (DOS).
Given order samples,
for some modular feature $f$ such as a parent-set feature or an edge feature, $p(f|\prec_i,D)$ can be computed by
Eq.~(\ref{eq-120330246}) or (\ref{eq-1-HR}), and then $p_{\prec}(f|D)$ can be estimated by Eq.~(\ref{eq-3-HR}).
(Since computing a parent-set feature or an edge feature by
Eq.~(\ref{eq-120330246}) or (\ref{eq-1-HR}) takes $O(1)$ time ,
estimating $p_{\prec}(f|D)$ by Eq.~(\ref{eq-3-HR}) only takes $O(N_o)$ time for these two features.)

As for the space costs of our DDS algorithm, note that
%both a total order and a DAG can be represented as
%a vector of length $n$, which requires $O(n)$ space.
%both the representation of a total order $(U_1, U_2, \ldots, U_n)$ and
%the representation of a DAG $(P_1, P_2, \ldots, P_n)$ require $O(n)$ space.
both a total order and a DAG can be represented in $O(n)$ space
(since a total order can be represented as a vector $(U_1, \ldots, U_n)$
and a DAG can be represented as a vector $(Pa_1, \ldots, Pa_n)$).
Therefore,
the overall memory requirement of our DDS algorithm is $O(n 2^n + n N_o)$:
Step 1 of our DDS takes $O(n 2^n)$ memory space;
and Steps 2 and 3 of our DDS take $O(n N_o)$ memory space.

%the same as the space complexity of the DP algorithm
%\citep{Koivisto:soo04}.

%Unlike these MCMC algorithms,
Due to Theorem \ref{thm-DDS}, the estimator $\hat{p}_{\prec}(f|D)$ based on our DDS algorithm
has the following desirable properties. %which are shown in the following corollary.

%\begin{coro}
\begin{corollary}
\label{coro-1}
For any structural feature $f$, with respect to the exact posterior $p_{\prec}(f|D)$,
the estimator $\hat{p}_{\prec}(f|D)$ based on the $N_o$ DAG samples from the DDS algorithm using Eq.~(\ref{eq-2-HR}) has the following properties:

(i) $\hat{p}_{\prec}(f|D)$ is an unbiased estimator for $p_{\prec}(f|D)$.

(ii) $\hat{p}_{\prec}(f|D)$ converges almost surely to $p_{\prec}(f|D)$.  %as $N_o \rightarrow \infty$.
%denoted by $\hat{p}_{\prec}(f|D) \rightarrow p_{\prec}(f|D) a.s.$.

%(iii) $\hat{p}_{\prec}(f|D)$ is a consistent estimator for $p_{\prec}(f|D)$,
%i.e., $\hat{p}_{\prec}(f|D)$ converges in probability to $p_{\prec}(f|D)$.
%denoted by $\hat{p}_{\prec}(f|D) \longrightarrow^{p} p_{\prec}(f|D)$.

(iii)
%$\sqrt{N_o}(\hat{p}_{\prec}(f|D) - p_{\prec}(f|D))$ $/ \sqrt{\hat{p}_{\prec}(f|D)(1 - \hat{p}_{\prec}(f|D))}$
%has a limiting standard normal distribution.
%if $0 < \hat{p}_{\prec}(f|D) < 1$,
If $0 < p_{\prec}(f|D) < 1$,
then
the random variable
$$
\frac{\sqrt{N_o}(\hat{p}_{\prec}(f|D) - p_{\prec}(f|D))} { \sqrt{\hat{p}_{\prec}(f|D)(1 - \hat{p}_{\prec}(f|D))} }
$$
has a limiting standard normal distribution. %as $N_o \rightarrow \infty$.

(iv) %$\forall \epsilon > 0, \forall 0 < \delta < 1$,
For any $\epsilon > 0$, any $0 < \delta < 1$,
if $N_o \geq (\ln (2/\delta))/(2\epsilon^2)$,
then $p(|\hat{p}_{\prec}(f|D) -p_{\prec}(f|D)| < \epsilon)$ $ \geq 1 - \delta$.

%\end{coro}
\end{corollary}
%The proof of Corollary \ref{coro-1} is given in the supplementary material due to the space limit.
In particular, Corollary \ref{coro-1} (iv), which is essentially %derived
from Hoeffding bound
%\cite{Hoeffding:1963,Koller:PGM},
\citep{Hoeffding:1963,Koller:PGM}:
%\begin{align*}
$p(|\hat{p}_{\prec}(f|D) -p_{\prec}(f|D)|\geq \epsilon) \leq 2e^{-2 N_o \epsilon^2}$,
%\end{align*}
states that in order to ensure that the probability that the error of the estimator $\hat{p}_{\prec}(f|D)$ from the DDS algorithm
is bounded by $\epsilon$ is at least $1 - \delta$,
we just need to require the sample size $N_o$ $\geq (\ln (2/\delta)) /(2\epsilon^2)$.
%This can be used to obtain performance guarantee on the quality of the estimator from our DDS algorithm.
This property,
which the existing MCMC algorithms \citep{friedman:kol03,niinimaki:etal11} do not have,
can be used to obtain quality guarantee for the estimator from our DDS algorithm.

\begin{comment}
\begin{remark}
\label{rm:max_indegree_k}
About the maximum-in-degree assumption.

The time complexity of our DDS algorithm depends on the assumption of the maximum in-degree $k$.
Such an assumption is fairly innocuous, as discussed on page 101 in \citep{friedman:kol03},
because DAGs with very large families tend to have low scores.
(The maximum-in-degree assumption is also justified
in the context of biological expression data
on page 270 in \citep{Marco:ML08}.)
Accordingly, this assumption has been widely used
in the literature
\citep{friedman:kol03,Koivisto:soo04,ellis:won08,Marco:ML08,niinimaki:etal11,Parviainen:ECML2011}
and the assumed maximum in-degree $k$ is set to be no greater than $6$ for their algorithms
in all their experiments.

%\citep{Koivisto:soo04}
%\citeyearpar{ellis:won08}
%\citep{Marco:ML08}
%\citep{niinimaki:etal11}
%\citeyearpar{Parviainen:ECML2011}

\end{remark}
\end{comment}

\begin{remark}
\label{rm:time_DAG_sampling}
%About the time cost of Step 3 of DDS.
About our time-saving strategy for the DAG sampling step of the DDS.

The running time of the DAG sampling step (Step 3) of the DDS algorithm is $O(n^{k+1} N_o)$,
which will actually dominate the overall running time of the DDS algorithm
when $n$ is moderate
%(say $n \leq 25$)
and the sample size $N_o$ reaches
%a number around
several thousands.
Thus, in the following we will introduce our strategy
%which is able to effectively reduce
for effectively reducing
the running time of the DAG sampling step
so that the efficiency of the overall DDS algorithm can be achieved.

In the DAG sampling step,
each sampled order $\prec_i$
$= (\sigma_1, \ldots, \sigma_n)^{\prec_i} $ ($1 \leq i \leq N_o $)
can be represented as
$\{(\sigma_1, U_{\sigma_1})^{\prec_i}, \ldots, (\sigma_n, U_{\sigma_n})^{\prec_i} \}$,
where $U_{\sigma_j}$ denotes the set of predecessors of $\sigma_j$ in the order.
%For notational convenience, in the following we will use the notation $U_{\sigma_j}^{\prec_i}$
%to actually represent $(\sigma_j, U_{\sigma_j})^{\prec_i}$,
%that is,
%the notation $U_{\sigma_j}^{\prec_i}$ itself also implicitly includes the information of $\sigma_j$.
For each sampled order $\prec_i$,
for each $(\sigma_j, U_{\sigma_j})^{\prec_i}$ ($1 \leq j \leq n $),
we need to sample one $Pa_{\sigma_j}$ of $\sigma_j$ (one parent set of $\sigma_j$) from
a list $\{Pa_{\sigma_j z} \}_z^{\sigma_j}$ including every parent set $Pa_{\sigma_j z}$ $ \subseteq U_{\sigma_j}^{\prec_i}$.
Let $Z_j$ be the length of such a list.
Since $Z_j = \sum_{i=0}^{min\{k, j-1\}} \binom{j-1}{i}$ $ = O(n^k)$,
sampling one $Pa_{\sigma_j}$ for $\sigma_j$  takes $O(n^k)$ time
and sampling one DAG takes $O(n^{k+1})$ time.
Note that $Z_j$ is actually an increasing function of $j$
but in the following we use the notation $Z$ instead of $Z_j$ for notational convenience
when the context is clear.

%However,
However, for $N_o > 1$,
the overall running time of the DAG sampling step can be reduced as follows.
Let
%$\theta^i_{jz}$
$\theta^{ (\sigma_j, U_{\sigma_j})^{\prec_i}  }_{z}$
%be $P(Pa^i_{jz} | \prec_i, D) = P(Pa^i_{jz} | U^{\prec_i}_j, D)$,
be $P( (\sigma_j,  Pa_{{\sigma_j}  z}) | \prec_i, D) = $
$P( (\sigma_j, Pa_{{\sigma_j} z})  | (\sigma_j, U_{\sigma_j})^{\prec_i}, D)$
$= \beta'_{\sigma_j}(Pa_{{\sigma_j} z}) / $ $[\alpha'_{\sigma_j}(U_{\sigma_j}^{\prec_i})/q_{\sigma_j}(U_{\sigma_j}^{\prec_i})]$,
for $z \in \{1, \ldots, Z\}$.
First, using the common strategy of sampling from a discrete distribution \citep{Koller:PGM},
for $(\sigma_j, U_{\sigma_j})^{\prec_i}$ we can create $S_I^{ (\sigma_j, U_{\sigma_j})^{\prec_i} }$,
a sequence of $Z$ probability intervals
with the form of
$<[0,
   \theta^{ (\sigma_j, U_{\sigma_j})^{\prec_i} }_{1}),
  [\theta^{ (\sigma_j, U_{\sigma_j})^{\prec_i} }_{1},
   \theta^{ (\sigma_j, U_{\sigma_j})^{\prec_i} }_{1} +
   \theta^{ (\sigma_j, U_{\sigma_j})^{\prec_i} }_{2}), $
$\ldots, $
 $[\sum_{z=1}^{Z-2} \theta^{ (\sigma_j, U_{\sigma_j})^{\prec_i} }_{z},
   \sum_{z=1}^{Z-1} \theta^{ (\sigma_j, U_{\sigma_j})^{\prec_i} }_{z}), $
 $[\sum_{z=1}^{Z-1} \theta^{ (\sigma_j, U_{\sigma_j})^{\prec_i} }_{z},
  1)>$,
where the $l$th interval is
$[\sum_{z=1}^{l-1} \theta^{ (\sigma_j, U_{\sigma_j})^{\prec_i} }_{z},
  \sum_{z=1}^{l}   \theta^{ (\sigma_j, U_{\sigma_j})^{\prec_i} }_{z})$.
Note that $S_I^{ (\sigma_j, U_{\sigma_j})^{\prec_i} }$ can be created in time $O(Z)$
and sampling one $Pa_{\sigma_j}$ for $\sigma_j$ from
a list $\{Pa_{\sigma_j z} \}_z^{\sigma_j}$ can then be achieved
using binary search in time $O(log Z)$ based on $S_I^{ (\sigma_j, U_{\sigma_j})^{\prec_i} }$.
Then the following observation is the key reason for reducing the running time of the DAG sampling step.
For two sampled orders $\prec_i$ and $\prec_{i'}$ ($1 \leq i, i' \leq N_o $),
even if $\prec_i$ $\neq  \prec_{i'}$,
it is possible that  $(\sigma_j, U_{\sigma_j})^{\prec_i}$ $ = (\sigma_j, U_{\sigma_j})^{ \prec_{i'} } $
for some $j \in \{1, \ldots, n\}$.
This is because for each $j$, $(\sigma_j, U_{\sigma_j})$ essentially has a multinomial distribution
with $N_o$ trials and a set of $n \binom{n-1}{j-1}$ cell probabilities $\{P((\sigma_j, U_{\sigma_j}) | D)\}$.
Actually, for any $j$, the following relation holds for each cell probability:
%$P((\sigma_j, U_{\sigma_j}) | D) $ $\propto  \alpha'_{\sigma_j}(U_{\sigma_j}) L'(U_{\sigma_j}) R'(V - U_{\sigma_j} - \{ \sigma_j \})$,
\begin{align}
\label{eq-relation_p-sigma_j-U_sigma_j}
P((\sigma_j, U_{\sigma_j}) | D) \propto  \alpha'_{\sigma_j}(U_{\sigma_j}) L'(U_{\sigma_j}) R'(V - U_{\sigma_j} - \{ \sigma_j \}),
\end{align}
where $R'(.)$ is the special case of $R(.)$ by setting $f \equiv 1$
and $R(.)$ is defined in Eq.~(\ref{eq-backward-contr}).
Its proof is very similar to the derivation shown
%in \citep{Koivisto:06}
by \citet{Koivisto:06}
and
is provided in Appendix~\ref{sec-appendix-proofs}.
Note that $(\sigma_j, U_{\sigma_j})^{\prec_i}$ $ = (\sigma_j, U_{\sigma_j})^{ \prec_{i'} } $
implies  $S_I^{ (\sigma_j, U_{\sigma_j})^{\prec_i} }$ $ = S_I^{ (\sigma_j, U_{\sigma_j})^{ \prec_{i'} } }$.
Thus, by storing the created $S_I^{ (\sigma_j, U_{\sigma_j})^{\prec_i} }$ in the memory,
once $(\sigma_j, U_{\sigma_j})^{\prec_i}$ $ = (\sigma_j, U_{\sigma_j})^{ \prec_{i'} } $ for $i' > i$,
creating $S_I^{ (\sigma_j, U_{\sigma_j})^{ \prec_{i'} } }$ can be avoided and
sampling one $Pa_{\sigma_j}$ for $\sigma_j$ takes only $O(log Z)$ time.

On one hand, our strategy will definitely save the running time
for these $j$'s such that $n \binom{n-1}{j-1}$ (the number of all the possible values of $(\sigma_j, U_{\sigma_j})$) is smaller than $N_o$
if every created $S_I^{ (\sigma_j, U_{\sigma_j}) }$ is stored.
This is because the running time of sampling $Pa_{\sigma_j}$ of $\sigma_j$ is only $O(log Z)$
in at least $N_o - n \binom{n-1}{j-1}$ samples out of the overall $N_o$ samples.
(In the worst case,  $S_I^{ (\sigma_j, U_{\sigma_j}) }$ will be created for each possible $(\sigma_j, U_{\sigma_j})$.)
%An extreme example is when $j$ $ = n$.
For example, when $j$ $ = n$,
%In this situation,
the number of all the possible values of $(\sigma_j, U_{\sigma_j})$
is only $n$ and $Z_j$ (the length of the list $\{Pa_{\sigma_j z} \}_z^{\sigma_j}$) achieves its maximum among all the $j$'s
so that sampling one $Pa_{\sigma_n}$ for $\sigma_n$ takes $O(log Z_n)$ time in at least $N_o - n$ samples.
%Furthermore,
%however,
Accordingly, when $j$ $ = n$,
the worst-case running time of sampling the $N_o$ $(\sigma_n, Pa_{\sigma_n})$ families is
$O(n (Z_n + log Z_n) + (N_o - n) log Z_n)$ $ = O(n Z_n + N_o log Z_n)$.
On the other hand,
our strategy usually can also save the running time even for these $j$'s such that
the number of all the possible values of $(\sigma_j, U_{\sigma_j})$ is larger than $N_o$.
This is because the probability mass usually is  not uniformly distributed among the set of all the
possible values of $(\sigma_j, U_{\sigma_j})$.
Once the majority of probability mass $p_{_{\Sigma}}$ is concentrated on $r_j$ $(\sigma_j, U_{\sigma_j})$ values,
where $r_j$ is a number smaller than $N_o$,
the probability that
only these $r_j$ $(\sigma_j, U_{\sigma_j})$ values appear in all the $N_o$ order samples
is $(p_{_{\Sigma}})^{N_o}$.
Accordingly, with the probability $(p_{_{\Sigma}})^{N_o}$,
the running time of sampling $Pa_{\sigma_j}$ for $\sigma_j$ is $O(log Z_j)$
in at least $N_o - r_j$ samples.
As a result,
the expected running time of sampling the $N_o$ $(\sigma_j, Pa_{\sigma_j})$ families
is below $O( [r_j Z_j + N_o log Z_j](p_{_{\Sigma}})^{N_o} + N_o (Z_j + log Z_j) ( 1 - (p_{_{\Sigma}})^{N_o}) )$;
the expected running time of sampling the $N_o$ DAGs
is below $O( \sum_{j=1}^n \{  [r_j Z_j + N_o log Z_j](p_{_{\Sigma}})^{N_o} + N_o (Z_j + log Z_j) ( 1 - (p_{_{\Sigma}})^{N_o})  \})$.
%Typically, the probability mass will be more unevenly distributed among the set of all the
%possible values of $(\sigma_j, U_{\sigma_j})$ when $m$ is not small.
%This is because when $m$ is not small local score $score_i(Pa_i:D)$ will typically be not uniform.
Typically, when $m$ is not small,
local score $score_i(Pa_i:D)$ will not be uniform at all.
Correspondingly,
it is likely that
the multinomial probability mass function
$P( (\sigma_j, U_{\sigma_j}) | D)$ will
concentrate dominant probability mass on a small number of
$(\sigma_j, U_{\sigma_j})$ candidates that these $(\sigma_j, Pa_{\sigma_j})$'s having large local scores
are consistent with.
%are consistent with some $(\sigma_j, Pa_{\sigma_j})$ with large $score_{\sigma_j}(Pa_{\sigma_j}:D)$.
As a result,
%$U_{\sigma_j}^{\prec_i}$ $ = U_{\sigma_j}^{ \prec_{i'} } $
%will occur more often
%so that
our time-saving strategy will usually become more effective when $m$ is not small.

\begin{comment}
Furthermore,
the above time-saving strategy
usually will become more effective when $m$ is not small.
Typically, when $m$ is not small,
local score $score_i(Pa_i:D)$ will be not uniform.
Correspondingly
it is likely that
the multinomial probability mass function
$P(U_{\sigma_j} | D)$ will also be not uniform and
concentrate large probability mass on a small number of
$U_{\sigma_j}$ candidates that include some $Pa_{\sigma_j}$ with large $score_{\sigma_j}(Pa_{\sigma_j}:D)$.
As a result, $U_{\sigma_j}^{\prec_i}$ $ = U_{\sigma_j}^{ \prec_{i'} } $
will occur more often
so that our time-saving strategy will become more effective.
\end{comment}

Note that we also include the policy of recycling the created
$S_I^{ (\sigma_j, U_{\sigma_j}) }$s for our strategy
because it is possible that all the memory in a computer will be exhausted in order to
store all the created $S_I^{ (\sigma_j, U_{\sigma_j}) }$s,
especially when $n$ is not small but $m$ is small.  %and $N_o$ is large.
(The space complexity of storing all the $S_I^{ (\sigma_j, U_{\sigma_j}) }$s is
$O(\sum_{j=1}^n n \binom{n-1}{j-1} Z_j)$ $= O(n^{k+1} 2^{n-1})$.)
For this paper, we use a simple recycling method as follows.
Some upper limit for the total number of the probability intervals
(representing
$[\sum_{z=1}^{l-1} \theta^{ (\sigma_j, U_{\sigma_j})^{\prec_i} }_{z},
  \sum_{z=1}^{l}   \theta^{ (\sigma_j, U_{\sigma_j})^{\prec_i} }_{z})$ )
is pre-specified based on the memory of the used computer.
Each time such an upper limit is reached during the DAG sampling step of the DDS,
which indicates a large amount of memory has been used to store $S_I^{ (\sigma_j, U_{\sigma_j}) }$s,
we recycle  the currently stored $S_I^{ (\sigma_j, U_{\sigma_j}) }$s according to
their usage frequencies which serve as the estimates of $P( (\sigma_j, U_{\sigma_j}) | D) $s.
The memory for each infrequently used $S_I^{ (\sigma_j, U_{\sigma_j}) }$ will be reclaimed to ensure that
at least a pre-specified number of probability intervals will be recycled %removed
from the memory.
In addition, in order to have a better use of each created $S_I^{ (\sigma_j, U_{\sigma_j}) }$
before it possibly gets reclaimed,
we sort the $N_o$ sampled orders according to the posterior $p(\prec|D)$ just before executing the DAG sampling step of the DDS.
The underlying rationale is that
if $p(\prec_i|D)$ is relatively close to $p(\prec_{i'}|D)$,
which indicates $p(\prec_i, D)$ is relatively close to $p(\prec_{i'}, D)$ (since $p_{\prec}(D)$ is a constant),
due to Eq.~(\ref{eq-p_prec_D}),
it is likely that
%$p(\prec_i, D)$ and $p(\prec_{i'}, D)$ share some component $\alpha'_{\sigma_j}(U_{\sigma_j})$ so that
$\prec_i$ and $\prec_{i'}$ share some $(\sigma_j, U_{\sigma_j})$ component(s).
(The extreme situation is that if $p(\prec_i|D)$ equals $p(\prec_{i'}|D)$,
it is very likely that $\prec_i$ equals $\prec_{i'}$
so that $\prec_i$ and $\prec_{i'}$ share every $(\sigma_j, U_{\sigma_j})$.)
Thus, $\prec_i$ and $\prec_{i'}$ having similar posteriors tend to be close to each other after the sorting
so that it is likely that the common $S_I^{ (\sigma_j, U_{\sigma_j}) }$ will be used before the reclamation.
Furthermore, as $N_o$ increases,
the probability that two orders %$\prec_i$ and $\prec_{i'}$
(out of the $N_o$ sampled orders)
share some $(\sigma_j, U_{\sigma_j})$ component(s) increases.
Accordingly, after the sorting,
the probability of reusing $S_I^{ (\sigma_j, U_{\sigma_j}) }$ before its reclamation
will also increase.
As a result, the benefit of our time-saving strategy will typically increase when $N_o$ increases.

%Such a benefit from the sorting will typically increase when $N_o$ increases.
%This is because the probability that two orders $\prec_i$ and $\prec_{i'}$ share one or more $(\sigma_j, U_{\sigma_j})$ components
%after the sorting
%will typically increase when $N_o$ increases.

The experimental results show that
our time-saving strategy for the DAG sampling step of the DDS is very effective. %especially when $m$ is not small.
Please see the discussion in Section \ref{sec-exp} about
$\hat{\mu}(T_{DAG})$ and $\hat{\sigma}(T_{DAG})$,
the sample mean and the sample standard deviation of the running time of the DAG sampling step of the DDS,
which are reported in Table \ref{tb:bias_comparison_time_1}
and Table \ref{tb:unbiased_comparison_time_1}.

\end{remark}

\subsection{IW-DDS Algorithm \label{sec-bias}}

In this subsection
we present our DAG sampling algorithm
under the general structure-modular prior (Eq.~(\ref{eq-smodularity-1}))
by effectively correcting the bias due to the use of the order-modular prior.

As mentioned in Section \ref{sec-intro},
%due to the use of order-modular prior,
$p_{\prec}(f|D)$ has the bias due to the assumption of the order-modular prior.
%from the perspective that $p_{\prec}(f|D)$ does not equal $ p_{\nprec}(f|D)$.
%which is based on the standard structure-modular prior instead of order-modular prior.
This is essentially because $p_{\prec}(G|D)$ based on the order-modular prior (Eq.~(\ref{eq-smodularity}))
%which equals $\sum_{\prec s.t. G \subseteq  \prec}  p(\prec, G | D)$,
is different from $p_{\nprec}(G|D)$ based on the structure-modular prior (Eq.~(\ref{eq-smodularity-1})).
%p<() based on the order-modular prior (6) is different with p<() based on the structure-modular prior (2).

In fact, with the common setting that $q_i(U_i)$ always equals $1$ $(q_i(U_i) \equiv 1)$, %use \equiv
%denoted by $q_i(U_i) \equiv 1$,
if $\rho_i(Pa_i)$ in Eq.~(\ref{eq-smodularity}) is set to be always equal to $p_i(Pa_i)$ in Eq.~(\ref{eq-smodularity-1})
$(\rho_i(Pa_i)$ $\equiv  p_i(Pa_i))$,  %for each $(i, Pa_i)$,
%denoted by $\rho_i(Pa_i)$ $\equiv  p_i(Pa_i)$,
the following relation holds: %\citep{ellis:won08}:
\begin{align}
\label{eq-relation_p_prec_p_nprec}
%p_{\prec}(G|D) \propto |\prec_{G}| \cdot  p_{\nprec}(G|D) \label{eq-p_prec_vs_p_nprec}
p_{\prec}(G|D) = \frac{p_{\nprec}(D)}{p_{\prec}(D)} \cdot |\prec_{G}| \cdot p_{\nprec}(G | D),
\end{align}
where $|\prec_{G}|$ is the number of %total
orders that $G$ is consistent with.
(The proof of Eq.~(\ref{eq-relation_p_prec_p_nprec}) is given in Appendix~\ref{sec-appendix-proofs}.)
Accordingly,
\begin{align*}
p_{\nprec}(f|D)
= \sum_G f(G) p_{\nprec}(G|D)
= \sum_G f(G) \frac{p_{\prec}(D)}{p_{\nprec}(D)} \cdot \frac{1}{|\prec_{G}|} p_{\prec}(G|D).
\end{align*}
Note that
$p_{\prec}(D)$ can be computed by the DP algorithm
of \citet{Koivisto:soo04}
in $O(n^{k+1}C(m) + kn2^n)$ time,
and $p_{\nprec}(D)$ can be computed by the DP algorithm
of \citet{tian:he2009}
in $O(n^{k+1}C(m) + kn2^n + 3^n)$ time.
Thus,
if $|\prec_{G_i}|$ is known for each sampled $G_i$ ($i \in \{1, 2, \ldots, N_o \}$),
we can use importance sampling to obtain a good estimator
\begin{align}
\label{eq-importance_samp}
\tilde{p}_{\nprec}(f|D) = \frac{1}{N_o}  \sum_{i=1}^{N_o} f(G_i) \frac{p_{\prec}(D)}{p_{\nprec}(D)} \cdot \frac{1}{|\prec_{G_i}|},
\end{align}
where each $G_i$ %($i \in \{1, \ldots, N_o \}$)
is sampled from our DDS algorithm.
%(The properties of $\tilde{p}_{\nprec}(f|D)$
%%with respect to
%are analogous to the properties of $\hat{p}_{\prec}(f|D)$
%stated in Corollary \ref{coro-1}.)
Unfortunately, $|\prec_{G_i}|$ is $\#$P hard to compute for each $G_i$
%\cite{Brightwell:TOC1991}.
\citep{Brightwell:TOC1991};
and the state-of-the-art DP algorithm proposed
%in \citep{Niinimaki:IJCAI2013}
by \citet{Niinimaki:IJCAI2013}
for computing $|\prec_{G_i}|$
takes $O(n 2^n)$ time.
Therefore, %due to the expensive computation cost,
in the following
we propose an estimator that can be much more efficiently computed than the estimator shown in Eq.~(\ref{eq-importance_samp}).
%we will propose an algorithm which can construct an estimator other than the estimator shown in Eq.~(\ref{eq-importance_samp}).
%In the following we propose an estimator that is much more efficient than (23).

%our DDS algorithm can not be directly used for Eq.~(\ref{eq-importance_samp}) due to the computation cost.
%the computation cost of $|\prec_{G_i}|$ for
%a large number of sampled DAGs
%would be a demanding task.

%and $|\prec_{G}|$ is $\#$P hard to compute
%\cite{Brightwell:TOC1991}.
%\citep{Brightwell:TOC1991}.
%(Please refer to the supplementary material for the proof of the relation.)
%(The proof of the relation is omitted due to the space constraint.)
%Unfortunately, because computing $|\prec_{G}|$ is  %$\sharp$ $\#$P hard \cite{Brightwell:TOC1991}.

Because $p_{\prec}(f|D)$ has the bias with respect to
%$p_{\prec}(f|D)$ $ \neq p_{\nprec}(f|D)$,
$p_{\nprec}(f|D)$,
a good estimator $\hat{p}_{\prec}(f|D)$ for $p_{\prec}(f|D)$ typically
is not appropriate to be directly used to estimate $p_{\nprec}(f|D)$.
Noticing this problem,
Ellis and Wong
%\cite{ellis:won08} %(2008) %\nocite{ellis:won08}  \cite{ellis:won08}
%\citeyear{ellis:won08}
\citeyearpar{ellis:won08}
%\citet{ellis:won08}
propose to correct this bias for the Order MCMC method as follows:
first run the Order MCMC to draw order samples;
%then for each unique sampled order,
%draw a number of consist DAGs until a large probability mass is covered;
then for each unique order $\prec$ out of the sampled orders,
keep drawing DAGs consistent with $\prec$ (but only keep unique DAGs)
until the sum of joint probabilities of these unique DAGs,
$\sum_{i} p(\prec, G_i, D)$,
is no less than a pre-specified large proportion (such as $95\%$) of $p(\prec, D)$
%$= \sum_{G \in \mathcal{G}_{\prec}} p(G, D)$;
$= \sum_{G \subseteq \prec} p(\prec, G, D)$;
finally the resulting union of all the DAG samples %(by eliminating the duplicate DAGs)
is treated
as an importance-weighted sample %with $p(G,D)$ as weights %commented by Ru since this is wrong
%to make future predictions.
for the structural discovery.

%The bias-correction strategy used in our IW-DDS
%solves the potential computational problem from the strategy in
%\citep{ellis:won08}.
%Please refer to the supplementary material for the discussion due to the space limit.

%By refining the idea of Ellis and Wong
%\citeyear{ellis:won08},
%\citeyearpar{ellis:won08},
%here we propose our bias-corrected algorithm, termed IW-DDS (Importance-weighted DDS), as follows:

%Unfortunately, the strategy
%of \citet{ellis:won08}
%has some drawbacks
%in terms of both computation costs and the properties of the sampled DAGs.
%To address the drawbacks,
%we propose our own bias-corrected algorithm,
%termed IW-DDS (Importance-weighted DDS), as follows:
Inspired by the idea of \citet{ellis:won08},
we develop our own bias-correction strategy
which is computationally more efficient
and can theoretically ensure the resulting estimator to have desirable properties.
(Please refer to Remark~\ref{rm:problem_in_Ellis} for detailed discussion.)
Our bias-corrected algorithm,
termed IW-DDS (Importance-weighted DDS),
is as follows:

\begin{itemize}
%\item Step 1: Run DP algorithm \cite{Koivisto:soo04} %(i.e., the first three steps of DP algorithm \cite{Koivisto:06})
%with all $f_i(Pa_i)$ setting to constant 1.
%\item Step 2: Sample $N_o$ orders according to Eq.~(\ref{eq-120330240}).
%\item Step 3: For each of these $N_o$ sampled orders, sample $N_d$ DAGs according to Eq.~(\ref{eq-120330246}).
%\item Step 4: Make the union set $\mathcal{G}$ of all the sampled DAGs %which restricts the occurrence of each sampled DAG to be $1$, %i.e., by eliminating the duplicate DAGs.
\item Step 1 (DDS Step): Run the DDS algorithm to draw $N_o$ DAG samples
                         with the setting that $q_i(U_i) \equiv 1$ and $\rho_i(Pa_i)$ $\equiv  p_i(Pa_i)$.
\item Step 2 (Bias Correction Step): Make the union set $\mathcal{G}$ of all the sampled DAGs %which restricts the occurrence of each sampled DAG to be $1$,
    by eliminating the duplicate DAGs.

\end{itemize}

Given $\mathcal{G}$, $\hat{p}_{\nprec}(f|D)$, the estimator for the exact posterior of any feature $f$, can then be constructed as
\begin{align}
\label{eq-esti-f-Given-D}
\hat{p}_{\nprec}(f|D) = \sum_{G\in \mathcal{G}} f(G) \hat{p}_{\nprec}(G|D),
\end{align}
where
\begin{align}\label{eq-pgd}
\hat{p}_{\nprec}(G|D)
= \frac{p_{\nprec}(G,D)}{\sum_{G\in \mathcal{G}} p_{\nprec}(G,D)},
%= \frac{ p(D|G) p(G) }{\sum_{G\in \mathcal{G}} p(D|G) p(G) }.
\end{align}
and $p_{\nprec}(G,D)$ is given in Eq.~(\ref{eq-p_G_D_struc}).

%In order to easily obtain the estimator $\hat{p}_{\nprec}(f|D)$ using Eq.~(\ref{eq-esti-f-Given-D}) and (\ref{eq-pgd}),
%the joint probability $p_{\nprec}(G_i, D)$ %given in Eq.~(\ref{eq-p_G_D_struc})
%needs to be conveniently obtained for each sampled $G_i$.
%This can be achieved by recording only $p(D | G_i)$ with each sampled $G_i$ if a uniform prior $p(G)$ is used,
%and recording both $p(D | G_i)$ and $p(G_i)$ with each sampled $G_i$ if a non-uniform prior $p(G)$ is used.
%By this way,
%obtaining $p_{\nprec}(G_i, D)$ for each sampled $G_i$ can be
%achieved by Eq.~(\ref{eq-p_G_D_struc}) in $O(1)$ time.

Since checking the equality of two DAGs takes $O(n)$ time by using their vector representations,
%$(Pa_1, \ldots, Pa_n)$,
with the usage of a hash table,
%as for the bias correction step of the IW-DDS algorithm,
%its expected time cost is $O(n N_o)$
%and its space cost is $O(n N_o)$.
both the expected time cost and the space cost of the bias correction step are $O(n N_o)$.
Therefore,
%the time and space complexity of our IW-DDS algorithm are the same as the ones of our DDS algorithm,
%that is,
%IW-DDS algorithm takes $O(n^{k+1}C(m) + kn2^n + n^{2} N_o + n^{k+1} N_o)$  time
%and requires  $O(n 2^n + n N_o)$ memory space.
the expected time cost of our IW-DDS algorithm is %bounded by
$O(n^{k+1}C(m) + kn2^n + n^{2} N_o + n^{k+1} N_o)$,
and the required memory space of our IW-DDS algorithm is $O(n 2^n + n N_o)$.

%Just as for DDS algorithm,
%we empirically tested the different settings of $(N_o, N_d)$.
%The experimental results show that if $N_o N_d$ is fixed,
%sampling one DAG for each order (i.e. $N_d=1$) consistently performs at least as good as large $N_d$.
%Therefore we will set $N_d=1$ in IW-DDS for the performance comparison in the experiment section.

%Let $C_{n, f}$ denote the time cost of determining the structural feature $f$ in a DAG of $n$ nodes.
%Just as the analysis for the estimator from the DDS,
%constructing the estimator $\hat{p}_{\nprec}(f|D)$ from the IW-DDS
%takes $O(C_{n, f} N_o)$ time
%when the joint probability $p_{\nprec}(G_i, D)$ can be obtained for each sampled $G_i$ in $O(1)$ time.

Note that when each $G_i$ gets sampled,
the corresponding joint probability $p_{\nprec}(G_i, D)$ can be easily computed and stored with $G_i$.
Therefore, just as constructing the estimator from the DDS,
constructing the estimator $\hat{p}_{\nprec}(f|D)$ from the IW-DDS
also takes $O(C_{n, f} N_o)$ time,
where $C_{n, f}$ denotes the time cost of determining the structural feature $f$ in a DAG of $n$ nodes.

%While  Ellis and Wong  \cite{ellis:won08} suggests sampling a large number of DAGs for each order, we empirically find that sampling a single DAG for each order ($N_d=1$) in IW-DDS consistently performs no worse than large $N_d$ (if $N_oN_d$ is kept the same). The reason could be that it is very cheap to sample an order in IW-DDS while in order MCMC it is very expensive to draw order samples compared with sampling a DAG given an order. Therefore we will set $N_d=1$ when testing the performance of IW-DDS in the next section.

While
\citet{ellis:won08}
show the effectiveness of their method in correcting the bias merely by the experiments,
we first theoretically prove that our estimator has desirable properties as follows.

\begin{theorem}\label{thm-IW-DDS}
For any structural feature $f$, with respect to the exact posterior $p_{\nprec}(f|D)$,
the estimator $\hat{p}_{\nprec}(f|D)$ based on the DAG samples from the IW-DDS algorithm using Eq.~(\ref{eq-esti-f-Given-D})
has the following properties:

%(i) $\hat{p}_{\nprec}(f|D)$ is a consistent estimator for $p_{\nprec}(f|D)$,
%i.e., $\hat{p}_{\nprec}(f|D)$ converges in probability to $p_{\nprec}(f|D)$.

(i) $\hat{p}_{\nprec}(f|D)$ is an asymptotically unbiased estimator for $p_{\nprec}(f|D)$.

(ii) $\hat{p}_{\nprec}(f|D)$ converges almost surely to $p_{\nprec}(f|D)$.

(iii) The convergence rate of $\hat{p}_{\nprec}(f|D)$ is $o(a^{N_o})$ for any $0 < a < 1$.

(iv) Define the quantity
$\Delta  = \sum_{G\in \mathcal{G}} p_{\nprec}(G|D)$,
%$ = \frac{\sum_{G\in \mathcal{G}} P(G,D)}{P(D)} $,
then
\begin{align}\label{eq-p_nprec_f_given_D_interval}
 \Delta \cdot \hat{p}_{\nprec}(f|D) \leq p_{\nprec}(f|D) \leq \Delta \cdot \hat{p}_{\nprec}(f|D) + 1 - \Delta.
\end{align}

\end{theorem}

%The proof of Theorem~\ref{thm-IW-DDS} is provided in the supplementary
%material due to the space limit.

Note that the introduced quantity
$\Delta $ $  = \sum_{G \in \mathcal{G}} p_{\nprec}(G, D) / p_{\nprec}(D) $
and
%$0 \leq \Delta \leq 1$.
$\Delta $ $\in [0, 1]$ essentially represents the cumulative posterior probability mass of the
DAGs in $\mathcal{G}$.
Eq.~(\ref{eq-p_nprec_f_given_D_interval}) provides a sound interval
$[\Delta \cdot \hat{p}_{\nprec}(f|D),  \Delta \cdot \hat{p}_{\nprec}(f|D) + 1 - \Delta ]$
in which $p_{\nprec}(f|D)$ \emph{must} reside.
(The ``sound interval'' is stronger than the concept of ``confidence interval''
because there is no probability that $p_{\nprec}(f|D)$ is outside the sound interval.)
The width of the sound interval is $(1 - \Delta)$,
%Also note that
where $\Delta $ is a nondecreasing function of $N_o$
(because if we increase the original $N_o$ to a larger $N_o'$ and sample additional $N_o' - N_o$ DAGs in the DDS step,
the resulting $\mathcal{G'}$ is always the superset of the original $\mathcal{G}$).
Thus, in the situations
where $m$ (the number of data instances) is not very small,
%Such an interval (of the width $1 - \Delta$) may be conservative;
it is possible for $\Delta$ to approach $1$ by a tractable number $N_o$ of DAG samples
%but if $\Delta$ is close to $1$,
so that a desired small-width interval for $p_{\nprec}(f|D)$ can be obtained. (Please refer to Section~\ref{sec-exp}
for the corresponding experimental results.)
Also note that Eq.~(\ref{eq-p_nprec_f_given_D_interval}) can be expressed in the following equivalent form:
\begin{align}\label{eq-p_nprec_f_given_D_interval_equiv}
   -(1 - \Delta) \hat{p}_{\nprec}(f|D)    \leq   p_{\nprec}(f|D) - \hat{p}_{\nprec}(f|D)  \leq  (1 - \Delta) (1 - \hat{p}_{\nprec}(f|D)),
\end{align}
which gives the bound for the estimation error $p_{\nprec}(f|D) - \hat{p}_{\nprec}(f|D)$.

%The width of the interval is $1 - \Delta$
%so that a small interval for $p_{\nprec}(f|D)$ can be determined
%if $\Delta$ is close to $1$.

\begin{remark}
\label{rm:problem_in_Ellis}

%Refinement from bias-correction strategy of IW-DDS
%\subsection{REFINEMENT FROM BIAS-CORRECTION STRATEGY OF IW-DDS}

%This subsection describes
%how the bias-correction strategy used in our IW-DDS refines the one in
%\citep{ellis:won08}.

%About our solution to the drawbacks of the strategy proposed by \citet{ellis:won08}.
%Our bias-correction strategy compared with that of Ellis and Wong (2008)
A comparison between our bias-correction strategy and the one of \citet{ellis:won08}.

Our bias-correction strategy used in the IW-DDS
%refines the idea in
%\citep{ellis:won08}
solves the computation problem existing in the idea
%in
%%\cite{ellis:won08}
%\citep{ellis:won08}
of \citet{ellis:won08}
and ensures the desirable properties of our estimator $\hat{p}_{\nprec}(f|D)$ stated in Theorem \ref{thm-IW-DDS}.

Since in the Order MCMC,
sampling an order is much more computationally expensive than sampling a DAG given an order,
the strategy
%in
%\cite{ellis:won08}
%\citep{ellis:won08}
of \citet{ellis:won08}
emphasizes making the full use of each sampled order $\prec$, that is,
keeping drawing DAGs consistent with each sampled $\prec$
until the sum of joint probabilities for the unique sampled DAGs, $\sum_{i} p(\prec, G_i, D)$,
is no less than a large proportion (such as $95\%$) of $p(\prec, D)$.
Unfortunately, such a strategy
%not only leads to the loss of the independence among the DAGs sampled from the same order, but also
has a %severe
computational problem
when the number of variables $n$ is not small and the number of data instances $m$ is small.
Because there are super-exponential number ($2^{O(k n log (n))}$) of DAGs (with the maximum in-degree $k$) consistent with each order
%(even with the restriction of the maximum in-degree $k$)
%\cite{friedman:kol03},
\citep{friedman:kol03},
it is possible that a non-negligible portion of probability mass $p(\prec, D)$
will be distributed almost uniformly to a majority of these consistent DAGs when $m$ is small.
%these consistent DAGs will possibly have nearly uniform joint probabilities when $m$ is small.
Consequently, $N_{G}^{\prec}$, the required number of DAGs sampled per each sampled order $\prec$,
will be extremely large, leading to %extremely
large computation costs.
For sampling $N_{G}^{\prec}$ DAGs consistent with each sampled order $\prec$,
its expected time cost is $O(n^{k+1} + n k log(n) N_{G}^{\prec})$
(even if a time-saving strategy like the one described in Remark \ref{rm:time_DAG_sampling} is used)
and its memory requirement is $O(n N_{G}^{\prec})$.
If the memory requirement exceeds the memory of the running computer,
the hard disk has to be used to temporarily store the sampled DAGs in some way.
%such as the way that our memory-saving strategy uses for the IW-DDS.
(We notice that  Ellis and Wong
\citeyearpar{ellis:won08}
limit their experiments to the data sets with at most 14 variables.)
If we take
the data set ``Child''
%\cite{Tsamardinos:ML06}
\citep{Tsamardinos:ML06}
with $n = 20$ and $m = 100$ for example,
%for example, for data set ``Child''
%\citep{Tsamardinos:ML06}
%with $n = 20$ and $m = 100$,
for an order $\prec$ randomly sampled by our order sampling algorithm,
our experiment shows that
$1 \times 10^7$ DAGs (which contain $932,137$ unique DAGs) need to be sampled
to let the ratio $\sum_{i} p(\prec, G_i, D)$ $ / p(\prec, D)$ reach $94.071\%$;
$1.5 \times 10^7$ DAGs (which contain $1,204,262$ unique DAGs) need to be sampled
to let the ratio $\sum_{i} p(\prec, G_i, D)$ $ / p(\prec, D)$ reach $94.952\%$.
To address this problem,
based on the efficiency of our order sampling algorithm,
our strategy samples only one DAG from each sampled order in the DDS step
%to maintain the independence among the sampled DAG and
so that
%the computational difficulty is eliminated
the large computation costs per each sampled order are avoided
for any data set.
Meanwhile, unlike the strategy
%in
%\cite{ellis:won08},
%\citep{ellis:won08},
of \citet{ellis:won08},
our strategy does not delete the duplicate order samples.
Therefore, if an order $\prec$ gets sampled $j$ $(\geq 1)$ times in the order sampling step,
essentially $j$ DAGs will be sampled for such a unique order in the DAG sampling step.
Thus, $j$, the number of occurrences, implicitly serves as
%indicates the importance weight
an importance indicator
for $\prec$ among the orders. %in the ordering level.

Furthermore,
the strategy
%in \citep{ellis:won08}
of \citet{ellis:won08}
can not guarantee that the sampled DAGs
are independent,
%the sampling strategy in our IW-DDS is able to guarantee
%the property that the DAGs  sampled from the DDS step
%are independent (stated in Theorem~\ref{thm-DDS}).
%Such a property does not hold using the strategy in \citep{ellis:won08}
even if large computation costs are spent in sampling a huge number of DAGs per each sampled order.
%does not have such property
This is essentially because
multiple DAGs sampled from a fixed order according to the strategy of
%\citep{ellis:won08}
\citet{ellis:won08}
are not independent.
For example, given that a DAG $G$ with an edge $X \rightarrow Y$ gets sampled from an order $\prec$ ,
which implies that node $X$ precedes node $Y$ in the given order $\prec$,
then the conditional probability that any DAG $G'$ with a reverse edge $Y \rightarrow X$
gets sampled under the fixed order $\prec$ becomes zero,
so that $G$ and $G'$ are not independent.
In general, once the number of sampled orders is fixed,
even if %you keep increasing
the number of sampled DAGs per each sampled order keeps increasing,
%all the DAGs that are not consistent with the sampled orders
every DAG that is consistent with none of the sampled orders
will still have no chance of getting sampled.
%As as our experiments show,
%if the total number $N_o$ of DAG samples is wanted,
%our DDS  the best performance is to
%From the perspective of performance,
%Such a limitation degrades the learning performance.
%For example, the learning performance of
In contrast,
the sampling strategy in our IW-DDS is able to guarantee
the property that all the DAGs  sampled from the DDS step
are independent, which has been stated in Theorem~\ref{thm-DDS}.
%Thus, the property of the independence among the sampled DAGs from our DDS step
Such a property actually
is a key to ensuring the good properties of our estimator $\hat{p}_{\nprec}(f|D)$
stated in Theorem \ref{thm-IW-DDS}.

\end{remark}

\begin{comment}
The competing state-of-the-art algorithms also applicable in BNs %networks
of moderate size are
Hybrid MCMC
%\cite{eaton:mur07b}
\citep{eaton:mur07b}
whose first phase runs DP algorithm
%\cite{Koivisto:06}
\citep{Koivisto:06}
with time complexity $O(n^{k+1}C(m) + $ $kn2^n)$ and space complexity $O(n 2^n)$;
%$O(n^{k+1}C(m) + $ $kn2^n + \textrm{MCMC costs})$
and $K$-best
%\cite{tian:he2010}
\citep{tian:he2010}
with time complexity  $O(n^{k+1}C(m) + $ $K^2 n 2^{n-1})$ and space complexity $O(K n 2 ^n)$.
When $n$
$ > 17$,
%the experimental results in \citep{tian:he2010} show that
$K$ can only take some moderate value (such as no more than 100) in order to
make $K$-best method feasible in current desktop computers.
\end{comment}

The competing state-of-the-art algorithms that are also applicable to BNs  %Bayesian networks
of moderate size are
the Hybrid MCMC method
%\cite{eaton:mur07b}
\citep{eaton:mur07b}
and the
$K$-best algorithm
%\cite{tian:he2010}
\citep{tian:he2010}.
The first competing method, the Hybrid MCMC,
includes the DP algorithm
%\cite{Koivisto:06}
%\citep{Koivisto:06}
of \citet{Koivisto:06}
(with time complexity $O(n^{k+1}C(m) + $ $kn2^n)$ and space complexity $O(n 2^n)$)
as its first phase
and then uses the computed posteriors
of all the edges
to make the global proposal for its second phase (MCMC phase).
When its MCMC phase eventually converges,
the Hybrid MCMC will correct the bias coming from the order-prior assumption
and provide DAG samples according to the DAG posterior so that
the estimator $\hat{p}_{\nprec}(f|D)$ can be constituted using Eq.~(\ref{eq-2-HR}) for any feature $f$.
The Hybrid MCMC has been empirically shown to
converge faster than both the Structure MCMC and the Order MCMC
so that more accurate structure learning performance can be obtained \citep{eaton:mur07b}.
Note that since
the REV-MCMC method \citep{Marco:ML08}
is shown to be only nearly
as efficient as the Order MCMC in the mixing and convergence,
the Hybrid MCMC is expected to converge faster than the REV-MCMC method
as long as $n$ is moderate so that the Hybrid MCMC is applicable.
(But the REV-MCMC method has its own value in learning large BNs  %Bayesian networks of a large $n$
since all these methods using some DP technique
(including the Hybrid MCMC, the $K$-best algorithm and our IW-DDS method)
are infeasible for a large $n$ due to the space cost.)
%However,
%Eaton and Murphy do not provide
One limitation of the Hybrid MCMC is that it can not obtain
the interval for $p_{\nprec}(f|D)$ as
specified by Theorem~\ref{thm-IW-DDS} (iv).
Additionally,
%neither the convergence rate  nor  the variance
the convergence rate
of the estimator from the Hybrid MCMC
is not theoretically provided by its authors.

The second competing method,
the $K$-best algorithm,
applies DP technique to
obtain a collection $\mathcal{G}$ of DAGs with the $K$ best scores
and then uses these DAGs to constitute the estimator $\hat{p}_{\nprec}(f|D)$
%based on Bayesian model averaging.
by Eq.~(\ref{eq-esti-f-Given-D}) and (\ref{eq-pgd}).
%The outcome of $K$-best method is not random
%so that it is meaningless to
One advantage of the $K$-best algorithm is that
its estimator also has the property specified as Theorem \ref{thm-IW-DDS} (iv)
so that it can
provide the sound interval for $p_{\nprec}(f|D)$
just as our IW-DDS.
However,
%since
the $K$-best algorithm has
time complexity
%$O(n^{k+1}C(m) + $ $K^2 n 2^{n-1})$ and space complexity $O(K n 2 ^n)$,
$O(n^{k+1}C(m) + $ $T'(K) n 2^{n-1})$ and space complexity $O(K n 2 ^n)$,
where $T'(K)$ is the time spent on the best-first search for $K$ solutions and
$T'(K)$ has been shown to be $O(K log K)$ by \citet{Flerova:GSKRR2012}.
Thus,
the increase in $K$ will dramatically increase the computation costs of the $K$-best algorithm
when $n$ is not small.
As a result, to obtain
%the interval of the width
an interval width
similar to the one from our IW-DDS,
much more time and space costs are required for the $K$-best.
%This computation problem becomes extremely severe for $n \geq 19$
%since $K$ can only take some small value (such as no more than 40 in our experiments) in order to
%prevent the $K$-best method from exhausting the memory of current desktop computers.
In our experiments using an ordinary desktop PC,
the computation problem becomes severe for $n \geq 19$
since $K$ can only take some small values (such as no more than 40) before the $K$-best algorithm exhausts the memory.
Accordingly,
$\Delta$ obtained from the $K$-best is usually smaller than the one from our IW-DDS
(so that the interval from the $K$-best is usually wider)
even if
%the setting of $K$ reaches
$K$ is set to reach
the memory limit of a computer.
%so that the bound for $p_{\nprec}(f|D)$ from from $K$-best
%tends to be more loose.
(Please refer to Section~\ref{sec-expr-IWDDS} for detailed discussion.)

%\section{Experimental Results \label{sec-exp}}
%\section{EXPERIMENTAL RESULTS \label{sec-exp}}
\section{Experimental Results \label{sec-exp}}
We have implemented our algorithms in a C++ language tool called ``BNLearner''
and
%we
run several experiments to demonstrate its capabilities.
(BNLearner  is available at \url{http://www.cs.iastate.edu/~jtian/Software/BNLearner/BNLearner.htm}.)
%(BNLearner  is available at http://www.cs.iastate.edu /$\sim$jtian/Software/BNLearner/BNLearner.htm.)
%(BNLearner  is available at http://www.cs.iastate.edu /$\sim$rhe/BNLearner.)
%(KBest      is available at http://www.cs.iastate.edu/$\sim$rhe/KBest/.)
Our tested data sets include ten real data sets
from the UCI Machine
Learning Repository
%\cite{asuncion:new07},
\citep{asuncion:new07}:
``Tic-Tac-Toe,'' (which is also simply called ``T-T-T,'')
``Glass,''
``Wine,''
``Housing,''
``Credit,''
``Zoo,''
``Letter,''
``Tumor,''
``Vehicle,''
and
``German''.
%(which is also often called ``Statlog'').
Our tested data sets also include
three synthetic data sets:
the first one is
a synthetic data set ``Syn15'' generated from a gold-standard 15-node Bayesian network built by us;  %\cite{our UAI 10}
the second one is a synthetic data set ``Insur19''
generated from a 19-node subnetwork of ``Insurance'' Bayesian network
%in
\citep{Binder:ML1997};
%such that the 19-node subnetwork is causally sufficient \citep{spirtes:etal2001};
and the third one is a synthetic data set ``Child'' from a 20-node ``Child'' Bayesian network used
%in
%\cite{Tsamardinos:ML06}.
%\citep{Tsamardinos:ML06}.
by \citet{Tsamardinos:ML06}.
All the data sets contain only discrete variables (or are discretized)
and have no missing values (or have its missing values filled in).
For the four data sets (``Syn15,'' ``Letter,'' ``Insur19'' and ``Child''),
since a large number of data instances are available,
we also vary $m$ (the number of instances)  to see the corresponding learning performance.
(All the data cases are also included in the tool of BNLearner.)
All the experiments in this section
were run under Linux on one ordinary desktop PC with a 3.0GHz Intel Pentium processor and 2.0 GB memory
if no extra specification is provided.
In addition, the maximum in-degree $k$ is assumed to be $5$ for all the experiments.

%\subsection{Experimental Results for DOS and DDS}
%\subsection{EXPERIMENTAL RESULTS FOR DOS AND DDS}
%\subsection{EXPERIMENTAL RESULTS FOR THE DDS}
\subsection{Experimental Results for the DDS}
\label{sec-expr-DDS}
%The algorithms using the order-modular prior still have a lot of applications as long as the bias  can be tolerated in the application scenario \cite{friedman:kol03, Koivisto:soo04}.
In this subsection,
we compare our DDS algorithm with the Partial Order MCMC method
%\cite{niinimaki:etal11},
\citep{niinimaki:etal11},
the state-of-the-art learning method
%that can estimate the posteriors of any features
under the order-modular prior.

The Partial Order MCMC (PO-MCMC) method is implemented in
BEANDisco, %\footnote{BEANDisco is available at http://www.cs.helsinki.fi/u/tzniinim/BEANDisco/.},
a C++ language tool provided by
%the authors of
%\cite{niinimaki:etal11}.
\citet{niinimaki:etal11}.
(BEANDisco is available at
\url{http://www.cs.helsinki.fi/u/tzniinim/BEANDisco/}.)
The current version of BEANDisco
%only supports to
can only
estimate the posterior of an edge feature,
but as
%the authors have stated in
%\cite{niinimaki:etal11},
\citet{niinimaki:etal11} have stated,
the PO-MCMC readily enables estimating the posterior of any structural feature
by further
sampling DAGs consistent with an order.

\begin{table*}[t]
%\begin{table}[t]
%\caption{The comparison of PO-MCMC, DOS and DDS in SAD}
%\caption{The comparison of PO-MCMC, DOS and DDS in terms of both SAD and time (time is in seconds).}
%\caption{Comparison of the PO-MCMC, the DOS \& the DDS in Terms of SAD}
%\label{tb:bias_comparison_sad_1}
%\vskip 0.10in
\begin{center}
%\begin{small}
%{\small
\hfill{}

\scalebox{0.72}{
\begin{tabular}{ccr||rr|rr|rr}
\hline
%\hline
Name    & $n$   & $m$       & \multicolumn{2}{c} {PO-MCMC}  &\multicolumn{2}{|c|} {DOS}
& \multicolumn{2}{c} {DDS}  \\
        &       &       & $\hat{\mu}$(SAD) & $\hat{\sigma}$(SAD) & $\hat{\mu}$(SAD)  & $\hat{\sigma}$(SAD)  & $\hat{\mu}$(SAD) & $\hat{\sigma}$(SAD) \\
\hline

T-T-T (Tic-Tac-Toe)	& 10   & 958	& 0.5174	   & 0.1280		    & 0.1350	     & 0.0257		    & 0.1547	& 0.0378 \\

Glass	  & 11	   & 214	& 0.0696	   & 0.0249		    & 0.0230	     & 0.0067		    & 0.0529	& 0.0076 \\

%BC-Wisc	  & 11	   & 699	& 0.0640	   & 0.0199		    & 0.0336	     & 0.0099		    & 0.0555	& 0.0094 \\

Wine	  & 13	   & 178	& 0.1616	   & 0.0403		     & 0.0505	    & 0.0138		     & 0.0839	& 0.0137	\\

Housing	  & 14	   & 506	& 0.3205	   & 0.1303		     & 0.0691	    & 0.0146		     & 0.1150	& 0.0117	\\

Credit	  & 16	   & 690	& 0.4549	    & 0.2495		& 0.0581	      & 0.0221		    & 0.1071	& 0.0165 \\

Zoo	      & 17	   & 101	& 0.6079	   & 0.1809		     & 0.1020	    & 0.0130		     & 0.2756	& 0.0137	\\

%HouseVotes &17	   & 435	& 0.7057	   & 0.1552		     & 0.1398	    & 0.0215		     & 0.2481	& 0.0172	\\

Tumor	  & 18	   & 339	& 0.6059	    & 0.1849		& 0.0877	      & 0.0226		    & 0.2050	& 0.0180 \\

Vehicle	  & 19	   & 846	& 6.9774	    & 8.7960		& 0.0200	      & 0.0042		    & 0.0547	& 0.0096 \\

German	   & 21	   & 1,000	& 2.8802	   & 2.0191		     & 0.0994	     & 0.0295		     & 0.1298	& 0.0338	\\

\hline

Syn15   & 15    & 100       & 0.9024	   & 0.2258        & 0.1246        & 0.0142            & 0.2622	    & 0.0190 \\
        &       & 200       & 0.6449	   & 0.1569		   & 0.0975	       & 0.0163	           & 0.2228	    & 0.0172 \\
        &       & 500       & 0.3424	   & 0.1214		   & 0.0651	       & 0.0153		       & 0.1116	    & 0.0126 \\
        &       & 1,000     & 0.1558	   & 0.0496		   & 0.0515	       & 0.0128		       & 0.0724	    & 0.0118 \\
        &       & 2,000     & 0.0465	   & 0.0209		   & 0.0359	       & 0.0075		       & 0.0473	    & 0.0071 \\
        &       & 5,000     & 0.0217	   & 0.0144		   & 0.0196	       & 0.0064		       & 0.0247	    & 0.0086 \\

\hline

Letter  & 17    & 100       & 0.9530	   & 0.1285         & 0.1559	   & 0.0210	            & 0.2948    & 0.0229 \\
        &       & 200       & 0.3854	   & 0.0825		    & 0.0827	   & 0.0153		        & 0.1758    & 0.0142 \\
        &       & 500       & 0.4369	   & 0.1529		    & 0.0755	   & 0.0157		        & 0.1326	& 0.0107 \\
        &       & 1,000     & 0.3007	   & 0.1254		    & 0.0537	   & 0.0140	            & 0.0828	& 0.0171 \\
        &       & 2,000     & 1.3740	   & 0.9177		    & 0.0895	   & 0.0238		        & 0.1386	& 0.0288 \\
        &       & 5,000     & 0.0669	   & 0.0139		    & 0.0218	   & 0.0059		        & 0.0292	& 0.0088 \\

\hline

Insur19   & 19    & 100     & 0.6150	   & 0.1995		    & 0.0947	   & 0.0179		        & 0.1575	& 0.0213   \\
          &       & 200     & 0.4428	   & 0.1319		    & 0.0663	   & 0.0111		        & 0.1024	& 0.0162   \\
          &       & 500     & 0.2757	   & 0.1127		    & 0.0436	   & 0.0140		        & 0.0594	& 0.0099   \\
          &       & 1,000    & 0.4539	   & 0.3031		    & 0.0301	   & 0.0088		        & 0.0422	& 0.0134   \\
          &       & 2,000    & 0.0100	   & 0.0073		    & 0.0073	    & 0.0042		    & 0.0079	& 0.0029   \\
          &       & 5,000    & 0.0110	   & 0.0100		    & 0.0094	    & 0.0041		    & 0.0116	& 0.0043   \\

\hline

Child   & 20    & 100       & 0.4997	   & 0.1153		    & 0.0745	   & 0.0147		        & 0.1772	& 0.0146 \\
        &       & 200       & 0.1896	   & 0.0528		    & 0.0425	   & 0.0081		        & 0.0982	& 0.0101 \\
        &       & 500       & 0.2385	   & 0.0702		    & 0.0434	   & 0.0068		        & 0.0816	& 0.0123 \\
        &       & 1,000     & 0.1079	   & 0.0525		    & 0.0307	   & 0.0093		        & 0.0406	& 0.0080 \\
        &       & 2,000     & 0.0864	   & 0.0521		    & 0.0231	   & 0.0090		        & 0.0275	& 0.0083 \\
        &       & 5,000     & 0.0938	   & 0.0539		    & 0.0235	   & 0.0078		        & 0.0246	& 0.0066 \\

\hline

%\hline
\end{tabular}
}
%}

\hfill{}
%\end{small}
\end{center}
%\end{table}
\caption{Comparison of the PO-MCMC, the DOS \& the DDS in Terms of SAD}
\label{tb:bias_comparison_sad_1}
\end{table*}
%\end{center}

Since $n$ (the number of the variables) in each investigated data case is moderate,
we are able to use
REBEL, %\footnote{REBEL is available at http://www.cs.helsinki.fi/u/mkhkoivi/REBEL/.},
a C++ language implementation of the DP algorithm
%in
%\cite{Koivisto:06},
%\citep{Koivisto:06},
of \citet{Koivisto:06},
to get the exact posterior of every edge under
the assumption of
the order-modular prior.
(REBEL is available at
\url{http://www.cs.helsinki.fi/u/mkhkoivi/REBEL/}.)
Therefore we can use the criterion of the sum of the absolute differences (SAD)
%\cite{eaton:mur07b}
\citep{eaton:mur07b}
to
measure the feature learning performance for each data case:
\begin{equation*}
\verb"SAD" = \sum_{f} | p( f | D) - \hat{p}(f | D)|,
\end{equation*}
where $p( f | D)$ is the exact posterior of the investigated feature $f$,
and   $\hat{p}(f | D)$ is the corresponding estimator.
In Section \ref{sec-expr-DDS},
SAD is essentially $\sum_{ij} | p_{\prec}( i \rightarrow j | D) - \hat{p}_{\prec}(i \rightarrow j | D)|$,
%$p_{\prec}( i \rightarrow j | D)$ is the exact posterior of edge $i \rightarrow j$,
%and $\hat{p}_{\prec}(i \rightarrow j | D)$ is the corresponding estimator.
since the investigated feature is the edge feature $i \rightarrow j$ under the order-modular prior.
A smaller SAD will indicate a better performance in structure discovery.
Note that the criterion SAD is closely related to another criterion MAD (the mean of the absolute differences)
since MAD = SAD$/(n(n-1))$.
Thus, for each data case the conclusion based on the comparison of SAD values
is the same as the one based on the comparison of MAD values
since $n(n-1)$  is just a constant for each data case.

%\begin{center}
%\begin{table*}[ht]
%\begin{table*}
\begin{table*}[t]
%\begin{table}[t]
%\caption{The comparison of PO-MCMC, DOS and DDS in SAD}
%\caption{The comparison of PO-MCMC, DOS and DDS in terms of both SAD and time (time is in seconds).}
%\caption{Comparison of the PO-MCMC, the DOS \& the DDS in Terms of Time (Time Is in Seconds)}
%\label{tb:bias_comparison_time_1}
%\vskip 0.10in
\begin{center}
%\begin{small}
%{\small
\hfill{}

\scalebox{0.72}{
\begin{tabular}{ccr||r|r|r||rrrrrr}
\hline
%\hline
Name    & $n$   & $m$       & PO-MCMC   &  DOS \;    & DDS \; & \multicolumn{6}{c} {DDS}\\
        &       &           & $\hat{\mu}(T_{t})$\;\;\; & $\hat{\mu}(T_{t})$    & $\hat{\mu}(T_{t})$
        & $\hat{\mu}(T_{DP})$   & $\hat{\sigma}(T_{DP})$
        & $\hat{\mu}(T_{ord})$   & $\hat{\sigma}(T_{ord})$
        & $\hat{\mu}(T_{DAG})$  & $\hat{\sigma}(T_{DAG})$  \\
\hline

T-T-T	&10	   & 958	& 104.30	    & 1.96		& 1.73	    & 1.42	    & 0.0159	& 0.24	& 0.0066	& 0.06	  & 0.0017 \\
												
Glass	  & 11	   & 214	& 222.02	    & 1.52		& 1.25	    & 0.89	    & 0.0136	& 0.25	& 0.0076	& 0.09	   & 0.0013 \\
												
%BC-Wisc	  & 11	   & 699	& 226.93	    & 2.82		& 2.55	    & 2.19	    & 0.0109	& 0.24	& 0.0043	& 0.11	   & 0.0022 \\

Wine	  & 13	   & 178	& 374.17		& 2.56	    & 2.53	   & 1.63	    & 0.0121	& 0.35	& 0.0039	& 0.53	& 0.0024	\\

Housing	  & 14	   & 506	& 510.65		& 4.92		& 4.54	   & 3.88	    & 0.0143	& 0.40	& 0.0033	& 0.23	& 0.0020	\\

Credit	  & 16	   & 690	& 962.97	    & 30.90		& 30.93	    & 29.57	    & 0.2118	& 0.44	& 0.0055	& 0.90	   & 0.0122 \\

Zoo	      & 17	   & 101	& 1,331.80		& 13.75		& 21.90	   & 12.05	    & 0.0837	& 0.62	& 0.0039	& 9.20	& 0.1156	\\

%HouseVotes &17	   & 435	& 1,352.58		& 20.54		& 22.16		      & 18.83	     & 2.66  \\

Tumor	  & 18	   & 339	& 1,856.03	    & 44.99		& 60.21	    & 43.33	    & 1.0763	& 0.72	& 0.0052	& 16.12	   & 0.2941 \\
												
Vehicle	  & 19	   & 846	& 2,683.91	    & 149.08	& 149.23	& 147.35	& 1.2863	& 0.65	& 0.0079	& 1.20	   & 0.0219 \\

German	   & 21	   & 1,000	& 4,887.33		& 333.43	& 356.17	& 330.64	& 1.3304	& 0.97	& 0.0087	& 24.52	  & 0.4441 \\

\hline

Syn15   & 15    & 100       & 677.29        & 4.82		& 7.13	     & 3.49	     & 0.0824	& 0.50	& 0.0021	& 3.12	   & 0.0337\\
        &       & 200       & 677.47	    & 5.91		& 8.91	     & 4.59	     & 0.0473	& 0.47	& 0.0044	& 3.83	   & 0.0258	\\
        &       & 500       & 686.51	    & 8.56		& 8.44	     & 7.41	     & 0.1911	& 0.48	& 0.0100	& 0.53	   & 0.0050\\
        &       & 1,000     & 716.31	    & 13.01		& 12.52	     & 11.78	 & 0.0927	& 0.48	& 0.0115	& 0.24	   & 0.0022\\
        &       & 2,000     & 731.50	    & 21.70		& 21.24	     & 20.58	 & 0.5310	& 0.49	& 0.0086	& 0.15	   & 0.0016	\\
        &       & 5,000     & 731.05	    & 48.01		& 47.26	     & 46.63	 & 0.4110	& 0.47	& 0.0031	& 0.14	   & 0.0008	\\

\hline

Letter  & 17    & 100       & 1,322.43	     & 16.35	& 21.91	    & 14.64	    & 0.2160	& 0.64	& 0.0036	& 6.60	    & 0.0497 \\
        &       & 200       & 1,315.01	     & 19.74	& 20.46	    & 18.14	    & 0.0809	& 0.55	& 0.0026	& 1.74	    & 0.0234 \\
        &       & 500       & 1,338.33	     & 27.97	& 28.21	    & 26.35	    & 0.0489	& 0.51	& 0.0066	& 1.32	    & 0.0130 \\
        &       & 1,000     & 1,343.88	     & 39.45	& 39.03	    & 38.11	    & 0.3011	& 0.47	& 0.0051	& 0.43	    & 0.0089\\
        &       & 2,000     & 1,358.29	     & 61.75	& 61.44	    & 60.52	    & 0.5109	& 0.47	& 0.0078	& 0.42	    & 0.0063 \\
        &       & 5,000     & 1,610.37	     & 126.53	& 126.49	& 125.67	& 0.7967	& 0.52	& 0.0058	& 0.27	    & 0.0053\\

\hline

Insur19   & 19    & 100     & 2,616.56		 & 53.39	& 86.06	   & 51.06	    & 0.2520	& 0.86	& 0.0091	& 34.11	   & 0.9630	\\
          &       & 200     & 2,633.44		 & 62.12	& 77.44	   & 59.95	    & 0.3025	& 0.84	& 0.0083	& 16.61	   & 0.2278\\
          &       & 500     & 2,680.85	     & 80.70	& 85.03	   & 77.89	    & 0.7618	& 0.79	& 0.0082	& 6.32	   & 0.0535\\
          &       & 1,000    & 2,734.10	     & 106.37	& 109.39   & 104.63	    & 1.0958	& 0.89	& 0.0122	& 3.85	   & 0.0296	 \\
          &       & 2,000    & 2,915.60		 & 155.05	& 158.31   & 154.26	    & 3.7703	& 0.90	& 0.0154	& 3.13	   & 0.0155 \\
          &       & 5,000    & 3,445.84		 & 297.31	& 300.51   & 297.41	    & 4.7737	& 0.96	& 0.0090	& 2.11	   & 0.0091 \\

\hline

Child   & 20    & 100       & 3,710.49	     & 102.42	& 181.38	& 99.71	      & 0.3497	& 1.07	& 0.0109	& 80.56	   & 1.1980 \\
        &       & 200       & 3,717.10	     & 112.68	& 168.48	& 110.05	  & 0.1793	& 1.04	& 0.0078	& 57.36	   & 0.5335 \\
        &       & 500       & 3,757.76	     & 136.98	& 193.11	& 134.32	  & 0.4652	& 1.07	& 0.0111	& 57.69	   & 0.7276 \\
        &       & 1,000     & 3,799.47	     & 174.18	& 186.15	& 171.54	  & 1.9832	& 1.08	& 0.0104	& 13.50	   & 0.9244\\
        &       & 2,000     & 4,018.03	     & 241.48	& 254.26	& 238.37	  & 2.4952	& 1.15	& 0.0214	& 14.71	   & 0.4833 \\
        &       & 5,000     & 4,531.20	     & 443.54	& 455.30	& 441.64	  & 4.8785	& 1.16	& 0.0068	& 12.47	   & 0.4113	\\

\hline

%\hline
\end{tabular}
}
%}

\hfill{}
%\end{small}
\end{center}
%\end{table}
\caption{Comparison of the PO-MCMC, the DOS \& the DDS in Terms of Time (Time Is in Seconds)}
\label{tb:bias_comparison_time_1}
\end{table*}
%\end{center}

For fair comparison, %in this subsection,
in our algorithms
we used the K2 score
%\cite{heckerman:etal95a}
\citep{heckerman:etal95a}
and we set $q_i(U_i) = 1$  and $\rho_i(Pa_i) = 1/ \binom {n-1} {|Pa_i|} $ for each $i, U_i, Pa_i$, where $|Pa_i|$ is the size of the set $Pa_i$,
since such a setting is used in both BEANDisco and REBEL.

For the setting of the PO-MCMC,
according to the suggestion for the optimal setting %for the partial order MCMC method according to the experimental results
 from
 %the authors
 %\cite{niinimaki:etal11}. %cire for partial order,
 \citet{niinimaki:etal11},
we set the bucket size $b$ to be 10 for all the data cases except Tic-Tac-Toe.
The bucket size $b$ was set to be 9 for the data case Tic-Tac-Toe
because Tic-Tac-Toe has only 10 variables so that the setting $b = 10$
will cause the tool BEANDisco to throw a run-time error.
We ran the first $10,000$ iterations for ``burn-in,''
and then took $200$ partial order samples at intervals of $50$ iterations.
Thus, there were $20,000$ iterations in total.
(The time cost of each iteration in the PO-MCMC is $O(n^{k+1} + n^2 2^b n / b)$.)
In the PO-MCMC,
for each sampled partial order $P_i$, $p(f | D, P_i) $ is obtained by
$p(D, f, P_i) / p(D, P_i)$ $ = p(D, f, P_i) / p(D, f \equiv 1,P_i)$,
where $p(D, f, P_i)$ $ = \sum_{\prec \supseteq P_i} \sum_{G \subseteq \prec} f(G)$ $ p(\prec, G) p(D|G)$.
The notation $\sum_{\prec \supseteq P_i}$ means that
all the total order $\prec$'s that are linear extensions of the sampled partial order $P_i$
will be included to obtain $p(D, f, P_i)$.
For example, for a data set with $n = 20$,
since our bucket size $b = 10$,
there are $20! / (10! 10!) = 184,756$ total orders
that will be included for each sampled partial order $P_i$.
The inclusion of the information of a large number of total orders consistent with each sampled partial order
gives great learning power to the PO-MCMC method;
and such an inclusion can be efficiently computed by the algorithm
%in \citep{Parviainen:AISTAT2010}
of \citet{Parviainen:AISTAT2010}
with the assumptions of the order-modular prior and the maximum in-degree $k$.
Finally, for the PO-MCMC, the estimated posterior of each edge is computed using $\hat{p}_{\prec}(f|D) = (1/T) \sum_{i=1}^T p(f|D, P_i)$.

%In BEANDisco,
%the posterior of each edge is then computed essentially
%using Eq.~(\ref{eq-1-HR}) and (\ref{eq-3-HR}).

Because the to-be-learned feature is the edge feature,
we can also use our DOS algorithm for the comparison. %to investigate its performance.
For both the DOS algorithm and the DDS algorithm,
%we took $20,000$ order samples, i.e., $N_o = 20,000$.
we set $N_o = 20,000$, that is, $20,000$ (total) orders were sampled.
%For DDS algorithm, we further sampled one DAG for each sampled order, i.e., $N_d = 1$.
%For the second way, we then sampled 1 DAG for each sampled order and then used Eq.~(\ref{eq-2-HR}).
%We intend to see the degree of the performance decrease from the first way to the second way.
Theoretically,
we expect that the learning performance of the DOS should be
better than the performance of the DDS
because the additional approximation coming from the DAG sampling step is avoided by the DOS.
By listing the performance of the DOS,
we mainly intend to examine how much the performance of the DDS decreases
due to the additional approximation from sampling one DAG per order.
However, since
the DDS but not the DOS
is capable of learning non-modular features,
the comparison between the PO-MCMC method and the DDS method is our main task.

%Table \ref{tb:bias_comparison_sad_1} shows experimental results in terms of both SAD and total running time $T_t$
%for each data case with $n$ variables and $m$ instances.  %(i.e., $m$ is sample size).
Table \ref{tb:bias_comparison_sad_1} shows the experimental results in terms of SAD
for each data case with $n$ variables and $m$ instances,
while Table \ref{tb:bias_comparison_time_1} lists the running time costs
corresponding to Table \ref{tb:bias_comparison_sad_1}.
%(The supplementary experimental results additional to Table \ref{tb:bias_comparison_sad_1} are shown in Appendix \ref{sec-appendix-hist-DDS}.)
For each of the three methods,
we performed 15 independent runs for each data case.
The sample mean and the sample standard deviation of the 15 SAD values of each method,
denoted by $\hat{\mu}$(SAD) and  $\hat{\sigma}$(SAD) respectively,
are listed along each method in Table \ref{tb:bias_comparison_sad_1}.
%In addition, the corresponding sample mean of the total running time $T_t$, denoted by $\hat{\mu}(T_t)$,
Correspondingly, the sample mean of the total running time $T_t$ of each method, denoted by $\hat{\mu}(T_t)$,
is shown in Table \ref{tb:bias_comparison_time_1}.
(Precisely speaking, the reported total running time $T_t$ of the DDS method includes
both the time of running the three steps of the DDS
and
the relatively tiny $O(N_o)$ time cost of computing $\hat{p}_{\prec}(f|D)$ for each edge $f$ using Eq.~(\ref{eq-2-HR}) at the end.
Similarly, the reported total running time $T_t$ of the DOS method also includes
the relatively tiny $O(N_o)$ time cost of computing $\hat{p}_{\prec}(f|D)$ for each edge $f$ using Eq.~(\ref{eq-1-HR}) and Eq.~(\ref{eq-3-HR}) at the end.)
%In addition, Table \ref{tb:bias_comparison_time_1} lists one more column named as $\hat{\mu}(T_{\beta})$ for the mean time of computing all the $\beta'_i(Pa_i)$'s for each data case, which is within the first step of our algorithm and often constitutes a significant portion of the total running time when the sample size $m$ is large, as you can easily find in Table \ref{tb:bias_comparison_time_1}.
In addition, the sample mean and the sample standard deviation of the running time of
three steps of the DDS (including the DP step, the order sampling step and the DAG sampling step),
%DP step (Step 1) and DAG sampling step (Step 3) of DDS,
denoted by
$\hat{\mu}(T_{DP})$, $\hat{\sigma}(T_{DP})$,
$\hat{\mu}(T_{ord})$, $\hat{\sigma}(T_{ord})$,
$\hat{\mu}(T_{DAG})$ and $\hat{\sigma}(T_{DAG})$ respectively,
are listed in the last six columns in Table \ref{tb:bias_comparison_time_1}.
Note that we still show $\hat{\mu}(T_{DP})$,
the mean running time of the DP step of the DDS in the 15 independent runs,
though the DP step is not a random algorithm at all.
The running time of the DP step is not exactly the same in each run
due to the randomness from uncontrolled factors such as the internal status of the computer.
%This can be clearly seen from the fact that $\hat{\sigma}(T_{DP})$ is always positive.
By showing $\hat{\mu}(T_{DP})$,
we can clearly see the percentage of the total running time
%the running time of
that
the DP step typically takes %accounts for %on average
by comparing $\hat{\mu}(T_{DP})$ and $\hat{\mu}(T_t)$.

Tables \ref{tb:bias_comparison_sad_1} and \ref{tb:bias_comparison_time_1}
clearly
illustrate the performance advantage of our DDS method over the PO-MCMC method.
The overall time costs of our DDS based on
$20,000$ DAG samples
are much smaller than the corresponding costs of the PO-MCMC method
based on $20,000$ MCMC iterations in the partial order space.
Using much
%smaller
shorter
time,
our DDS method has its
$\hat{\mu}$(SAD)
%and  $\hat{\sigma}$(SAD)
much smaller than $\hat{\mu}$(SAD) from the PO-MCMC method
%for 16 out of all the 18 data cases.
%for 24 out of all the 28 data cases.
for 28 out of all the 33 data cases.
The five exceptional cases are
Glass,
Syn15 with $m = 2,000$, Syn15 with $m = 5,000$,
Insur19 with $m = 2,000$, and Insur19 with $m = 5,000$.
%where $\hat{\mu}$(SAD)  using our DDS method is not much smaller than the ones using PO-MCMC method
%(For Insur19 with $m = 2,000$,
(In both Glass and Insur19 with $m = 2,000$,
$\hat{\mu}$(SAD) using our DDS method is still smaller than the one using the PO-MCMC method
but their difference is not very large relatively to $\hat{\sigma}$(SAD) from the PO-MCMC method.)
%In these two exceptions, only $(\text{Our Method})^1$ but not $(\text{Our Method})^2$) has both the smaller $\hat{\mu}$(SAD) and the smaller $\hat{\sigma}$(SAD) than the partial order MCMC method.
Furthermore, since both $\hat{\mu}$(SAD) and $\hat{\sigma}$(SAD) are given,
by the two-sample $t$ test with unequal variances
%\cite{CB:StatisticalI}
\citep{CB:StatisticalI},
%with unequal variances,
we can conclude with
%very
strong evidence
%(with level 0.001) %also correct
(at the significance level $5 \times 10^{-3}$)
that the real mean of SAD using our DDS method is smaller than
the real mean of SAD using the PO-MCMC method for each of the 28 data cases.
For the exceptional data case Glass,
the p-value of the $t$ test is $0.0120$ so that
we can conclude at the significance level $0.05$ %$5 \times 10^{-2}$
that the real mean of SAD using our DDS method is smaller than
the real mean of SAD using the PO-MCMC method.
For each of the other four exceptions,
by the same $t$ test
we accept
(with the p-value $> 0.2$)
%(with p-value $> 0.24$) %also correct
the null hypothesis that there is no significant difference %.
in the real means of SAD from two methods.
%in terms of the real mean of SAD between using our method and using the partial order MCMC method.
Thus, the advantage of our DDS algorithm over the PO-MCMC method in learning
Bayesian networks of a moderate $n$ can be clearly seen,
though the value of the PO-MCMC method still remains for larger $n$
where our DDS algorithm is infeasible.

%In addition, Table \ref{tb:bias_comparison_time_1} shows that $\hat{\sigma}(T_t)$ using our method is always less than $5\%$ of the corresponding $\hat{\mu}(T_t)$ across all the 18 data cases, which indicates the variability of the running time of our method (based on the random sampling) is small.

In terms of the total running time of the DDS algorithm,
Table \ref{tb:bias_comparison_time_1} shows that the running time of the DP step always accounts for the largest portion.
The running time of the DAG sampling step is less than $81$ seconds to get $20,000$ DAG samples for all the 33 cases.
%which shows.
Though both the order sampling step and the DAG sampling step involve randomness,
the variability of their running time is actually small.
This can be seen from the ratio of $\hat{\sigma}(T_{ord})$ to $\hat{\mu}(T_{ord})$
(which is always less than $3.04\%$ for all the 33 cases)
and the ratio of $\hat{\sigma}(T_{DAG})$  to $\hat{\mu}(T_{DAG})$
(which is always less than $6.85\%$ for all the 33 cases).
The ratio of $\hat{\mu}(T_{DAG})$ to $\hat{\mu}(T_{ord})$ ranges from $0.25$ to $75.29$ across the 33 cases,
which is much smaller than the upper bound of the ratio of $O(n^{k+1} N_o)$ to $O(n^2 N_o)$.
This indicates that our time-saving strategy introduced in Remark~\ref{rm:time_DAG_sampling}
can effectively reduce the running time of the DAG sampling step.
In addition,
the running time of the DAG sampling step often decreases further when $m$ increases,
which can be clearly seen from
all the four data sets (Syn15, Letter, Insur19 and Child) with different values of $m$.
Take the data set Letter for example,
when $m$ increases from $100$ to $1,000$,
the corresponding $\hat{\mu}(T_{DAG})$ decreases from $6.60$ to $0.43$ second,
a $93.48\%$ of decrease.
In summary,
the effectiveness of our time-saving strategy introduced in Remark~\ref{rm:time_DAG_sampling}
has been clearly shown in Table \ref{tb:bias_comparison_time_1}.

%\subsection{SUPPLEMENTARY EXPERIMENTAL RESULTS FOR DDS}

%In this subsection, we present additional experimental results
Finally, we present the experimental results
for the DDS by
%fixing $m$  and
varying the sample size.
We first choose the data case Letter with $m = 500$ as an example.
For the DDS,
we tried sample size
$N_o$ $ = 5,000 \cdot i$,
where $i \in \{1, 2, \ldots, 10 \}$.
For each $i$, we independently ran the DDS $15$ times
to get the sample mean and the sample standard deviation of SAD
for the (directed) edge features.
For the PO-MCMC,
with the bucket size $b = 10$, %(the same setting as Section \ref{sec-expr-DDS}),
we ran totally $5,000 \cdot i$ MCMC iterations in the partial order space,
where $i \in \{1, 2, \ldots, 10 \}$.
For each $i$, we discarded the first $2,500 \cdot i$ MCMC iterations for ``burn-in''
and set the thinning parameter to be $50$
so that $50 \cdot i$ partial orders finally got sampled.
Again, for each $i$, we independently ran the PO-MCMC $15$ times
to get the sample mean and the sample standard deviation of SAD
for the edge features.

Figure
\ref{fig-errorbar_POMCMC_DDS_SAD_Letter_m500}
shows the SAD performance of the two methods with each $i$ in terms of the edge features,
where an error bar represents
one sample standard deviation $\hat{\sigma}$(SAD)
across $15$ runs from
a method
(the PO-MCMC or the DDS)
at each $i$.
For example, Figure
\ref{fig-errorbar_POMCMC_DDS_SAD_Letter_m500}
shows that when $i = 4$,
$\hat{\mu}$(SAD) $ = 0.1326$ and $\hat{\sigma}$(SAD) $ = 0.0107$ from our DDS algorithm;
while $\hat{\mu}$(SAD) $ = 0.4369$ and $\hat{\sigma}$(SAD) $ = 0.1529$ from the PO-MCMC method.
This exactly matches the results previously shown in Table \ref{tb:bias_comparison_sad_1}.
Correspondingly, Figure
\ref{fig-plot_POMCMC_DDS_Time_Letter_m500}
shows $\hat{\mu}(T_{t})$ (the sample mean of the total running time) of the PO-MCMC and the DDS with each $i$,
where the running time is in seconds.
The advantage of our DDS can be clearly seen by combining
Figures \ref{fig-errorbar_POMCMC_DDS_SAD_Letter_m500}
and \ref{fig-plot_POMCMC_DDS_Time_Letter_m500}.
For each $i \in \{1, 2, \ldots, 10 \}$,
the real mean of SAD from the DDS is significantly smaller than the one from the PO-MCMC
with the p-value
%$ < 1 \times 10^{-4}$  %for letter_samp500
$ < 5 \times 10^{-3}$  %for all the investigated cases
returned by the two-sample $t$ test (with unequal variances).
In terms of the running time,
the total running time of the DDS is very short relative to the one of the PO-MCMC.
For example,
$\hat{\mu}(T_{t})$ of the DDS
increases %(linearly)
with respect to $i$
and reaches only $29.55$ seconds
at $i = 10$.
%$29.86$ seconds are even smaller than $9\%$ of $\hat{\mu}(T_{t})$
This is shorter than $9\%$ of $\hat{\mu}(T_{t})$
of the PO-MCMC at $i = 1$, which is $336.09$ seconds.
Therefore,
the learning performance of the DDS with each sample size is significantly
better than the one of the PO-MCMC for the data case Letter  with $m=500$.

We also performed the experiment with the same experimental settings for the data cases
Tic-Tac-Toe, Wine,
%Housing,
Child with $m = 500$, and German.
Please refer to the supplementary material for the experimental results.
(The supplementary material is available at
\url{http://www.cs.iastate.edu/~jtian/Software/BNLearner/BN_Learning_Sampling_Supplement.pdf}. )
The same conclusion about the learning performance can be clearly drawn
by examining
the figures shown in the supplementary material.
\begin{comment}
each figure from
Figures
  \ref{fig-errorbar_POMCMC_DDS_SAD_Tic_Tac_Toe},
  \ref{fig-errorbar_POMCMC_DDS_SAD_Wine},
  %\ref{fig-errorbar_POMCMC_DDS_SAD_Housing},
  \ref{fig-errorbar_POMCMC_DDS_SAD_Child_m500},
  \ref{fig-errorbar_POMCMC_DDS_SAD_German},
and the corresponding figure from
Figures
  \ref{fig-plot_POMCMC_DDS_Time_Tic_Tac_Toe},
  \ref{fig-plot_POMCMC_DDS_Time_Wine},
  %\ref{fig-plot_POMCMC_DDS_Time_Housing},
  \ref{fig-plot_POMCMC_DDS_Time_Child_m500},
  \ref{fig-plot_POMCMC_DDS_Time_German}.
\end{comment}

\begin{figure}
\centering
%\begin{minipage}{.42\textwidth}
\begin{minipage}[b]{0.46\linewidth}
  %\centering
  \includegraphics[width=1 \linewidth]{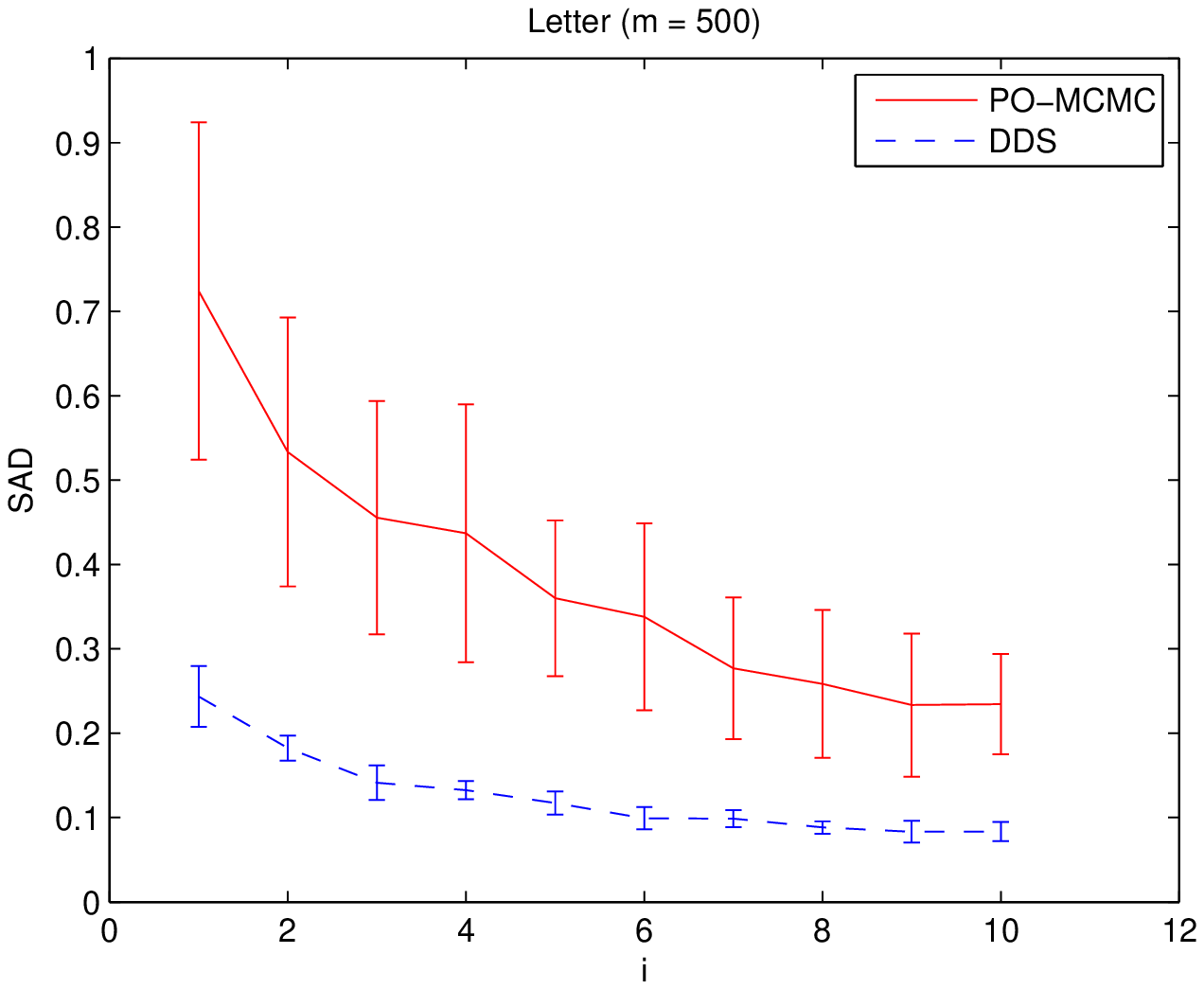}
  %\captionof{figure}{Plot of the SAD Performance of the PO-MCMC and the DDS for Letter ($m = 500$)}
  \caption{Plot of the SAD Performance of the PO-MCMC and the DDS for Letter ($m = 500$)}
  \label{fig-errorbar_POMCMC_DDS_SAD_Letter_m500}
\end{minipage}
%\begin{minipage}{.42\textwidth}
\qquad
\begin{minipage}[b]{0.46\linewidth}
  %\centering
  \includegraphics[width=1 \linewidth]{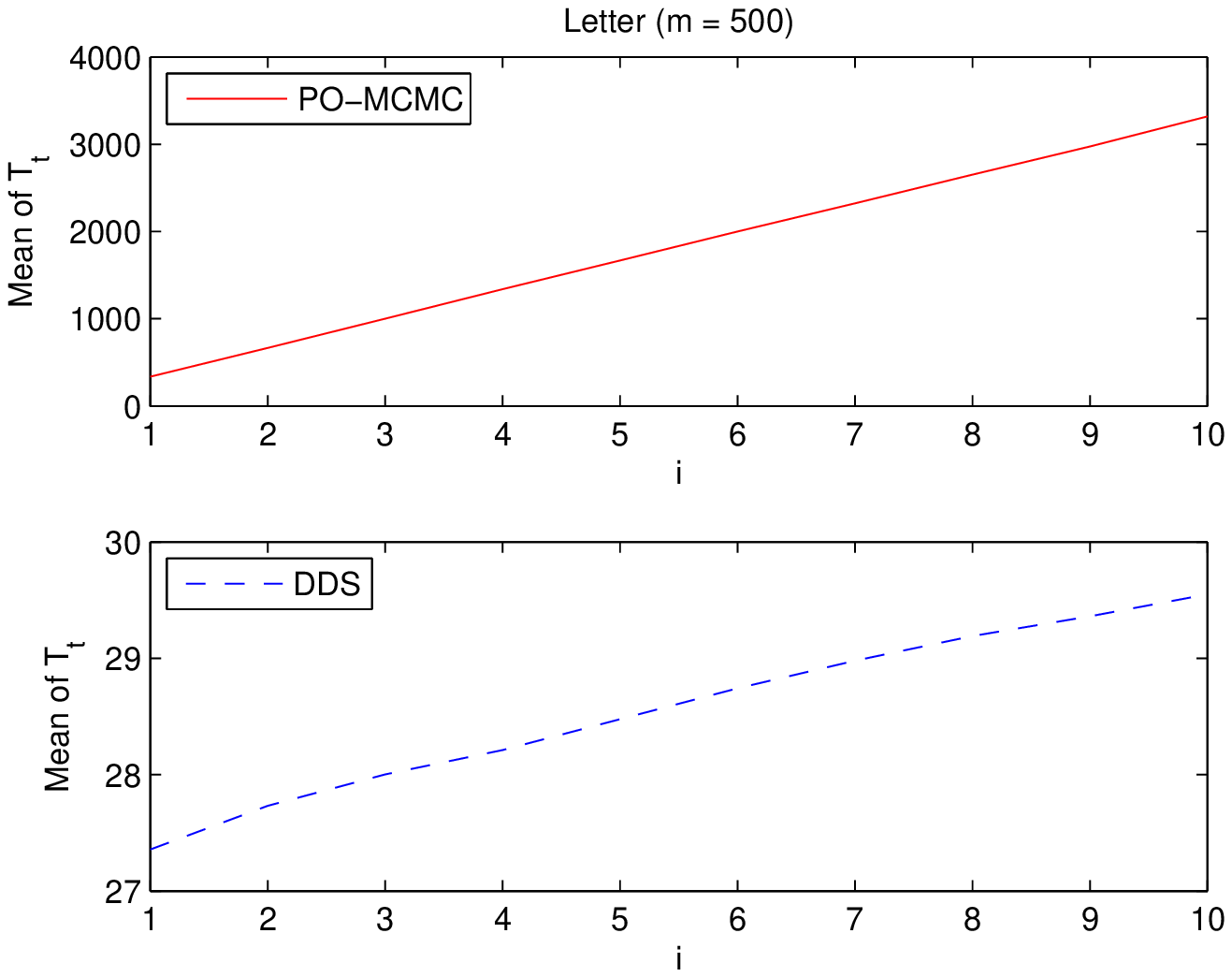}
  %\captionof{figure}{Plot of the Total Running Time of the PO-MCMC and the DDS for Letter ($m = 500$)}
  \caption{Plot of the Total Running Time of the PO-MCMC and the DDS for Letter ($m = 500$)}
  \label{fig-plot_POMCMC_DDS_Time_Letter_m500}
\end{minipage}
\end{figure}

\subsection{Experimental Results for the IW-DDS}
\label{sec-expr-IWDDS}
In this subsection,
we compare our IW-DDS algorithm with the Hybrid MCMC (i.e., DP+MCMC) method
%\cite{eaton:mur07b}
\citep{eaton:mur07b}
and the $K$-best algorithm
%\cite{tian:he2010},
\citep{tian:he2010},
two state-of-the-art methods
that can estimate the posteriors of any features %of a moderate-size BN
without
%the bias from
the order-modular prior assumption.
The implementation of the Hybrid MCMC
(called BDAGL)
%\footnote{BDAGL is available at http://www.cs.ubc.ca/$\sim$murphyk/Software/BDAGL/.
%The original BDAGL has very expensive time cost and space cost for the DP+MCMC method because it computes the local scores of all the $n  %2^{n-1}$ possible families and stores them in an array.
%Therefore, we have updated the original Matlab code in BDAGL so that only the local scores of all the families having no more than a fixed %in-degree are computed and stored in a hash table.})
and the implementation of the $K$-best
(called KBest)
%\footnote{K-best is available at http://www.cs.iastate.edu$\sim$/rhe/KBest$_{-}$POSTER/.})
are both made available online by their corresponding authors.
(BDAGL is available at \url{http://www.cs.ubc.ca/~murphyk/Software/BDAGL/}.)
(KBest is available at \url{http://www.cs.iastate.edu/~jtian/Software/UAI-10/KBest.htm}.)
%(KBest is available at http://www.cs.iastate.edu/$\sim$rhe/KBest/.)
%(KBest      is available at http://www.cs.iastate.edu/$\sim$rhe/KBest/.)

%Again,
%since the size of the experimented problems is moderate,
Since $n$ in each investigated data case is moderate,
we are able to use
POSTER, %\footnote{POSTER is available at http://www.cs.iastate.edu/$\sim$rhe/POSTER/.},
a C++ language implementation of the DP algorithm
%\cite{tian:he2009}
%\citep{tian:he2009}
of \citet{tian:he2009}
to get $p_{\nprec}( i \rightarrow j | D)$, the exact posterior of each edge $i \rightarrow j$
under the structure-modular prior.
(POSTER is available at
\url{http://www.cs.iastate.edu/~jtian/Software/UAI-09/Poster.htm}.)
%http://www.cs.iastate.edu/$\sim$jtian/Software/UAI-09/Poster.htm.)
%(POSTER is available at http://www.cs.iastate.edu/$\sim$rhe/POSTER/.)
Therefore we can use the SAD criterion
%which is essentially $\sum_{ij} | p_{\nprec}( i \rightarrow j | D) $ $- \hat{p}_{\nprec}(i \rightarrow j | D)|$ in this subsection,
($\sum_{ij} | p_{\nprec}( i \rightarrow j | D) $ $- \hat{p}_{\nprec}(i \rightarrow j | D)|$)
to measure the performance of these three methods in the structural learning for each data case.

For fair comparison,
%for the experiments in this subsection,
in our algorithm we used the BDeu score
%\cite{heckerman:etal95a}
\citep{heckerman:etal95a}
with equivalent sample size 1
and set $ p_i(Pa_i) (\equiv \rho_i(Pa_i)) \equiv  1$,
%to make the corresponding structure prior relatively close to a uniform structure prior $p(G)$ \cite{eaton:mur07b},
since these settings are also used in POSTER  %, BDAGL
and the implementation of the $K$-best algorithm.

\begin{table*}[t]
%\caption{The comparison of methods with bias correction in SAD}
%\caption{Comparison of the DP+MCMC, the $K$-best \& the IW-DDS in Terms of SAD}  %and $\Delta$}
%\label{tb:unbiased_comparison_sad_1}
%\vskip 0.10in
\begin{center}

%{\small
\hfill{}

\scalebox{0.71}{
\begin{tabular}{ccr||r|rr|r|rr||r|rr}
\hline
Name    & $n$   & $m$       & DP \;\;   & \multicolumn{2}{|c|} {DP+MCMC}              & \multicolumn{1}{|c|} {$K$-best} &\multicolumn{2}{|c||} {IW-DDS} & \multicolumn{1}{|c|}{$K$-best}   & \multicolumn{2}{|c} {IW-DDS} \\
        &       &           & SAD \;    & $\hat{\mu}$(SAD) & $\hat{\sigma}$(SAD)    & \multicolumn{1}{|c|} {SAD}      & $\hat{\mu}$(SAD)  & $\hat{\sigma}$(SAD)     & $\Delta$\;\;\;\;\;\;  & $\hat{\mu}$($\Delta$)\;\;\; & $\hat{\sigma}$($\Delta$)\;\;\; \\
\hline

T-T-T	& 10     &958		& 0.1651	& 15.0079	   & 2.9877		& 1.4194		
& 0.0227	& 0.0102		& 6.943E-01		& 9.935E-01	    & 8.636E-04 \\
																			
Glass	& 11	& 214		& 1.5444	& 0.3587	    & 0.4599	& 0.0904		
& 0.0381	& 0.0019		& 9.780E-01		& 9.901E-01	    & 6.368E-04 \\

%BC-Wisc	& 11	& 699		& 1.5255	& 0.1922	    & 0.0748	& 0.0144		
%& 0.0187	& 0.0030		& 9.920E-01		& 9.924E-01	    & 1.100E-03 \\

Wine	& 13	& 178		& 1.4786	& 0.4605	   & 0.2968		& 0.2011		
& 0.1041	& 0.0075		& 9.023E-01		& 9.670E-01	    & 1.700E-03	 \\

Housing	& 14	& 506		& 5.6478	& 8.0000	   & 3.3408		& 9.1179		
& 4.8276	& 0.0624		& 2.880E-02		& 1.096E-01	    & 1.200E-03 \\

Credit	& 16	& 690		& 4.0580		& 5.0261	    & 2.3482	& 5.1492		
& 2.9148	& 0.0336		& 5.010E-02		& 1.793E-01	    & 1.700E-03 \\

Zoo	    & 17	& 101		& 8.2142	   & 32.4189   & 10.0953	& 35.1215		
& 19.7025	& 4.0767		& 3.815E-08		& 1.652E-11	   & 1.211E-11 \\

%HouseVotes & 17	& 435		& 5.6144	& 21.2448   & 5.0364		& 21.0690		
%&11.5512	0.5665		6.893E-04		1.803E-04	1.079E-05

Tumor	& 18	& 339		& 5.1536		& 17.5104	     & 7.9198	& 20.5793		
& 10.3139	& 0.6950		& 2.940E-05		& 3.619E-06	    & 3.586E-07 \\
																			
Vehicle	& 19	& 846		& 3.5759		& 3.8234	     & 0.4011	& 3.3683		
& 0.4648	& 0.0194		& 5.387E-01		& 9.432E-01	    & 2.600E-03 \\

German	& 21	& 1,000		& 3.7261	    & 5.0207	   & 3.2223		& 5.5902		
& 0.9891	& 0.0449		& 7.800E-02		& 7.422E-01	    & 7.200E-03	 \\

\hline

Syn15   & 15    & 100       & 4.9321	& 12.8705	& 7.6384		& 11.8685	
& 10.1341	& 0.1495		& 1.604E-04		& 1.183E-04	   & 4.664E-06	 \\

        &       & 200       & 3.2557	& 4.5090	& 2.5875		& 7.5232	
& 4.9079	& 0.0583		& 2.700E-03		& 7.300E-03	    & 1.265E-04\\

        &       & 500       & 6.9798	& 5.5466	& 1.9175		& 4.4379	
& 4.2965	& 0.1619		& 1.526E-01		& 2.388E-01	    & 4.900E-03	\\

        &       & 1,000     & 1.3000	& 0.3974	& 0.3299		& 0.0848	
& 0.0498	& 0.0048		& 9.699E-01		& 9.843E-01	    & 8.909E-04	\\

        &       & 2,000     & 1.7192	& 1.8263	& 1.7095		& 0.3701	
& 0.1081	& 0.0147		& 8.521E-01		& 9.703E-01	    & 3.300E-03 \\

        &       & 5,000     & 1.9473	& 0.0304	& 0.0094		& 8.89E-04	
& 0.0022	& 0.0002		& 9.998E-01		& 9.994E-01	    & 9.821E-05 \\

\hline
Letter  & 17    & 100       & 9.2140	& 27.1507   & 4.0940		& 24.4313	
& 15.8780	& 0.2764		& 2.908E-04		& 1.621E-04	   & 5.336E-06 \\	

        &       & 200       & 7.2855	& 15.1587   & 3.5615		& 9.4512	
& 6.7936	& 0.1191		& 7.800E-03		& 1.220E-02	    & 3.341E-04	\\

        &       & 500       & 6.0961	& 3.4637	   & 4.6789		& 1.7237	
& 0.6347	& 0.0119		& 6.948E-01		& 8.808E-01	    & 3.000E-03	\\

        &       & 1,000     & 0.6394	& 0.1761	   & 0.0166		& 0.0837	
& 0.0766	& 0.0039		& 9.834E-01		& 9.864E-01	    & 8.678E-04 \\

        &       & 2,000     & 2.3913	& 3.5085	   & 3.1132		& 2.0976	
& 0.1338	& 0.0213		& 6.859E-01		& 9.756E-01	    & 2.700E-03	\\

        &       & 5,000     & 0.8407	& 0.1182	   & 0.0442		& 0.0160 	
& 0.0072	& 0.0005		& 9.948E-01		& 9.972E-01	    & 2.077E-04	\\

\hline

Insur19	& 19	& 100		& 5.3356	& 9.4318	   & 3.9576		& 11.7779		
& 6.5062	& 0.0891		& 6.000E-03		& 2.010E-02	    & 4.767E-04	\\

        &       & 200		& 5.9844	& 5.5465	   & 2.7295		& 2.2572		
& 1.4630	& 0.0557		& 4.495E-01		& 7.049E-01	    & 1.120E-02	\\

        &       & 500		& 1.8274	& 0.3605	   & 0.2287		& 0.4970		
& 0.1328	& 0.0105		& 7.464E-01		& 9.379E-01	    & 3.000E-03 \\

        &       & 1,000		& 1.7386	& 0.2186	   & 0.0762		& 0.7498		
& 0.0623	& 0.0111		& 5.866E-01		& 9.785E-01	    & 2.600E-03	\\

        &       & 2,000		& 1.2737	& 0.1217	   & 0.0380		& 0.0174		
& 0.0062	& 0.0012		& 9.900E-01		& 9.976E-01	    & 4.864E-04 \\
	
        &       & 5,000		& 1.9511	& 0.1765	   & 0.0594		& 0.0103			
& 0.0092	& 0.0010		& 9.970E-01		& 9.973E-01	    & 2.086E-04	\\
	
\hline

Child   & 20    & 100       & 6.9783	& 11.8987	& 3.1086		& 11.6189	
& 7.0304	& 0.0909		& 7.848E-04		& 2.700E-03	    & 1.066E-04	\\

        &       & 200       & 3.2826	& 4.7066	   & 4.3749		& 5.0729	
& 2.8510	& 0.0192		& 4.800E-03		& 1.506E-01	    & 1.200E-03 \\

        &       & 500       & 2.5580	& 2.4716	   & 1.3489		& 1.5304	
& 0.5416	& 0.0305		& 3.582E-01		& 8.222E-01	    & 5.100E-03 \\

        &       & 1,000     & 2.4708	& 2.6061	   & 2.2909		& 0.7066	
& 0.1499	& 0.0198		& 7.013E-01		& 9.545E-01	    & 3.400E-03	\\

        &       & 2,000     & 2.3330	& 1.4286	   & 1.2290		& 1.5279	
& 0.0662	& 0.0161		& 6.509E-01		& 9.828E-01	    & 2.800E-03	\\

        &       & 5,000     & 2.0365	& 1.2533	   & 1.7313		& 0.8783	
& 0.0150	& 0.0013		& 8.209E-01		& 9.940E-01	    & 4.696E-04	\\

\hline

\end{tabular}
}
%}

\hfill{}
\end{center}
\caption{Comparison of the DP+MCMC, the $K$-best \& the IW-DDS in Terms of SAD}
\label{tb:unbiased_comparison_sad_1}
\end{table*}

%For DP+MCMC (100,000 + 30,000 * 3)
%\begin{center}
%\begin{table*}[ht]
\begin{table*}[t]
%\caption{The comparison of methods with bias correction in SAD}
%\caption{Comparison of the DP+MCMC, the $K$-best \& the IW-DDS in Terms of Time (Time Is in Seconds)}
%\label{tb:unbiased_comparison_time_1}
%\vskip 0.10in
\begin{center}
%\begin{small}
%{\small
\hfill{}

\scalebox{0.705}{
\begin{tabular}{ccr||r|r|r||rrrrrr}
\hline
Name    & $n$   & $m$       & \multicolumn{1}{|c|} {DP+MCMC}                  &  \multicolumn{1}{|c|} {$K$-best}     &  \multicolumn{1}{|c||}{IW-DDS} &  \multicolumn{6}{|c} {IW-DDS} \\
        &       &           & \multicolumn{1}{|c|}{$\hat{\mu}(T_{t})$}        &  \multicolumn{1}{|c|}{$T_{t}$}       &  \multicolumn{1}{|c||}{$\hat{\mu}(T_{t})$}    &
        \multicolumn{1}{c}{$\hat{\mu}(T_{DP})$}     &
        \multicolumn{1}{c}{$\hat{\sigma}(T_{DP})$}  &
        \multicolumn{1}{c}{$\hat{\mu}(T_{ord})$}     &
        \multicolumn{1}{c}{$\hat{\sigma}(T_{ord})$}  &
        \multicolumn{1}{c}{$\hat{\mu}(T_{DAG})$}    &
        \multicolumn{1}{c}{$\hat{\sigma}(T_{DAG})$}     \\
\hline

T-T-T   & 10	& 958	& 1.29 + 1,032.40		& 8.37	
                        & 3.28	     & 1.27	         & 0.0060	& 1.57	    & 0.0109	& 0.44	    & 0.0128 \\
																	
Glass	& 11	& 214		& 1.04 + 1,037.26		& 18.13
                        & 1.72	     & 0.98	         & 0.0210	& 0.50	    & 0.0054	& 0.24	    & 0.0018 \\
																	
%BC-Wisc	& 11	& 699		& 2.31 + 1,117.00		& 19.36	
%                        & 6.05	     & 2.24	         & 0.0211	& 3.49	    & 0.0514	& 0.32	    & 0.0021 \\

Wine	& 13	& 178       & 2.44 + 1,127.80		  & 141.20	
                        & 3.52	     & 2.15	        & 0.0199	& 0.63	    & 0.0061	& 0.74	    & 0.0057	 \\

Housing	& 14	& 506       & 4.83 + 1,421.00		  & 323.37		
                        & 5.67	     & 4.21	        & 0.0813	& 0.63	    & 0.0042	& 0.52	    & 0.0043 \\

Credit	& 16	& 690		& 33.73 + 1,476.90		& 2,073.41	
                        & 35.20	    & 29.30	        & 0.3322	& 0.81	    & 0.0055	& 4.94	    & 1.6191 \\

Zoo	    & 17	& 101       & 22.33 + 2,107.50		  & 5,531.81	
                        & 26.16	     & 12.49	    & 0.1853	& 1.11	    & 0.0048	& 12.05	     & 0.1898	\\

%HouseVotes & 17	& 435       & 28.85 + 1,775.50		  & 5,551.35	& 26.07	    & 19.17	     & 5.36 \\

Tumor	& 18	& 339		& 60.86 + 1,799.99		& 18,640.09	
                        & 87.35	    & 39.49	        &   0.4419	&    1.19	&  0.0184	& 46.31	     & 0.2937 \\
																	
Vehicle	& 19	& 846		& 207.27 + 1,886.10		& 17,126.45
                        & 171.75	&  160.55	     & 0.9310	& 1.49	    & 0.0175	&  9.67	    & 0.0665 \\

German	& 21	& 1,000     & 600.19 + 1,849.90		  & 10,981.29	
                        & 540.61	  & 392.25	    & 2.8656	& 1.68	    & 0.0207	& 146.65	  & 1.4290 \\

\hline

Syn15   & 15    & 100       & 5.21 + 1,284.00	      & 889.97
                        & 9.41	     & 3.56	         & 0.0667	& 0.81	    & 0.0048	& 4.80         & 0.0273	 \\

        &       & 200       & 6.28 + 1,286.20	      & 901.23
                        & 10.29	     & 4.76	         & 0.2090	& 0.80	    & 0.0111	& 4.53	       & 0.0256 \\

        &       & 500       & 9.64 + 1,336.80	      & 899.42
                        & 9.59	     & 7.50	         & 0.0821	& 0.85	    & 0.0051	& 1.08	       & 0.0128	 \\

        &       & 1,000     & 13.69 + 1,364.60	      & 910.76
                        & 13.35	     & 12.15	     & 0.7578	& 0.85	    & 0.0156	& 0.33	       & 0.0038 \\

        &       & 2,000     & 22.64 + 1,372.10	      & 907.02
                        & 21.98	     & 20.81	     & 0.5200	& 0.85	    & 0.0061	& 0.30	       & 0.0019\\

        &       & 5,000     & 54.79 + 1,356.70		  & 932.96
                        & 52.07	     & 47.80	     & 1.0887	& 3.99	    & 0.0235	& 0.28	       & 0.0023 \\

\hline

Letter  & 17    & 100       & 25.53 + 1,572.60	      & 7,639.02	
                        & 20.27	     & 15.83	     & 0.0601	& 1.15	    & 0.0062	& 2.83	      & 0.0292	 \\

        &       & 200       & 30.01 + 1,576.80	      & 7,966.25	
                        & 25.87	     & 20.22	     & 0.1769	& 1.09	    & 0.0069	& 4.21	      & 0.0304\\

        &       & 500       & 39.58 + 1,598.90	      & 8,257.60
                        & 31.88	     & 29.84	     & 0.1575	& 0.95	    & 0.0055	& 1.01	      & 0.0123	 \\

        &       & 1,000     & 52.85 + 1,575.60	      & 8,380.57
                        & 44.48	     & 42.93	     & 0.4835	& 0.98	    & 0.0093	& 0.56	      & 0.0100 \\

        &       & 2,000     & 77.48 + 1,591.00	      & 7,619.29
                        & 69.14	     & 66.34	     & 0.5184	& 2.01	    & 0.0241	& 0.78	      & 0.0068 \\

        &       & 5,000     & 157.69 + 1,636.40	      & 8,134.85
                        & 136.95	 & 134.94	     & 1.6127	& 1.36	    & 0.0097	& 0.65	      & 0.0067 \\

\hline

Insur19 & 19    & 100       & 101.47 + 1,828.00		  & 6,745.41	
                        & 112.28	   & 55.85	    & 0.1540	& 1.41	    & 0.0117	& 54.71	        & 0.5438\\

        &       & 200       & 113.79 + 1,896.10		  & 6,783.76	
                        & 82.52	       & 68.12	    & 0.4187	& 1.45	    & 0.0120	& 12.90	        & 0.1329 \\

        &       & 500       & 137.18 + 1,864.40		  & 6,894.15	
                        & 98.96	       & 91.44	     & 0.4547	& 1.43	    & 0.0128	& 6.07	        & 0.0358\\

        &       & 1,000     & 168.55 + 1,862.30		  & 6,966.90	
                        & 133.87	   & 123.42	     & 0.8007	& 1.44	    & 0.0157	& 9.00	        & 0.0601 \\

        &       & 2,000     & 226.36 + 1,781.70		  & 7,277.60		
                        & 185.02	   & 177.66	     & 1.3165	& 3.30	    & 0.0483	& 4.06	        & 0.0356 \\

        &       & 5,000     & 380.78 + 1,814.80		  & 7,061.76	
                        & 336.03	   & 329.78	     & 4.5841	& 1.89	    & 0.0220	& 4.34	        & 0.0713	 \\

\hline

Child   & 20    & 100       & 203.44 + 1,785.10	      & 15,085.86
                        & 223.62	   & 106.01	     & 0.3976	& 1.67	      & 0.0176	& 115.60	  & 1.3183	 \\

        &       & 200       & 215.70 + 1,760.80	      & 14,222.86
                        & 225.76	   & 119.82	     & 1.8604	& 1.69	      & 0.0107	& 104.09	  & 0.9659	 \\

        &       & 500       & 248.17 + 1,818.70	      & 14,016.94
                        & 214.91	   & 149.73	     & 0.8344	& 1.67	      & 0.0183	& 63.48	       & 0.5783	 \\

        &       & 1,000     & 292.40 + 1,817.20	      & 15,504.82
                        & 232.11	   & 193.18	     & 1.1041	& 1.72	      & 0.0135	& 37.20	       & 0.3176 \\

        &       & 2,000     & 371.12 + 1,841.40	      & 16,109.91
                        & 306.42	   & 268.78	     & 2.3209	& 1.82	      & 0.0165	& 35.81	       & 0.3680	 \\

        &       & 5,000     & 589.99 + 1,846.40	      & 15,372.61
                        & 524.63	   & 483.90	     & 6.7403	& 1.93	      & 0.0160	& 38.80	       & 0.5411 \\

\hline

\end{tabular}
}
%}

\hfill{}
%\end{small}
\end{center}
\caption{Comparison of the DP+MCMC, the $K$-best \& the IW-DDS in Terms of Time (Time Is in Seconds)}
\label{tb:unbiased_comparison_time_1}
\end{table*}
%\end{center}

As for the DP+MCMC,
we note that most part of its implementation in BDAGL tool is written in Matlab,
%which is different from
%%the tool of $K$-best method and our tool which are written in C++.
%the implementations of the $K$-best algorithm and our IW-DDS method, which are written in C++.
whereas both the K-best and the IW-DDS are implemented in C++.
In order to make relatively fair comparison
in terms of the running time, % with $k$-best method and our method,
we used REBEL tool, a C++ implementation of the DP algorithm
of \citet{Koivisto:06},
to perform the computation in the DP phase of the DP+MCMC;
but for fair comparison we changed its scoring criterion into
the BDeu score with equivalent sample size 1
%and set $q_i(U_i) = 1$  and $\rho_i(Pa_i) = 1$).
and set $q_i(U_i) \equiv 1$ and $\rho_i(Pa_i) \equiv 1$.
%We tried two different $\rho_i(Pa_i)$ prior settings for our data cases:
%$\rho_i(Pa_i) \equiv 1$
%and $\rho_i(Pa_i) = 1/ \binom {n-1} {|Pa_i|} $.
%Consistent with the experimental results
%shown by \citet{eaton:mur07b},
%our experiments also show that the former prior setting
%introduces smaller biases in most of the data cases.
%(Please refer to the two columns DP and DP' in Table \ref{tb:unbiased_comparison_sad_1} and later discussion.)
%Thus, in order to get better performance of the DP+MCMC method,
%we used the former prior setting $\rho_i(Pa_i) \equiv 1$ in the DP phase of the DP+MCMC.
To perform the computation in the MCMC phase,
%we had to use %BDAGL's
we used
the Matlab implementation of BDAGL
\footnote{
The original BDAGL was found to get an out-of-memory error for any data case with more than 19 variables in our experiments.
This is because the original BDAGL intends to pre-compute the local scores of all the $n 2^{n-1}$ possible families and store them in an array for the later usage in both the DP phase and the MCMC phase.
To solve this out-of-memory issue, we have updated the original Matlab code in BDAGL and provided the BDAGL-New package which is
also available at
\url{http://www.cs.iastate.edu/~jtian/Software/BNLearner/BNLearner.htm}.
%http://www.cs.iastate.edu /$\sim$rhe/BNLearner.
The main update is that, with the assumed maximum in-degree, only the local scores of all the families whose sizes are no more than the assumed maximum in-degree are pre-computed and stored in a hash table.
With BDAGL-New, the experiments for all the data cases in this paper can be performed without any error.
%For the purpose of this paper, we have incorporated the max-indegree assumption into the original BDAGL software to generate the corresponding BDAGL-New software.
%With the max-indegree assumption, the time cost and space cost of the original BDAGL can be largely reduced and the fair comparison of the experimental results can be performed.
%On one hand, in BDAGL-New only the local scores of all the families whose sizes are no more than the assumed maximum in-degree are pre-computed and stored in a hash table.
%This largely reduces the time cost and space cost of the original BDAGL
%since the original BDAGL pre-computes the local scores of all the $n 2^{n-1}$ possible families and stores them in an array.
%On the other hand, the global proposal from original BDAGL can propose any DAG in each MCMC iteration. But the global proposal from BDAGL-New %will only propose a DAG with the max-indegree restriction in each MCMC iteration.
%Most part of BDAGL is written in Matlab.
%The original BDAGL code for the DP+MCMC method has very expensive time cost and space cost
%because it pre-computes the local scores of all the $n 2^{n-1}$ possible families and stores them in an array.
%Therefore, we have updated the original Matlab code in BDAGL
%so that only the local scores of all the families whose sizes are no more than the assumed maximum in-degree are pre-computed and stored in a %hash table.
}
and
we ran it under Windows 7 on an ordinary laptop with 2.40 Intel Core i5 CPU and 4.0 GB memory.
The MCMC used the pure global proposal (with local
proposal choice $\beta$ = 0) since such a setting was reported
%in
%\cite{eaton:mur07b}
%\citep{eaton:mur07b}
by \citet{eaton:mur07b}
to have the best performance for edge discovery when up to about $190,000$ MCMC iterations were performed in their
experimental results.
We ran totally 190,000 MCMC iterations each time and discarded the first 100,000 iterations as the burn-in period.
Then we set the thinning parameter to be 3 to get the final 30,000 DAG samples.
As a result, the time statistics of the DP phase (the number before + sign in Table \ref{tb:unbiased_comparison_time_1})
but not the MCMC phase (the number after + sign)
can be directly compared with the ones of the other two methods. % in Table \ref{tb:unbiased_comparison_time_1}.
For each data case, we performed 20 independent MCMC runs based on the DP outcome from REBEL  to get the results.

For our IW-DDS,
we set $N_o = 30,000$. % and $N_d = 1$.
We performed 20 independent runs for each data case to get the results.
For the $K$-best,
note that its SAD is fixed since there is no randomness in the computed results.
So we only ran it once to get the result.
We set $K$ to be $100$
for Tic-Tac-Toe, Glass, Wine, Housing, Credit, Zoo, Syn15, and Letter,
that is,
we got the $100$ best DAGs from Tic-Tac-Toe, Glass, Wine, Housing, Credit, Zoo,
each of the six cases of Syn15,
and each of the six cases of Letter.
%We set $K$ to be only $40$ for Insur19
%because our experiments showed that for Insur19 the K-best program ran out of memory with the setting of $K > 40$.
We set $K$ to be only $80$ for Tumor
because our experiments showed that for Tumor the $K$-best program ran out of memory with $K > 80$.
Due to the same out-of-memory issue,
we set $K$ to be only $40$ for Vehicle and Insur19;
and we set $K$ to be only $20$ for Child and $9$ for German.
The fact that $K$ can only take a value no greater than $40$ for $n$ $\geq 19$ in our experiments
confirms our claim about the
computation problem of the $K$-best algorithm in terms of its space cost.

%In addition, we added one column named $DP$ in Table \ref{tb:unbiased_comparison_sad_1}, which shows the SAD between
%the edge posteriors computed by the exact DP method  with bias \cite{Koivisto:06} and the corresponding ones
%computed by the exact DP method  without bias \cite{tian:he2009}.
%The edge posteriors computed by the exact DP method with bias in \cite{Koivisto:06} are exactly the results
%of DP phase in DP+MCMC method. The large SAD in this column will indicate the large bias intrinsic in these methods
%that use order-modular prior.

Table \ref{tb:unbiased_comparison_sad_1} shows the experimental results in terms of SAD for each data case
while
Table \ref{tb:unbiased_comparison_time_1} shows
%the corresponding experimental results in terms of the total running time $T_t$
the running time costs
corresponding to Table \ref{tb:unbiased_comparison_sad_1}.
(Just as Table \ref{tb:bias_comparison_time_1},
Table \ref{tb:unbiased_comparison_time_1} also lists
the sample mean and the sample standard deviation
of the running time of three steps of the DDS in the IW-DDS algorithm.)
The column named DP in Table \ref{tb:unbiased_comparison_sad_1} shows the SAD
($\sum_{ij} | p_{\nprec}( i \rightarrow j | D) $ $- p_{\prec}(i \rightarrow j | D)|$),
where each edge posterior $p_{\nprec}( i \rightarrow j | D)$ is computed by the exact DP method
of \citet{tian:he2009},
and  each edge posterior $p_{\prec}( i \rightarrow j | D)$ is computed by the exact DP method
of \citet{Koivisto:06}.
%with the prior setting $\rho_i(Pa_i) \equiv 1$.
%The column named DP' in Table \ref{tb:unbiased_comparison_sad_1} shows the SAD
%($\sum_{ij} | p_{\nprec}( i \rightarrow j | D) $ $- p_{\prec}(i \rightarrow j | D)|$),
%where each edge posterior $p_{\nprec}( i \rightarrow j | D)$ is still computed by the exact DP method
%of \citet{tian:he2009},
%but each edge posterior $p_{\prec}( i \rightarrow j | D)$ is computed by the exact DP method
%of \citet{Koivisto:06}
%with the prior setting $\rho_i(Pa_i) = 1/ \binom {n-1} {|Pa_i|} $.
%The SAD values reported in these two columns indicate the bias due to the assumption of the order-modular prior.
%Since the SAD in the DP column is smaller than the one in the DP' column for 29 out of the 33 cases,
%we use the prior setting $\rho_i(Pa_i) \equiv 1$ for the DP phase so that the DP+MCMC method
%can have better performance.
The SAD values reported in this column indicate the bias due to the assumption of the order-modular prior.
%Next to the DP and DP' columns,
Next to the DP column,
the SAD values of the three methods are listed in Table \ref{tb:unbiased_comparison_sad_1}.
Both the DP+MCMC method and the IW-DDS method are random so that both $\hat{\mu}$(SAD) and $\hat{\sigma}$(SAD) are shown for these two methods.
The outcome of the $K$-best algorithm is not random so that only its SAD is shown.
Finally Table \ref{tb:unbiased_comparison_sad_1} also shows the cumulative posterior probability mass $\Delta$ for both the $K$-best algorithm and the IW-DDS method.

Tables \ref{tb:unbiased_comparison_sad_1} and \ref{tb:unbiased_comparison_time_1}
clearly demonstrate the advantage of our method over the other two methods.
By using much shorter computation time,
our method has its
$\hat{\mu}$(SAD)
less than the corresponding one from the DP+MCMC for 32 out of the 33 data cases.
The only exceptional case is Syn15 with $m = 200$.
Furthermore, based on the two-sample $t$ test with unequal variances,
we can conclude at the significance level 0.05 that
the real mean of SAD using our method is less than the corresponding one using the DP+MCMC
for each of the 31 cases;
the two exceptional cases are Syn15 with $m = 100$ and Syn15 with $m = 200$.
(For 30 out of these 31 cases, the p-value of the two-sample $t$ test is less than $0.01$.)
Meanwhile,  $\hat{\sigma}$(SAD) using our method is always much smaller than the one using the DP+MCMC for each of the 33 cases,
which indicates higher stability of the performance of our method.
%Note that for small $m$ such as $100$ and $200$,
%$\hat{\mu}$(SAD) from DP+MCMC is even larger than the SAD from DP \cite{Koivisto:06} itself,
%indicating that MCMC phase of DP+MCMC will even increase the distance between $\hat{p}(i \rightarrow j | D)|$
%and the true $p( i \rightarrow j | D)$ when $m$ is small.
Similarly,
using much shorter computation time,
our method has its
$\hat{\mu}$(SAD)
less than the SAD
from the $K$-best
for 32 out of the 33 cases.
The only exception is Syn15 with $m$ = 5,000.
Furthermore, based on the one-sample $t$ test
%\cite{CB:StatisticalI},
\citep{CB:StatisticalI},
we can conclude at the significance level $5 \times 10^{-4}$ that
the real mean of SAD using our method is less than the SAD using the $K$-best
for each of these 32 cases.

There are several other interesting things shown in Tables \ref{tb:unbiased_comparison_sad_1} and \ref{tb:unbiased_comparison_time_1}.
In terms of the SAD,
for very small $m$,
$\hat{\mu}$(SAD) using the DP+MCMC method is even larger than the SAD from the DP phase \citep{Koivisto:06} itself.
For example, for the data case Zoo,
the SAD from the DP phase is $8.2142$
but $\hat{\mu}$(SAD) obtained after the MCMC phase of the DP+MCMC method is $32.4189$.
Similar situations also occur in Syn15, Letter, Insur19, and Child when $m = 100$.
This indicates that for very small $m$,
the MCMC phase of the DP+MCMC method is unable to reduce the bias from the DP method
of
\citet{Koivisto:06}
for all these cases based on $190,000$ MCMC iterations.
As for the running time,
please note that $\hat{\mu}$($T_{DP}$) of our IW-DDS is always less than the running time of the DP phase of the DP+MCMC method. %for each case.
This is because the DP step of our method uses the DP algorithm
%in \citep{Koivisto:soo04},
of \citet{Koivisto:soo04},
that is,
the first three steps of the DP algorithm
%in \citep{Koivisto:06};
of \citet{Koivisto:06};
while the DP phase of the DP+MCMC method uses
all the five steps of the DP algorithm
%in \citep{Koivisto:06}.
of \citet{Koivisto:06}.
In other words, compared with the DP algorithm
%in \citep{Koivisto:soo04},
of \citet{Koivisto:soo04},
the DP algorithm
%in \citep{Koivisto:06}
of \citet{Koivisto:06}
includes a larger constant factor hidden in the $O(n^{k+1}C(m) + kn2^n)$ notation
though these two DP algorithms have the same time complexity.
This difference will
make the total running time of our IW-DDS
even less than the running time of the DP phase of the DP+MCMC method
when the remaining steps of the IW-DDS run faster than the last two steps of the DP algorithm
%in \citep{Koivisto:06}.
of \citet{Koivisto:06}.
For example, for the data case Child with $m = 5,000$,
$\hat{\mu}(T_{t})$ of the IW-DDS is $524.63$ seconds
while the corresponding running time of the DP phase of the DP+MCMC method is $589.99$ seconds.
Actually, Table \ref{tb:unbiased_comparison_time_1} shows that
there are 21 out of the 33 cases where $\hat{\mu}(T_{t})$ of the IW-DDS is
less than the running time of the DP phase of the DP+MCMC method.
In addition, just as shown in Section~\ref{sec-expr-DDS},
the effectiveness of our time-saving strategy can also be clearly seen from
%all the four data sets with various values of $m$.
%$\hat{\mu}(T_{DAG})$ and $\hat{\sigma}(T_{DAG})$ shown in
Table \ref{tb:unbiased_comparison_time_1}.
For example, the ratio of $\hat{\mu}(T_{DAG})$ to $\hat{\mu}(T_{ord})$
ranges from 0.07 to 87.29 across the 33 cases, which is
much smaller than the upper bound of the ratio of $O(n^{k+1} N_o)$ to $O(n^2 N_o)$.

Table \ref{tb:unbiased_comparison_sad_1} also shows the resulting cumulative probability mass $\Delta$
from the $K$-best and our IW-DDS for all the data cases.
(Note that $\Delta$ is computed by the formula
$\Delta $ $  = \sum_{G \in \mathcal{G}} p_{\nprec}(G, D) / p_{\nprec}(D) $,
where $p_{\nprec}(D)$ is computed using POSTER tool.) %the DP algorithm \citep{tian:he2009}.)
In the table,
$\hat{\mu}$($\Delta$) from our IW-DDS  is greater than $\Delta$ from the $K$-best
for 28 out of the 33 data cases.
The five exceptional cases %where $\hat{\mu}$($\Delta$) from our IW-DDS  is less than $\Delta$ from $K$-best method
are Zoo, Tumor, Syn15 with $m = 100$, Syn15 with $m = 5,000$, and Letter with $m = 100$.
Interestingly, for four out of the five exceptional cases (as well as the other 28 cases),
$\hat{\mu}$(SAD) from our IW-DDS is significantly smaller than SAD from the $K$-best.
One possible reason is that the $K$ best DAGs
tend to have the same or similar local structures (family $(i, Pa_i)$'s)
that have relatively large local scores while
a large number of DAGs sampled from our IW-DDS include various local structures for each node $i$.
%When $\Delta$ does not constitute a large percentage,
When $\Delta$ is far below $1$,
the inclusion of various local structures seems to be more effective in improving the structural learning performance.

In addition, Table \ref{tb:unbiased_comparison_sad_1} shows that
%when $m$ is not very small
%(such as no smaller than $500$ or $1,000$),
when $m$ is not very small (such as no smaller than $1,000$),
$\Delta$ from our IW-DDS with $N_o = 30,000$ can reach a large percentage
(such as greater than $90\%$) in most of our data cases.
As a result,
we can obtain
a sound interval for $p_{\nprec}(f|D)$
with a small width
(such as less than $0.1$) for any feature $f$.
%and therefore provide a meaningful guarantee on the quality of estimation.

%\subsection{SUPPLEMENTARY EXPERIMENTAL RESULTS FOR $\Delta$}

To further demonstrate that our IW-DDS can obtain a large $\Delta$ efficiently when $m$ is not very small,
we increased $N_o$ from $100,000$ to $600,000$ with each increment $100,000$ to see
its performance for the data cases Letter with $m = 500$, Child with $m = 500$,
and German.
Again, we performed
20 independent runs for each data case to get the results.
Figures \ref{fig-errorbar_Delta_Letter_m500_100k_600k},
\ref{fig-errorbar_Delta_Child_m500_100k_600k}
and \ref{fig-errorbar_Delta_German_100k_600k} show
the increase in $\Delta$ with respect to the increase in $N_o$ for these three data cases.
Correspondingly,
Figures \ref{fig-plot_Time_for_Delta_Letter_m500_100k_600k},
\ref{fig-plot_Time_for_Delta_Child_m500_100k_600k}
and \ref{fig-plot_Time_for_Delta_German_100k_600k}
indicate the increase in $\hat{\mu}(T_{t})$, %(the sample mean of the total running time),
$\hat{\mu}(T_{ord})$
and $\hat{\mu}(T_{DAG})$
with respect to the increase in $N_o$ for these three data cases,
where the running time is in seconds.
Combining these figures,
it is clear that our IW-DDS can efficiently achieve a large $\Delta$.
Take the data case German for example,
with the time cost $\hat{\mu}(T_{t})$ = $1,493.02$ seconds,
our IW-DDS can collect $N_o = 600,000$ DAG samples so that
the corresponding mean of $\Delta$ can reach $91.74\%$.
Therefore, for any feature $f$ in the data case German,
our IW-DDS can provide a sound interval for $p_{\nprec}(f|D)$
with width $0.0826$.
Note that the $K$-best can only provide a meaningless sound interval for $p_{\nprec}(f|D)$ with huge width $0.922$
because its $\Delta$ can only reach $0.078$ in the data case German before running out of the memory.
%%This interval length is much smaller than the corresponding one using the $K$-best method
%%since $\Delta = 0.078$ when the $K$-best method uses all the memory.
%%Note that the DP+MCMC method can not provide such an interval
%%and the $K$-best method can only provide the interval of $\hat{p}_{\nprec}(f|D)$ with the large width $0.922$
%%due to its small $\Delta = 0.078$
%These experiments show that our IW-DDS is capable of providing a meaningful sound interval in a reasonable time, unlike existing sampling algorithms.
Also note that the ratio of $\hat{\mu}(T_{DAG})$ to $\hat{\mu}(T_{ord})$
decreases from  $56.45$ to $30.95$ when $N_o$ increases from $100,000$ to $600,000$.
(The increase rate of $\hat{\mu}(T_{ord})$ is a constant with respect to $N_o$;
but the increase rate of $\hat{\mu}(T_{DAG})$ actually decreases as $N_o$ increases.)
This witnesses the statement in Remark \ref{rm:time_DAG_sampling}
that the benefit from our time-saving strategy will typically
increase when $N_o$ increases.

\begin{figure}
\centering

\begin{minipage}[b]{0.46\linewidth}
  %\centering
  \includegraphics[width=1 \linewidth]{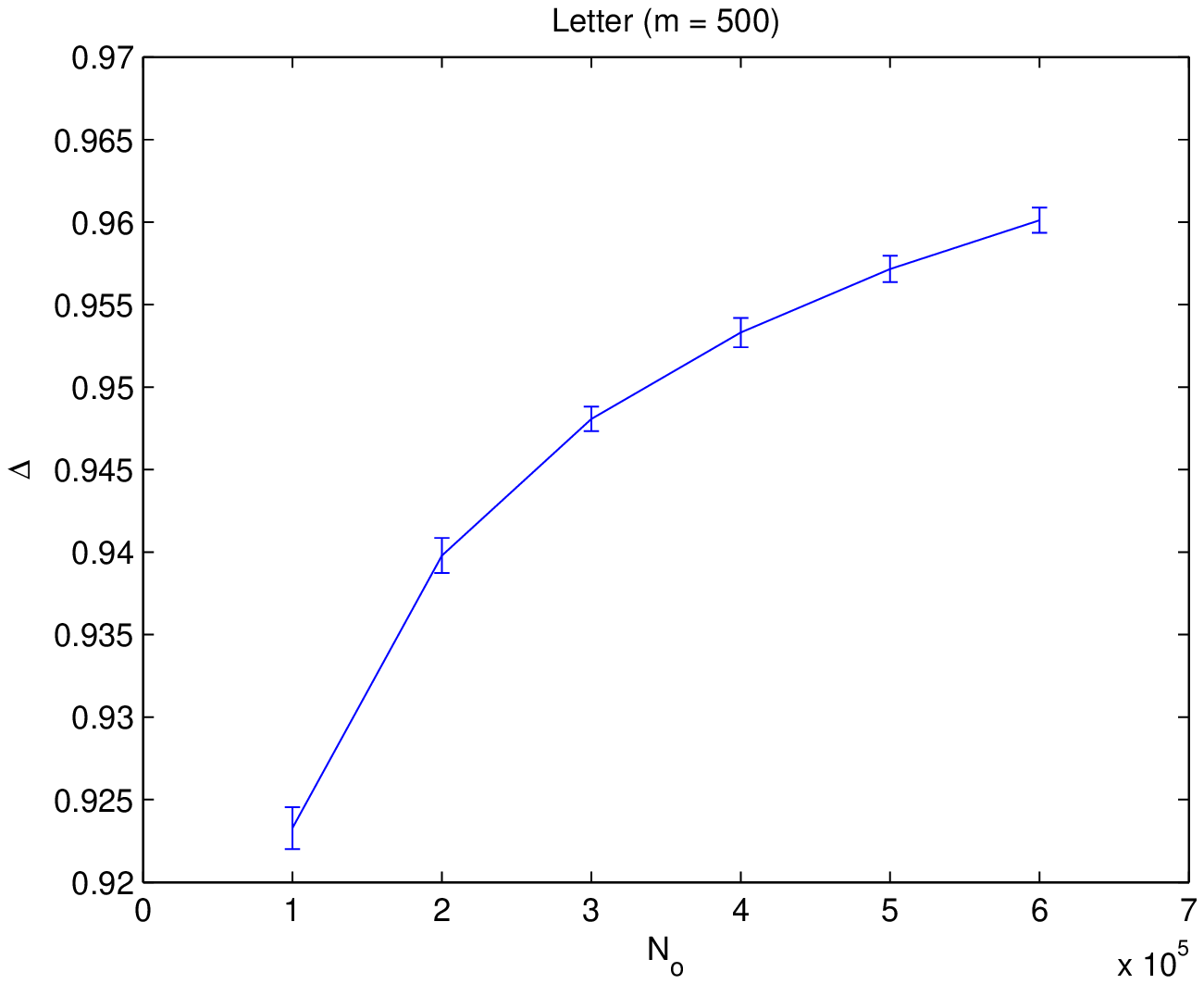}
  \caption{Plot of $\Delta$ versus $N_o$ for Letter ($m = 500$)}
  \label{fig-errorbar_Delta_Letter_m500_100k_600k}
\end{minipage}
\qquad
\begin{minipage}[b]{0.46\linewidth}
  %\centering
  \includegraphics[width=1 \linewidth]{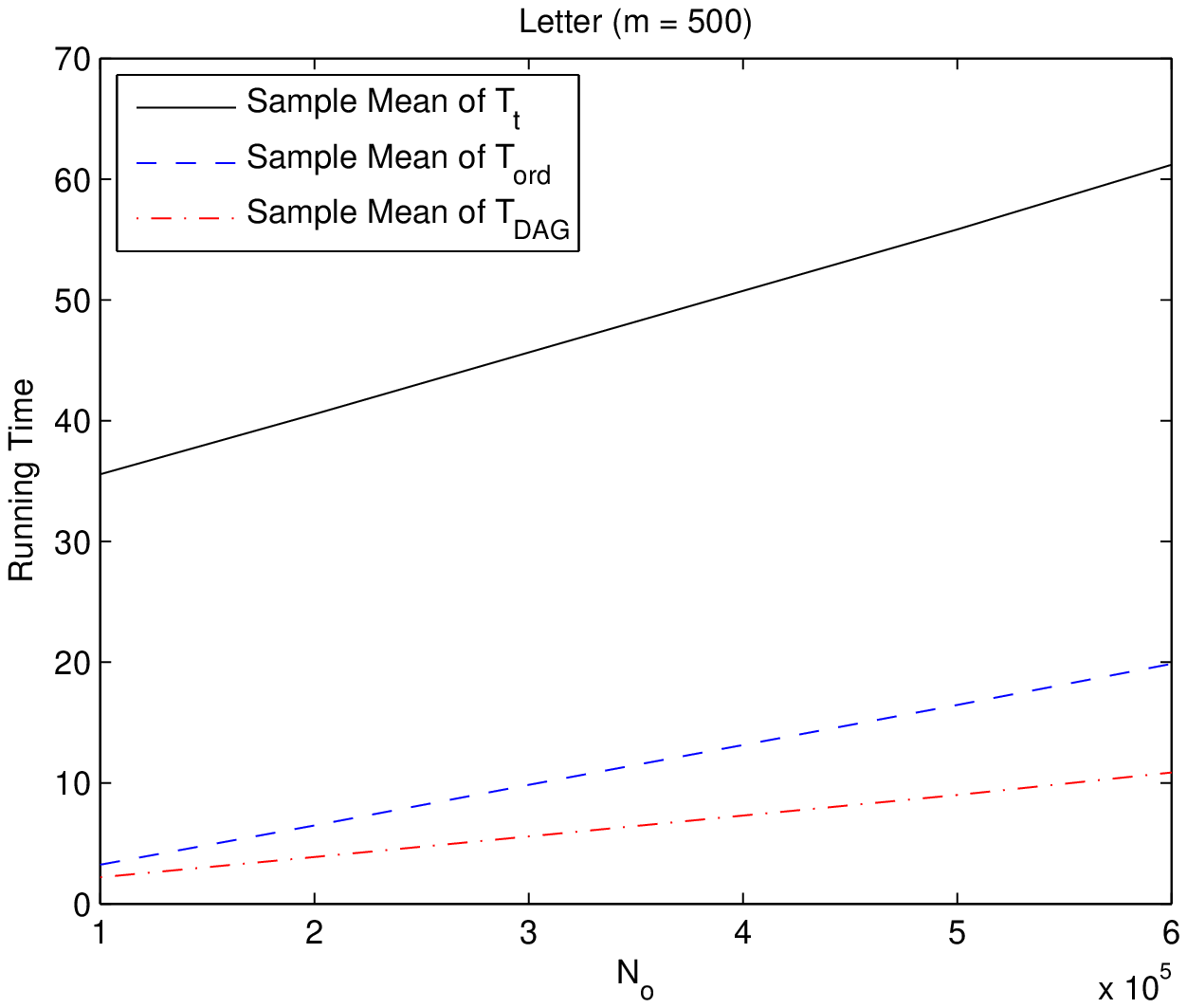}
  \caption{Plot of the Running Time versus $N_o$ for Letter ($m = 500$)}
  \label{fig-plot_Time_for_Delta_Letter_m500_100k_600k}
\end{minipage}

\qquad

\begin{minipage}[b]{0.46\linewidth}
  %\centering
  \includegraphics[width=1 \linewidth]{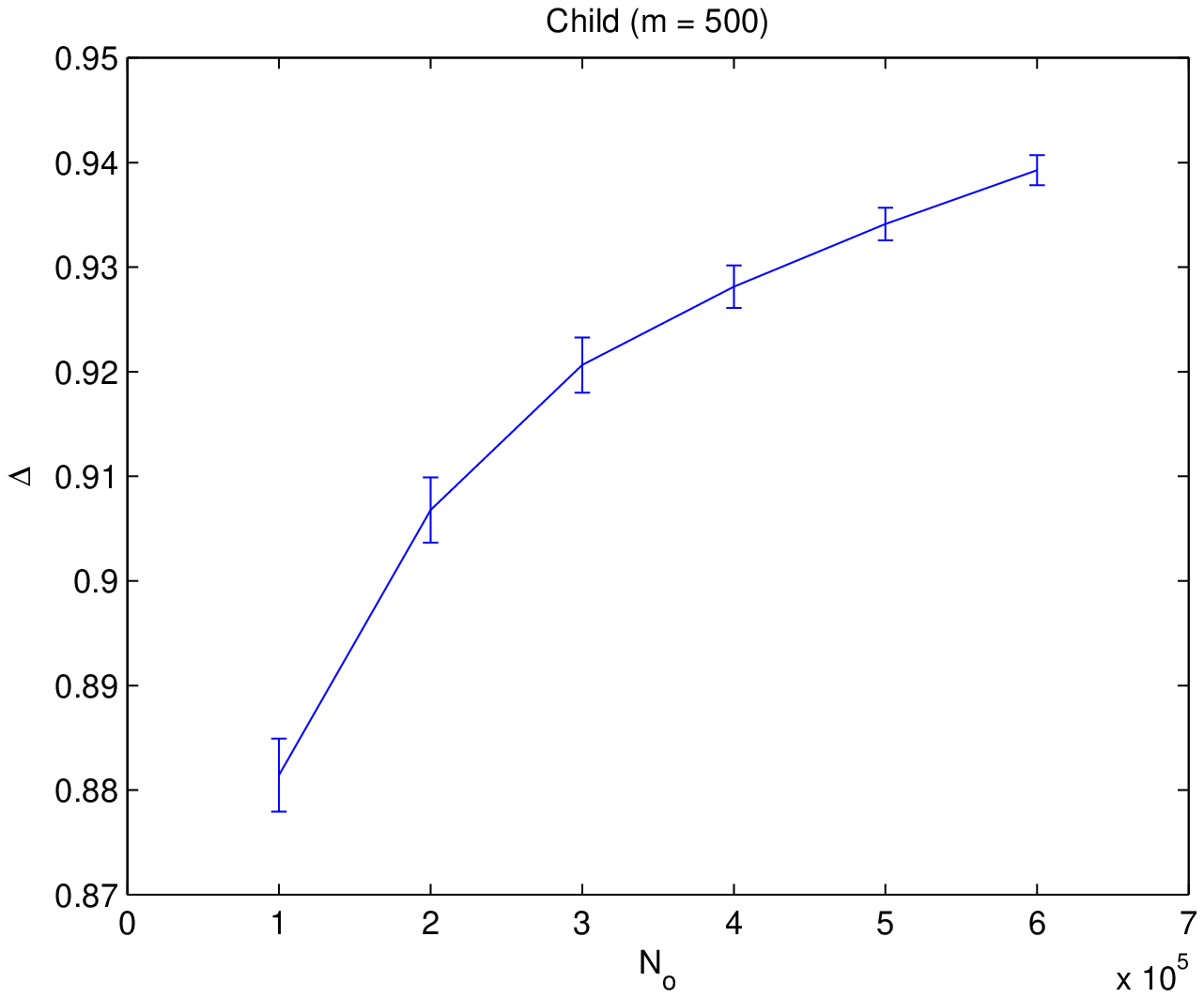}
  \caption{Plot of $\Delta$ versus $N_o$ for Child ($m = 500$)}
  \label{fig-errorbar_Delta_Child_m500_100k_600k}
\end{minipage}
\qquad
\begin{minipage}[b]{0.46\linewidth}
  %\centering
  \includegraphics[width=1 \linewidth]{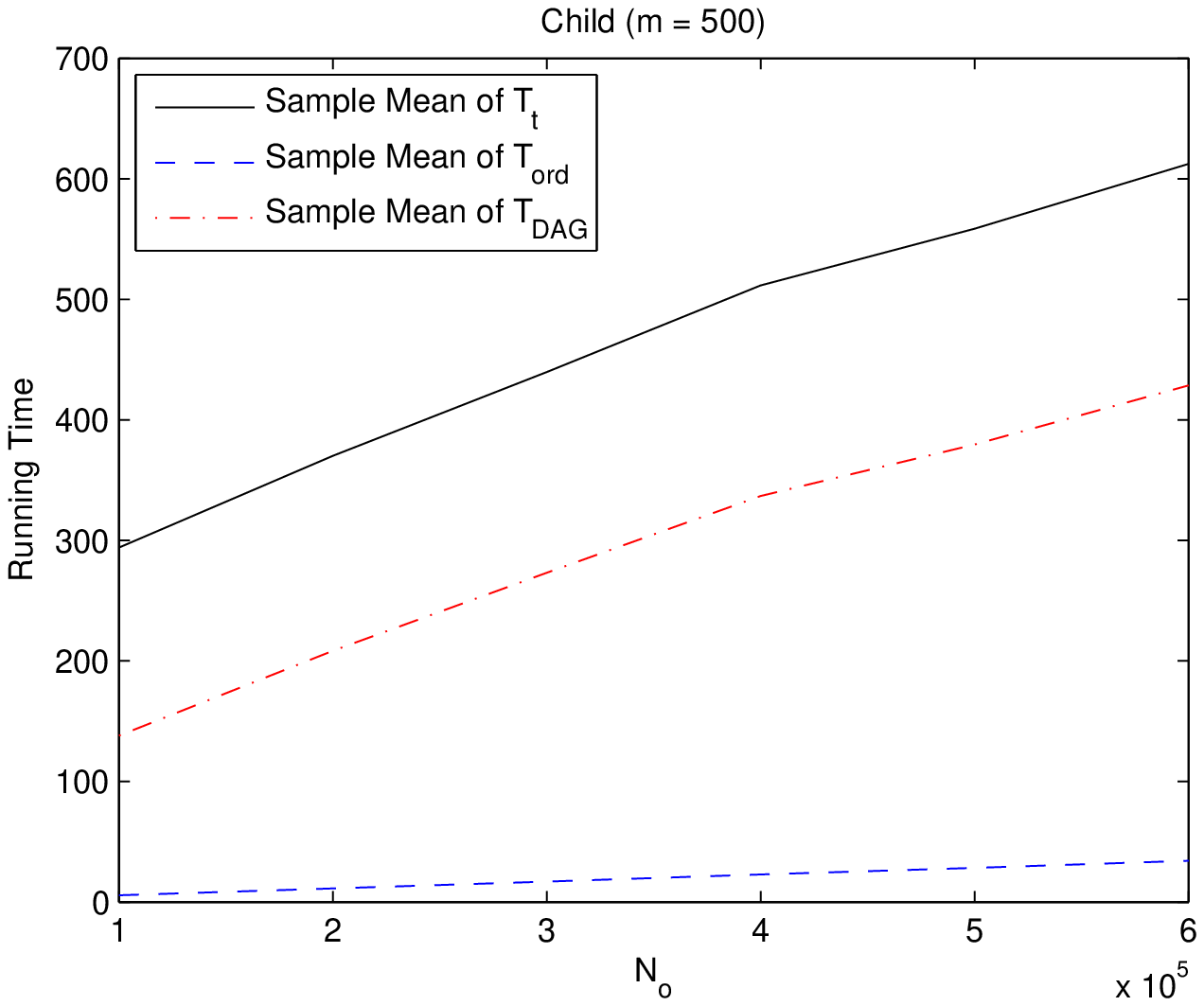}
  \caption{Plot of the Running Time versus $N_o$ for Child ($m = 500$)}
  \label{fig-plot_Time_for_Delta_Child_m500_100k_600k}
\end{minipage}

\qquad

\begin{minipage}[b]{0.46\linewidth}
  %\centering
  \includegraphics[width=1 \linewidth]{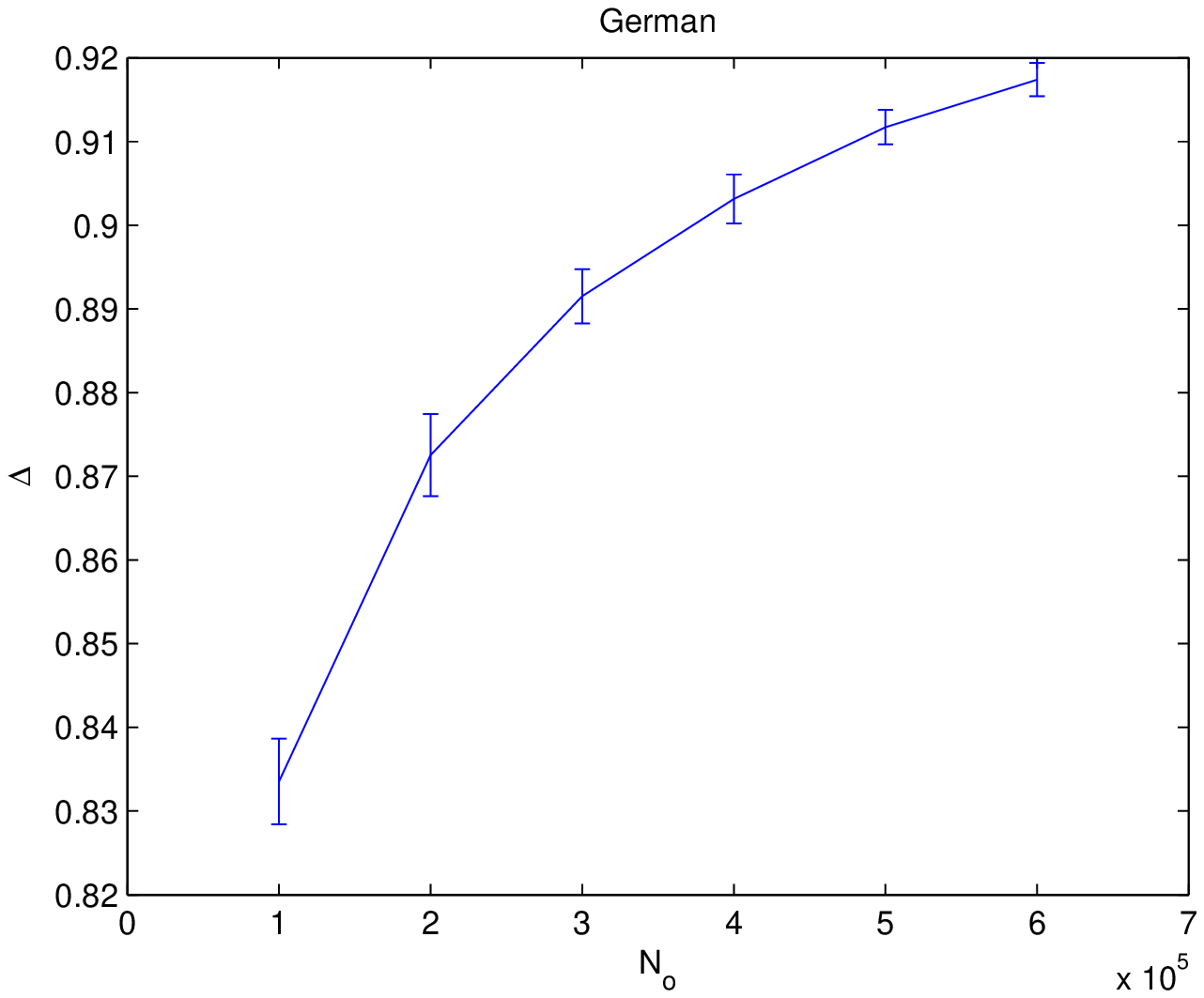}
  \caption{Plot of $\Delta$ versus $N_o$ for German}
  \label{fig-errorbar_Delta_German_100k_600k}
\end{minipage}
\qquad
\begin{minipage}[b]{0.46\linewidth}
  %\centering
  \includegraphics[width=1 \linewidth]{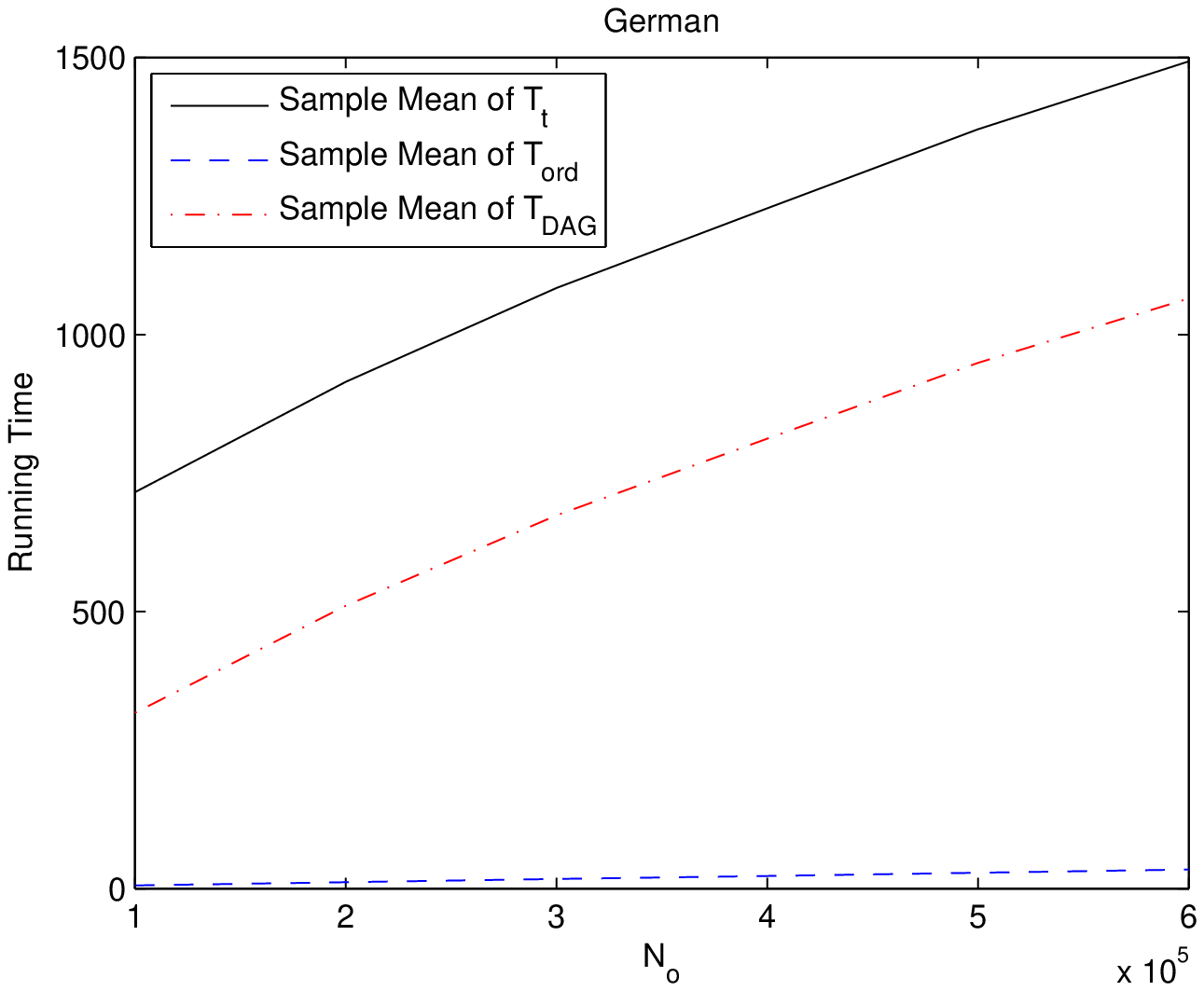}
  \caption{Plot of the Running Time versus $N_o$ for German}
  \label{fig-plot_Time_for_Delta_German_100k_600k}
\end{minipage}

\end{figure}

Finally, we present the experimental results
for the IW-DDS by
%fixing $m$ %(the number of data instances)
%and
varying the sample size.
Just as Section \ref{sec-expr-DDS},
the experiments were performed for
the five data cases
Tic-Tac-Toe,
Wine,
Letter with $m = 500$,
Child with $m = 500$,
and German.
For the IW-DDS,
we tried the sample size
$N_o$ $ = 5,000 \cdot i$,
where $i \in \{1, 2, \ldots, 10 \}$.
For each $i$, we independently ran the IW-DDS $20$ times
to get the sample mean and the sample standard deviation of SAD
for the (directed) edge features.
For the DP+MCMC,
we ran totally $50,000 \cdot i$ MCMC iterations,
where $i \in \{1, 2, \ldots, 10 \}$.
For each $i$, we discarded the first $25,000 \cdot i$ MCMC iterations for ``burn-in''
and set the thinning parameter to be $5$
so that $5,000 \cdot i$ DAGs got sampled.
Again, for each $i$, we independently ran the MCMC $20$ times to get the sample mean and the sample standard deviation of SAD
for the edge features.
As for the $K$-best,
different experimental settings were used for different data cases due to the out-of-memory issue.
For two data cases Tic-Tac-Toe and Wine,
we ran the $K$-best program with $K = 20 \cdot i$, where $i \in \{1, 2, \ldots, 10 \}$.
(The setting of $K = 20 \cdot i$ guarantees that for these two data cases the running time of the $K$-best is longer than the running time of the IW-DDS at each $i$, which will be demonstrated soon.)
For the data case Letter with $m = 500$, we only ran the $K$-best with $K = 162$ because
the $K$-best program will run out of memory when $K > 162$ due to its expensive space cost.
The corresponding result of the $K$-best would be compared with the result of the IW-DDS at $i = 10$
(i.e., the IW-DDS with $N_o = 50,000$).
For the same out-of-memory issue,
we only set $K = 20$ for Child with $m = 500$ and set $K = 9$ for German when running the $K$-best program.
Note that since there is no randomness in the outcome of the $K$-best algorithm,
we always ran the $K$-best program only once to get its fixed SAD for the edge features.

The experimental results of comparing the three methods based on the data case Tic-Tac-Toe are shown in Figures
\ref{fig-errorbar_DPMCMC_Kbest_IWDDS_SAD_Tic_Tac_Toe} and \ref{fig-plot_DPMCMC_Kbest_IWDDS_Time_Tic_Tac_Toe}.
Figure
\ref{fig-errorbar_DPMCMC_Kbest_IWDDS_SAD_Tic_Tac_Toe}
shows the SAD performance of the three methods with each $i$ $\in \{1, 2, \ldots, 10 \}$ in terms of the edge features,
where an error bar represents one sample standard deviation $\hat{\sigma}$(SAD) across $20$ runs
from the DP+MCMC or the IW-DDS
at each $i$.
Figure
\ref{fig-plot_DPMCMC_Kbest_IWDDS_Time_Tic_Tac_Toe}
shows $\hat{\mu}(T_{t})$
(the sample mean of the total running time) of the DP+MCMC and the IW-DDS
as well as $T_{t}$ (the total running time) of the $K$-best
with each $i$, where the running time is in seconds.
The advantage of our IW-DDS can be clearly seen
by combining
Figures \ref{fig-errorbar_DPMCMC_Kbest_IWDDS_SAD_Tic_Tac_Toe}
and \ref{fig-plot_DPMCMC_Kbest_IWDDS_Time_Tic_Tac_Toe}.
Comparing with the DP+MCMC,
for each $i \in \{1, 2, \ldots, 10 \}$,
the IW-DDS uses the shorter running time
but has its real mean of SAD significantly smaller than the corresponding real mean of SAD from the DP+MCMC,
with the p-value $ < 1 \times 10^{-10}$ from the two-sample $t$ test with unequal variances.
Comparing with the $K$-best,
for each $i \in \{1, 2, \ldots, 10 \}$,
the IW-DDS uses the shorter running time than the $K$-best,
but the IW-DDS has its real mean of SAD significantly smaller than the SAD from the $K$-best,
with the p-value $ < 1 \times 10^{-35}$ from the one-sample $t$ test.
Therefore, %using much less running time,
the learning performance of the IW-DDS is significantly
better than the performance of the other two methods at each $i$  for the data case Tic-Tac-Toe.

The experimental results based on the data case Letter with $m = 500$  are shown in Figures
\ref{fig-errorbar_DPMCMC_Kbest_IWDDS_SAD_Letter_m500} and \ref{fig-plot_DPMCMC_Kbest_IWDDS_Time_Letter_m500}.
Just as the description for Figures \ref{fig-errorbar_DPMCMC_Kbest_IWDDS_SAD_Tic_Tac_Toe} and \ref{fig-plot_DPMCMC_Kbest_IWDDS_Time_Tic_Tac_Toe},
Figure
\ref{fig-errorbar_DPMCMC_Kbest_IWDDS_SAD_Letter_m500}
shows the SAD performance of the three methods in terms of the edge features,
%where an error bar represents one sample standard deviation $\hat{\sigma}$(SAD) across $20$ runs
%for the DP+MCMC or the IW-DDS at each $i$.
and Figure
\ref{fig-plot_DPMCMC_Kbest_IWDDS_Time_Letter_m500}
shows the corresponding time costs of the three methods.
The only difference is that in
both Figure \ref{fig-errorbar_DPMCMC_Kbest_IWDDS_SAD_Letter_m500} and
Figure \ref{fig-plot_DPMCMC_Kbest_IWDDS_Time_Letter_m500},
the corresponding result of the $K$-best with $K = 162$ is marked as a star and
compared with the one of the IW-DDS at $i = 10$.
The advantage of our IW-DDS can be clearly seen
by combining
Figures \ref{fig-errorbar_DPMCMC_Kbest_IWDDS_SAD_Letter_m500} and \ref{fig-plot_DPMCMC_Kbest_IWDDS_Time_Letter_m500}.
Comparing with the DP+MCMC,
for each $i \in \{1, 2, \ldots, 10 \}$,
the IW-DDS uses the much shorter running time
but has its real mean of SAD significantly smaller than the corresponding real mean of SAD from the DP+MCMC,
with the p-value $ < 0.013$ from the two-sample $t$ test with unequal variances.
($\hat{\mu}(T_{t})$ of the IW-DDS at $i = 10$ is only $33.00$ seconds, which is even less than
the running time of the DP phase of the DP+MCMC method ($39.58$ seconds).)
Note that
%Actually,
$\hat{\sigma}$(SAD) %the sample standard deviation of SAD
from the DP+MCMC
is shown to be very large.
The DP+MCMC has its
$\hat{\sigma}$(SAD)
%the sample standard deviation of SAD
even larger than its $\hat{\mu}$(SAD) (the sample mean of SAD) when $i \geq 8$,
which indicates that the performance of the DP+MCMC is not stable based on the 500,000 MCMC iterations.
Comparing with the $K$-best,
the IW-DDS with $N_o = 50,000$ uses the much shorter running time
but has its real mean of SAD significantly smaller than the SAD from the $K$-best with $K = 162$
since the p-value from the corresponding one-sample $t$ test is less than $ 1 \times 10^{-35}$.
Therefore, %using much less running time,
the learning performance of the IW-DDS is also significantly
better than the performance of the other two methods for the data case Letter with $m = 500$.

The experimental results for the other three data cases Wine, Child with $m = 500$, and German
are represented similarly in the supplementary material.
The same conclusion about the learning performance can be clearly drawn
by examining
the figures shown in the supplementary material.

\begin{figure}
\centering
\begin{minipage}[b]{0.46\linewidth}
  %\centering
  \includegraphics[width=1 \linewidth,  height= 1.08 \linewidth]{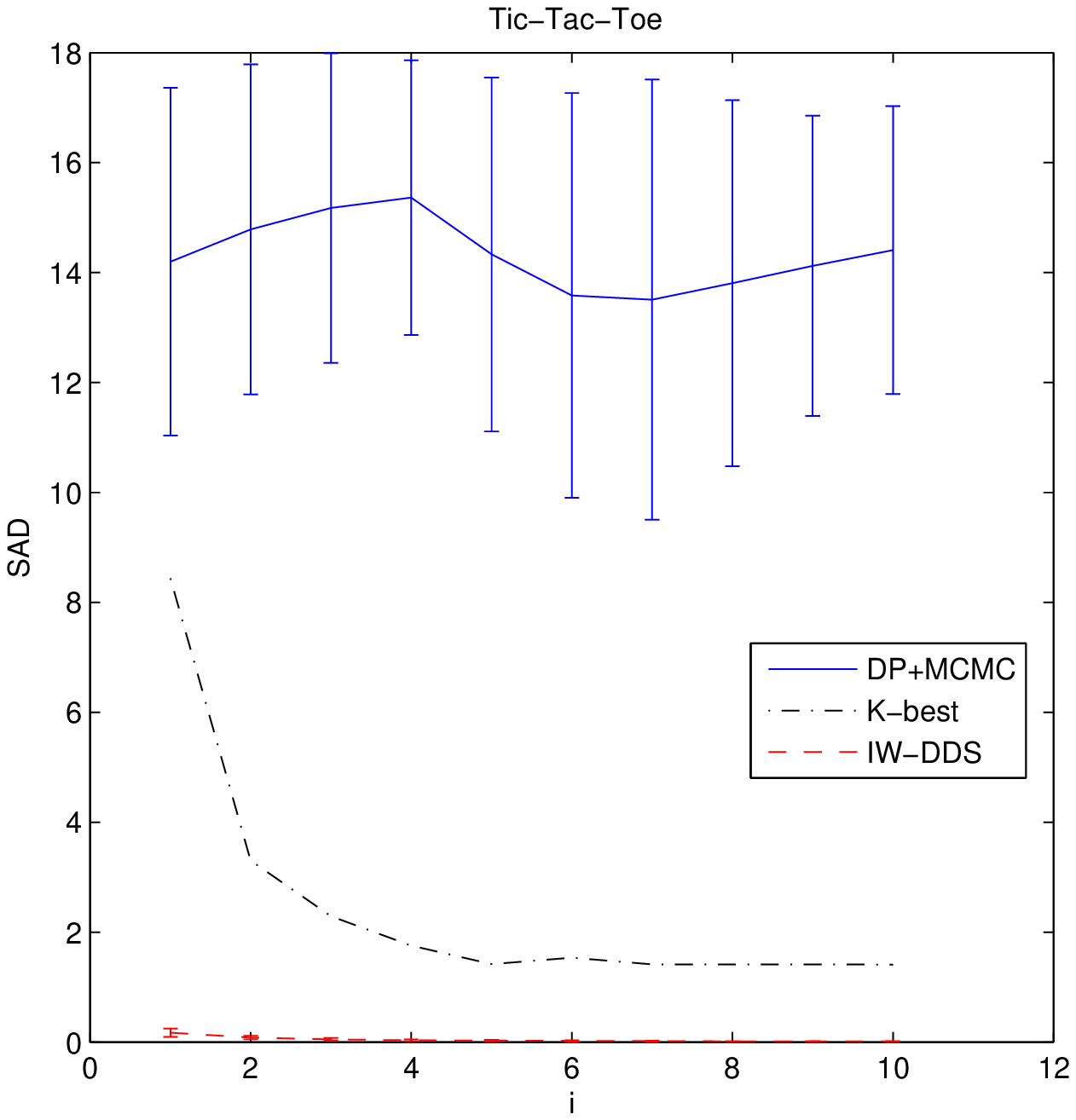}
  \caption{Plot of the SAD Performance of the DP+MCMC, the $K$-best and the IW-DDS for Tic-Tac-Toe}
  \label{fig-errorbar_DPMCMC_Kbest_IWDDS_SAD_Tic_Tac_Toe}
\end{minipage}
\qquad
\begin{minipage}[b]{0.46\linewidth}
  %\centering
  \includegraphics[width=1 \linewidth, height= 1.08 \linewidth]{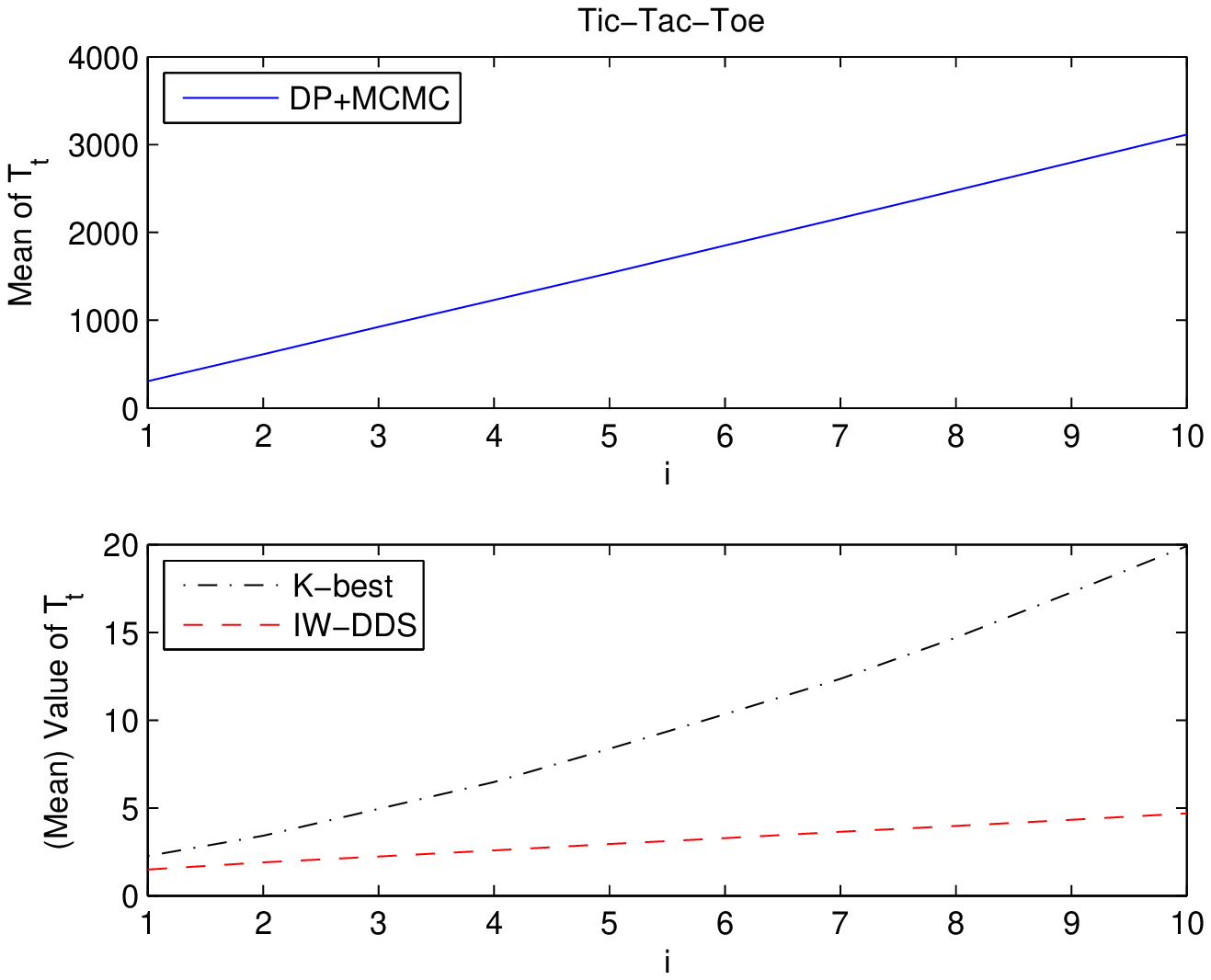}
  \caption{Plot of the Total Running Time of the DP+MCMC, the $K$-best and the IW-DDS for Tic-Tac-Toe}
  \label{fig-plot_DPMCMC_Kbest_IWDDS_Time_Tic_Tac_Toe}
\end{minipage}

\end{figure}

\begin{figure}
\centering

\begin{minipage}[b]{0.46\linewidth}
  %\centering
  %\includegraphics[width=1 \linewidth]{Errorbar_DPMCMC_Kbest_IWDDS_SAD_Letter_m500_5k_50k}
  \includegraphics[width=1 \linewidth]{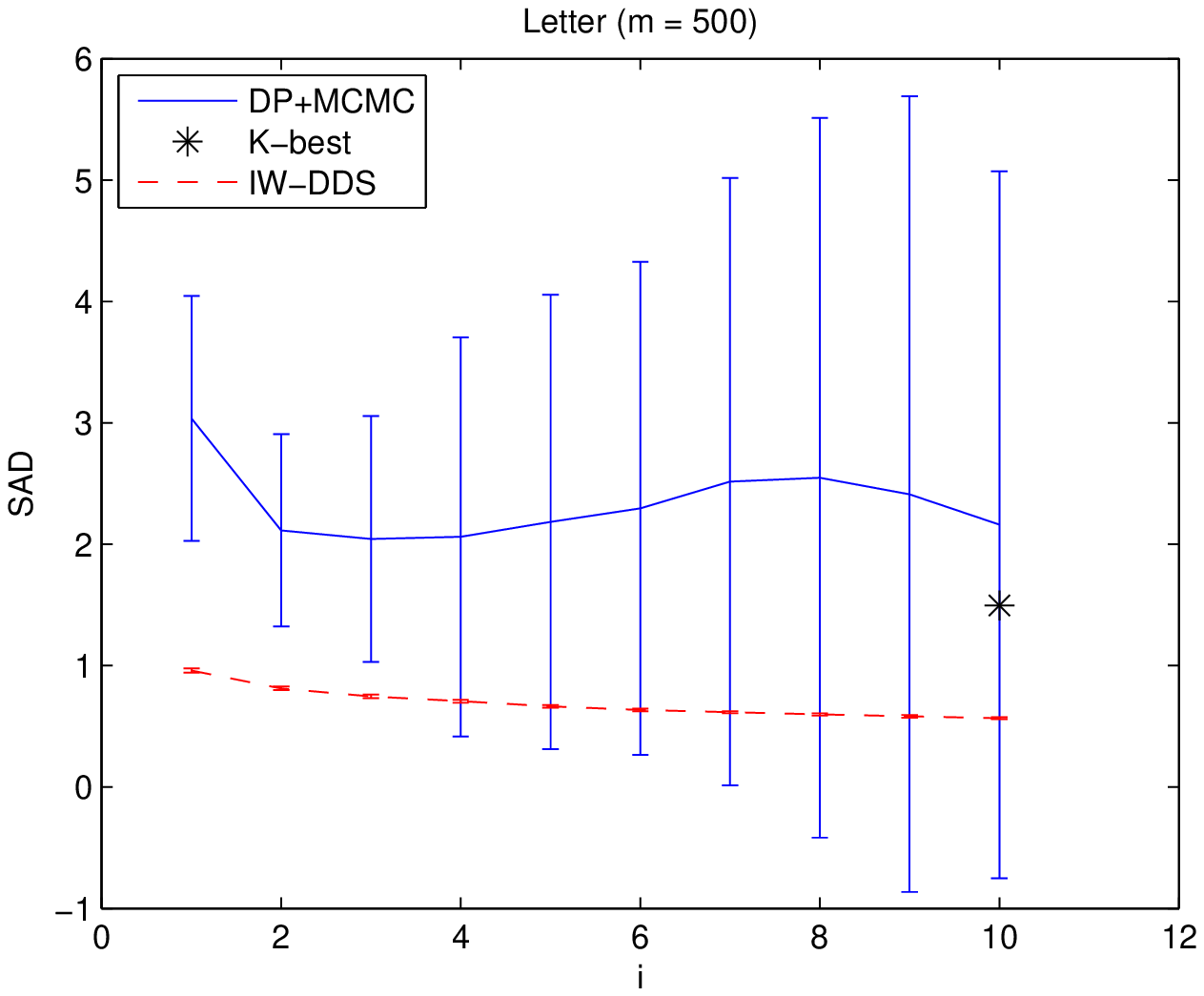}
  \caption{Plot of the SAD Performance of the DP+MCMC, the $K$-best and the IW-DDS for Letter ($m = 500$)}
  \label{fig-errorbar_DPMCMC_Kbest_IWDDS_SAD_Letter_m500}
\end{minipage}
\qquad
\begin{minipage}[b]{0.46\linewidth}
  %\centering
  %\includegraphics[width=1 \linewidth]{Plot_DPMCMC_Kbest_IWDDS_Time_Letter_m500_5k_50k}
  \includegraphics[width=1 \linewidth]{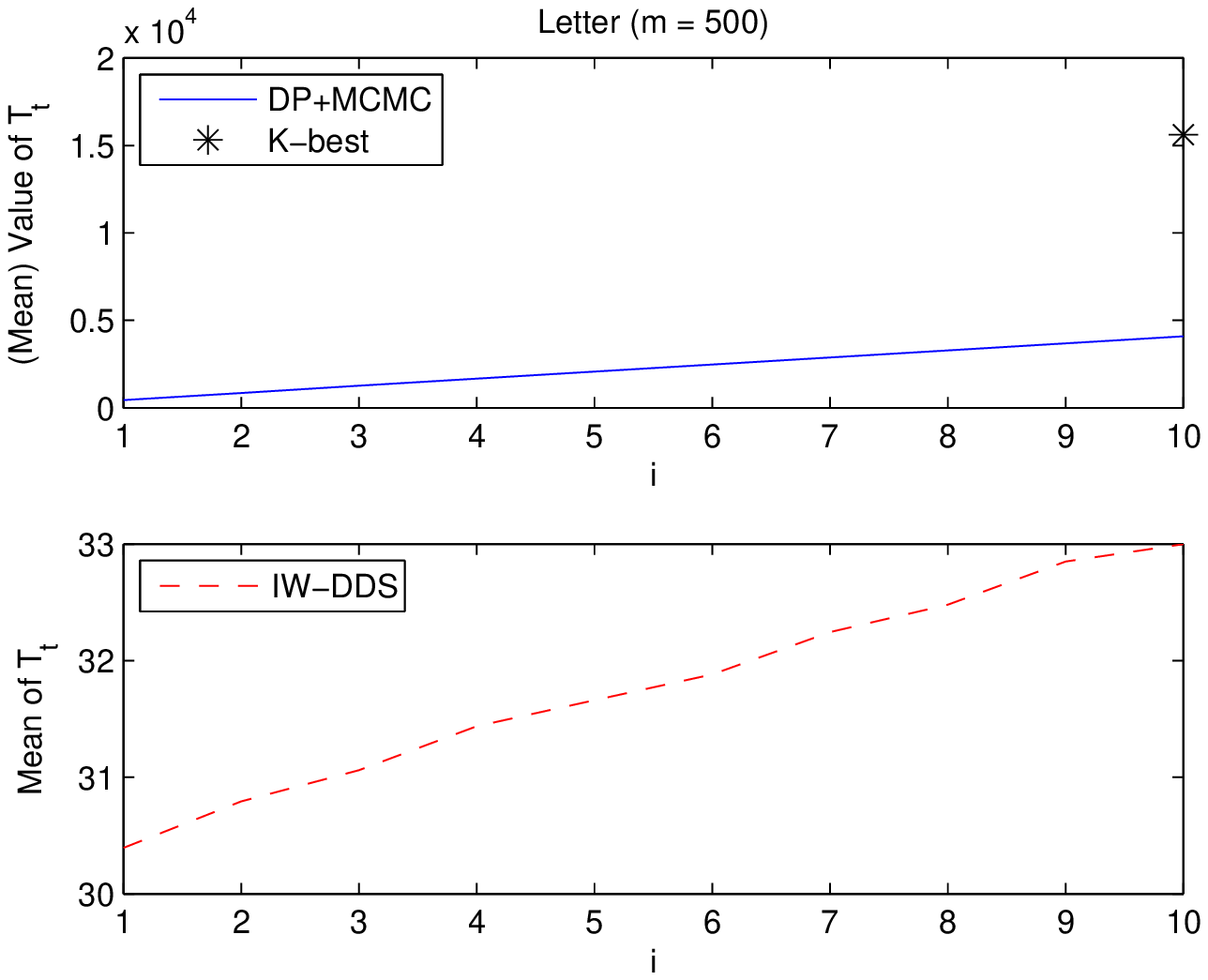}
  \caption{Plot of the Total Running Time of the DP+MCMC, the $K$-best and the IW-DDS for Letter ($m = 500$)}
  \label{fig-plot_DPMCMC_Kbest_IWDDS_Time_Letter_m500}
\end{minipage}

\end{figure}

%\subsection{Relation between the Learning Performance of Modular Features and the One of Non-modular Features}
%\subsection{RELATION BETWEEN THE LEARNING PERFORMANCE OF MODULAR FEATURES AND THE ONE OF NON-MODULAR FEATURES}
%\subsection{LEARNING PERFORMANCE OF NON-MODULAR FEATURES}
\subsection{Learning Performance of Non-modular Features}
\label{sec-expr-path}
In Sections \ref{sec-expr-DDS} and \ref{sec-expr-IWDDS}
%In the previous two subsections
we did not provide experimental results on the learning performance of non-modular features.
We did not do so in Section \ref{sec-expr-IWDDS}
because there is no known method to compute the true/exact posterior probability of any non-modular feature $p_{\nprec}(f|D)$
except by the brute force enumeration over all the (super-exponential number of) DAGs  %using Eq.~(\ref{eq-HR-pre1})
so that the quality of the corresponding $\hat{p}_{\nprec}(f|D)$
learned from any approximate method
%provided by any estimation method
cannot be precisely measured.
We did not do so in Section \ref{sec-expr-DDS}
%because of the similar reason as well as the fact
%that the latest version of PO-MCMC tool (BEANDisco)
because the current PO-MCMC tool (BEANDisco)
only supports the estimation of the posterior of an edge feature
so that the comparison of our method and the PO-MCMC can only be made for the edge feature.
(Thus, we did not make the comparison for the  path feature (which is one particular non-modular feature),
though the DP algorithm
%\cite{Parviainen:ECML2011}
%\citep{Parviainen:ECML2011}
of \citet{Parviainen:ECML2011}
can compute the exact
posterior of a path feature
%$p_{\prec}(i \rightsquigarrow j |D)$.)
$p_{\prec}(f |D)$.)
Our idea is that by showing that
our algorithms have significantly better performance in computing fundamental structural features
(directed edge features),
which
should be
%is
due to the better quality of our DAG samples with respect to
%the corresponding $p_{\nprec | \prec}(G|D)$,
the corresponding $p_{\nprec}(G|D)$ or $p_{\prec}(G|D)$,
we expect that they will also be superior in computing other complicated structural features
%built upon directed edge features %(such as paths)
using the same set of DAG samples.

To verify our expectation, %in this subsection
we performed the experiments on the real data set ``Iris''
(with $n = 5$)
%and $m = 150$)
from the UCI Machine
Learning Repository
%\cite{asuncion:new07} % which has $5$ variables and $150$ data instances.
\citep{asuncion:new07}
and the well-studied data set ``Coronary'' (Coronary Heart Disease)
(with $n = 6$)
%and $m = 640$)
%\cite{edwards:2000}.
\citep{edwards:2000}.
Since $n$ is small, by enumerating all the DAGs, we were able to compute $p_{\nprec}(f|D)$,
the true posterior probability for
any interesting non-modular feature $f$.
For the demonstration purpose,
we investigated the following five interesting non-modular features.
%$f_1$, a directed path feature $x \rightsquigarrow y$,
%$f_1$, a directed path feature $x \dashrightarrow y$,
%$f_1$, a directed path feature $x \times \dashrightarrow y$,
%$f_1$, a directed path feature $x | \dashrightarrow y$,
$f_1$, a directed path feature from node $x$ to node $y$, denoted by $x \sim > y$,
%$f_1$, a directed path feature $x \sim \rightsquigarrow y$,
%$f_1$, a directed path feature $x \cdots  \rightarrow y$,
represents the situation that variable $x$ eventually influences variable $y$.
%denoted by $i \rightsquigarrow j$;
$f_2$, a limited-length directed path feature $x \sim > y$ that has its path length no more than $2$,
%denoted by $i \rightsquigarrow^{\leq 2} j$;
represents that variable $x$ can influence variable $y$ via at most one intermediate variable.
$f_3$, a combined path feature $x \sim > y \sim > z$,
%from node $i$ via node $j$ to node $k$.
%denoted by  $i \rightsquigarrow j \rightsquigarrow k$.
can be interpreted as the situation that variable $x$ eventually influences variable $y$ which in turn eventually influences variable $z$.
%$f_4$, a combined feature $y \leftsquigarrow x \rightsquigarrow z$ with $y \neq z$,
$f_4$, a combined path feature $y < \sim x \sim > z$ with $y \neq z$,
means that variable $x$ eventually influences both variable $y$ and variable $z$.
$f_5$, a combined path feature $y < \sim x \nsim > z$ with $x \neq z$,
%$f_5$, a combined feature $y \rightsquigarrow x  \rightsquigarrow z$ with $x \neq z$,
%$f_5$, a combined feature $y \rightsquigarrow x \nrightsquigarrow z$ with $x \neq z$,
represents that variable $x$ eventually influences variable $y$ but not variable $z$.
Then we compared the SAD performance of the (directed) edge feature
with the corresponding SAD performance \footnote{
More specifically,
SAD =
%($\sum_{x y} | p_{\nprec}( x \rightsquigarrow y | D) $ $- \hat{p}_{\nprec}(x \rightsquigarrow y | D)|$)
($\sum_{x y} | p_{\nprec}( x \sim > y | D)  - \hat{p}_{\nprec}(x \sim > y | D)|$)
for the path feature
%$x \rightsquigarrow y$;
$x \sim > y$;
SAD =
%($\sum_{x y} | p_{\nprec}( x \rightsquigarrow y | D) $ $- \hat{p}_{\nprec}(x \rightsquigarrow y | D)|$)
($\sum_{x y} | p_{\nprec}( x \sim > y | D)  - \hat{p}_{\nprec}(x \sim > y | D)|$)
for the path feature
%$x \rightsquigarrow y$
$x \sim > y$
whose length is no more than $2$;
SAD =
%($\sum_{x y z} | p_{\nprec}(x \rightsquigarrow y \rightsquigarrow z | D) $ $- \hat{p}_{\nprec}(x \rightsquigarrow y \rightsquigarrow z | D)|$)
($\sum_{x y z} | p_{\nprec}(x \sim > y \sim > z | D)  - \hat{p}_{\nprec}(x \sim > y \sim > z | D)|$)
for the combined feature
%$x \rightsquigarrow y \rightsquigarrow z$.
$x \sim > y \sim > z$;
SAD =
($\sum_{x y z, y \neq z} | p_{\nprec}(y < \sim x \sim > z | D) - \hat{p}_{\nprec}(y < \sim x \sim > z | D)|$)
for the combined feature
$y < \sim x \sim > z$ with $y \neq z$;
SAD =
($\sum_{x y z, x \neq z} | p_{\nprec}(y < \sim x \nsim > z | D) - \hat{p}_{\nprec}(y < \sim x \nsim > z | D)|$)
for the combined feature
$y < \sim x \nsim > z$ with $x \neq z$.
}
of each feature $f_j$ ($j \in \{1, 2, 3, 4, 5 \}$)
from the DP+MCMC, the $K$-best and the IW-DDS.
%from these three methods.
The experimental results on both data sets show that
%if SAD of Method A is significantly smaller than SAD
if the SAD of the IW-DDS is significantly smaller than the SAD
of the competing method
%(which is arbitrarily chosen from DP+MCMC and $K$-best)
(the DP+MCMC or the $K$-best)
for the edge feature,
then the SAD of the IW-DDS will also be
%significantly
%smaller
significantly smaller
%(usually significantly smaller)
than the SAD
of the competing method for each investigated non-modular feature $f_j$ using the same set of DAG samples.
Thus, our expectation is supported by the experiments.
%The detailed experimental results are provided in the supplementary material due to the space limit.
The detailed experimental results are as follows.

\begin{figure}
\centering

\begin{minipage}[b]{0.46\linewidth}
  %\centering
  \includegraphics[width=1 \linewidth]{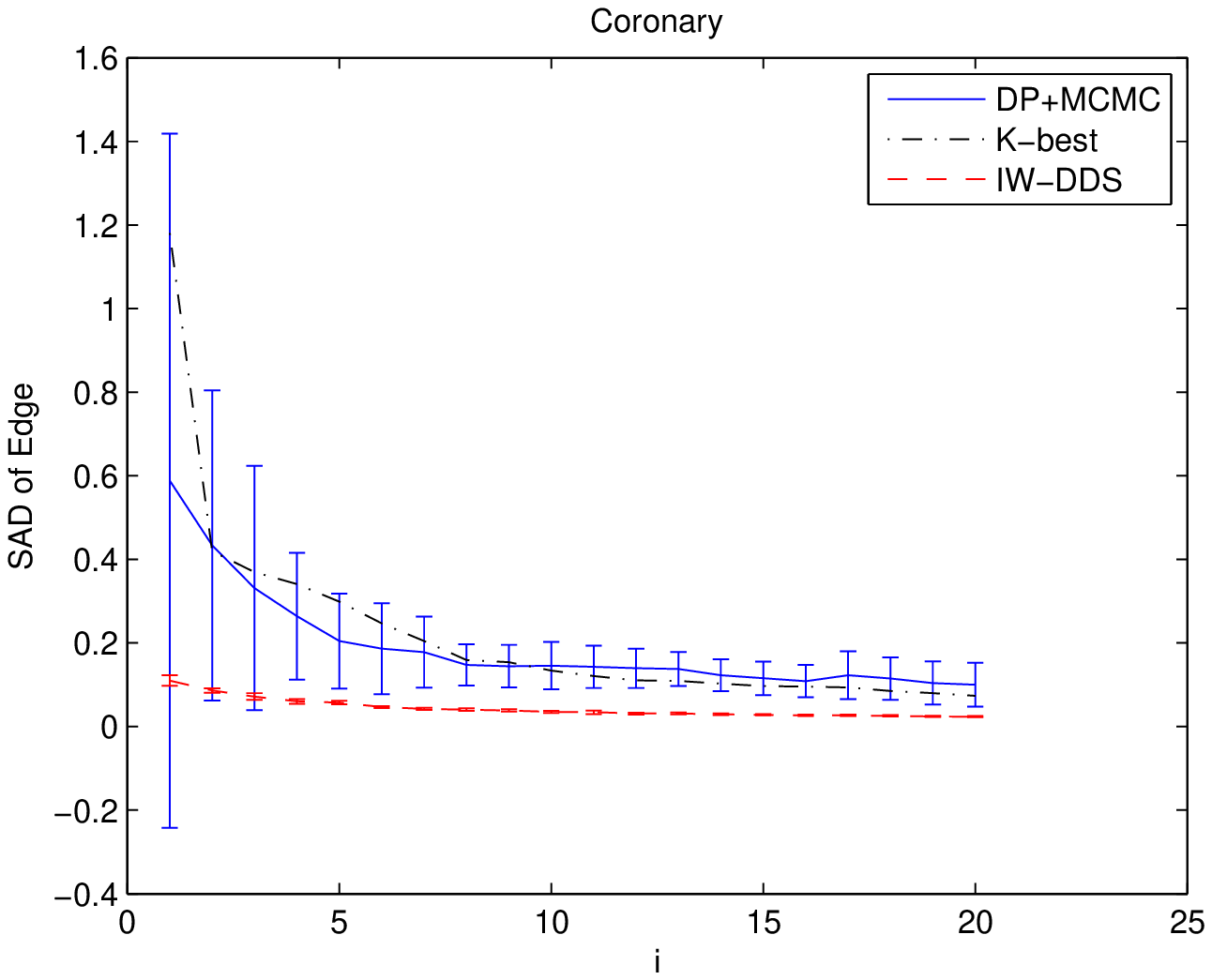}
  \caption{SAD of the Learned Edge Features for Coronary}
  \label{fig-SAD_edge_coronary_perm}
\end{minipage}
\qquad
\begin{minipage}[b]{0.46\linewidth}
  %\centering
  \includegraphics[width=1 \linewidth]{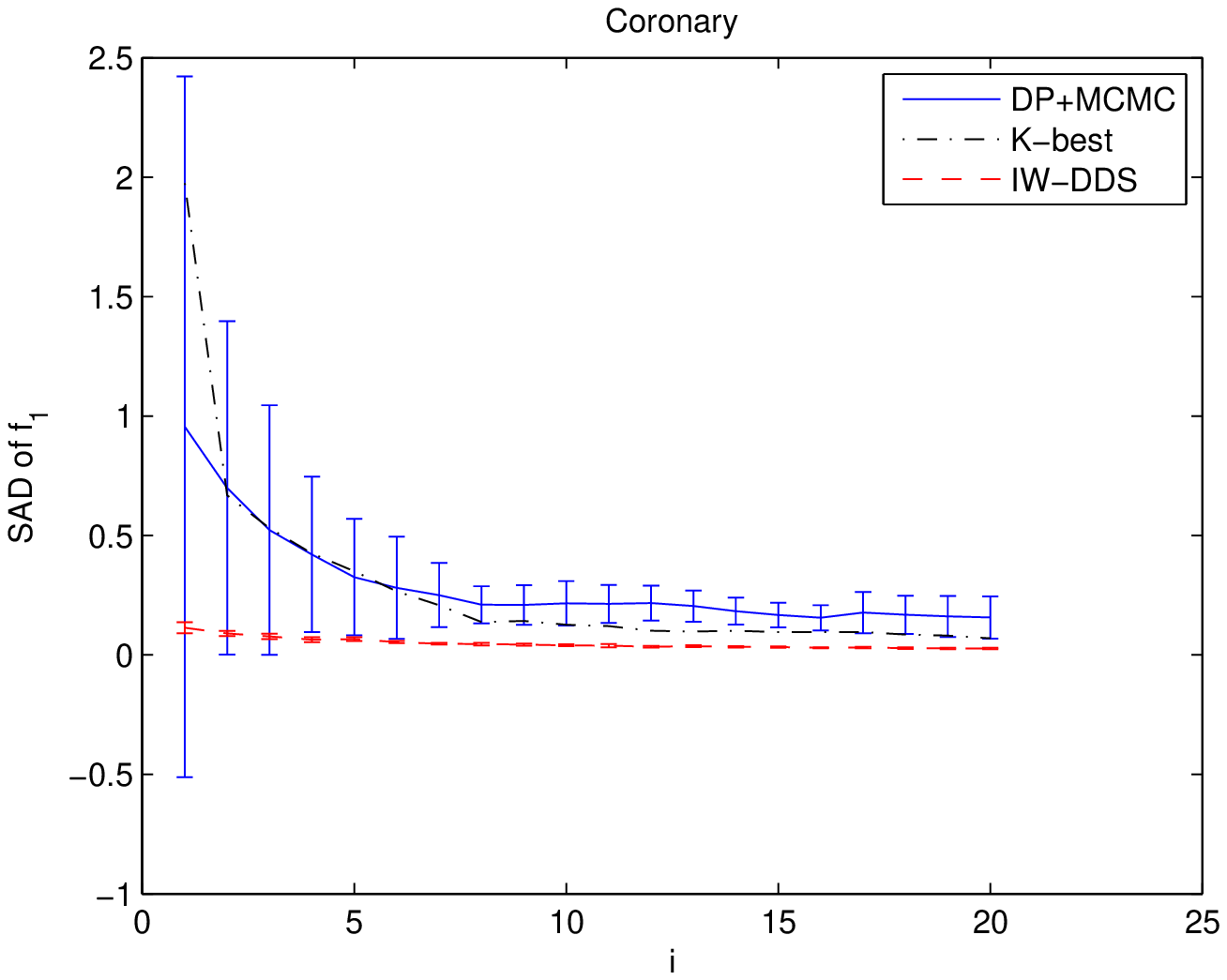}
  \caption{SAD of the Learned $f_1$ Features for Coronary}
  \label{fig-SAD_path_coronary_perm}
\end{minipage}

\qquad

\begin{minipage}[b]{0.46\linewidth}
  %\centering
  \includegraphics[width=1 \linewidth]{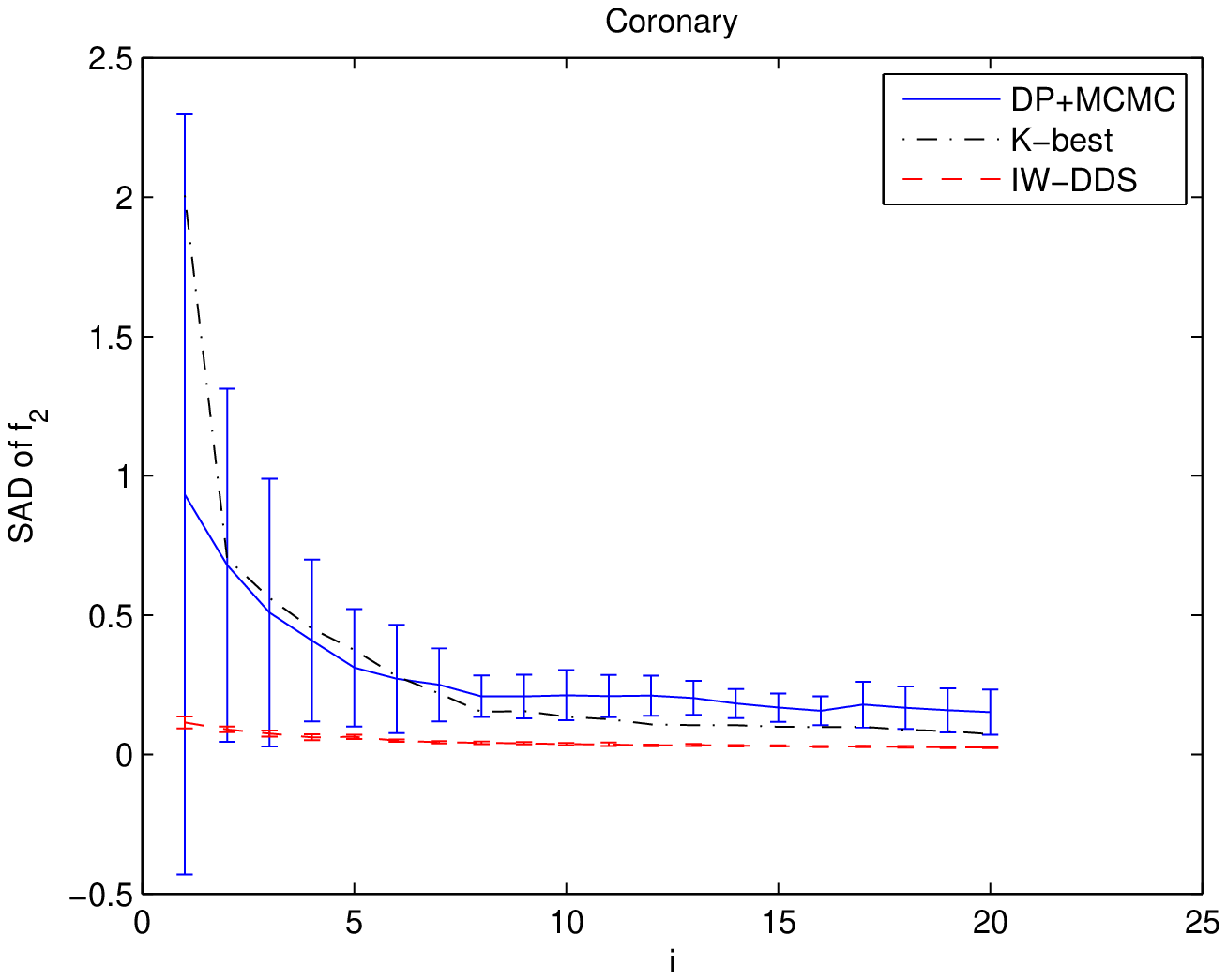}
  \caption{SAD of the Learned $f_2$ Features for Coronary}
  \label{fig-SAD_limited_leng_path_coronary_perm}
\end{minipage}
\qquad
\begin{minipage}[b]{0.46\linewidth}
  %\centering
  \includegraphics[width=1 \linewidth]{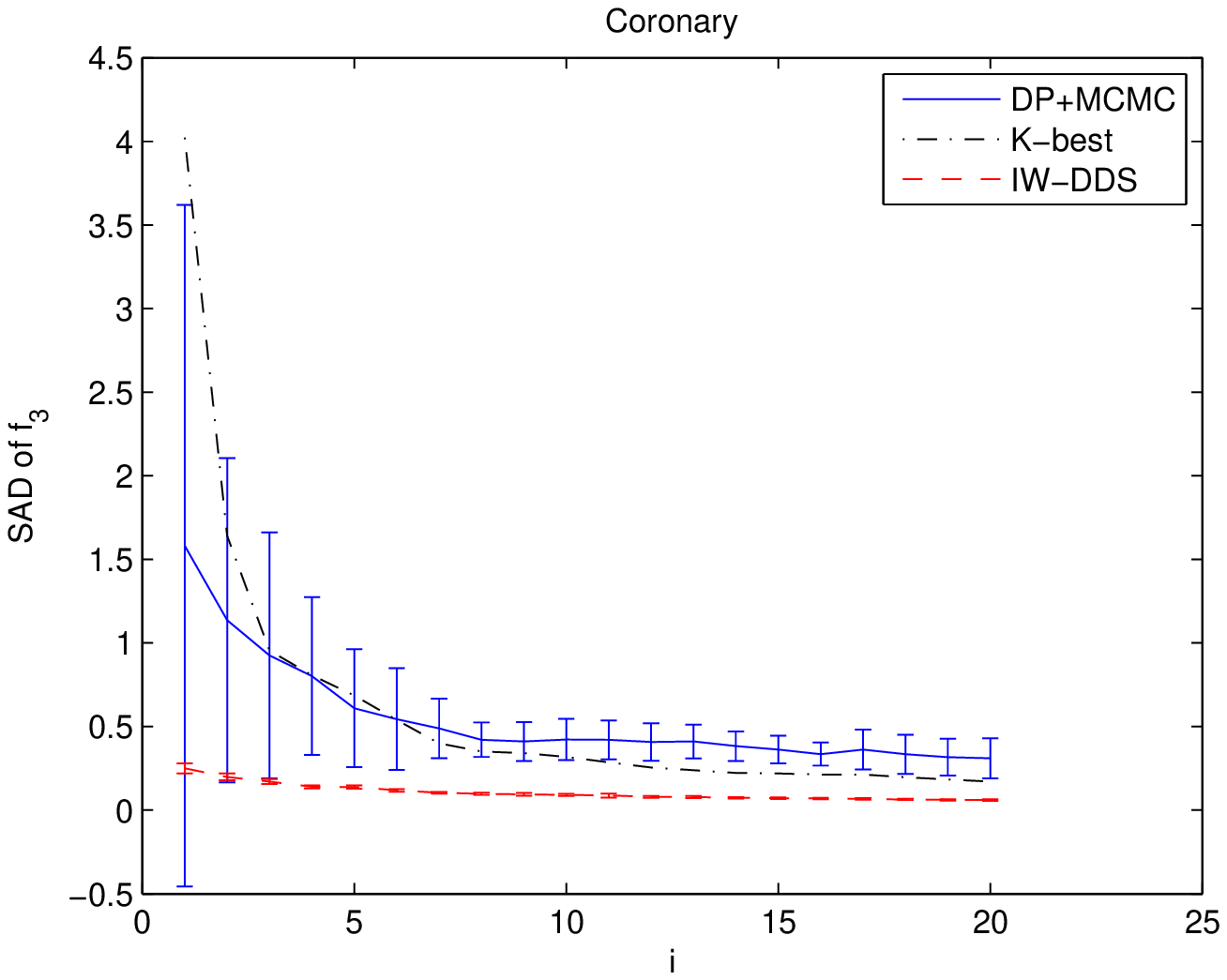}
  \caption{SAD of the Learned $f_3$ Features for Coronary}
  \label{fig-SAD_combined_two_path_coronary_perm}
\end{minipage}

\qquad

\begin{minipage}[b]{0.46\linewidth}
  %\centering
  \includegraphics[width=1 \linewidth]{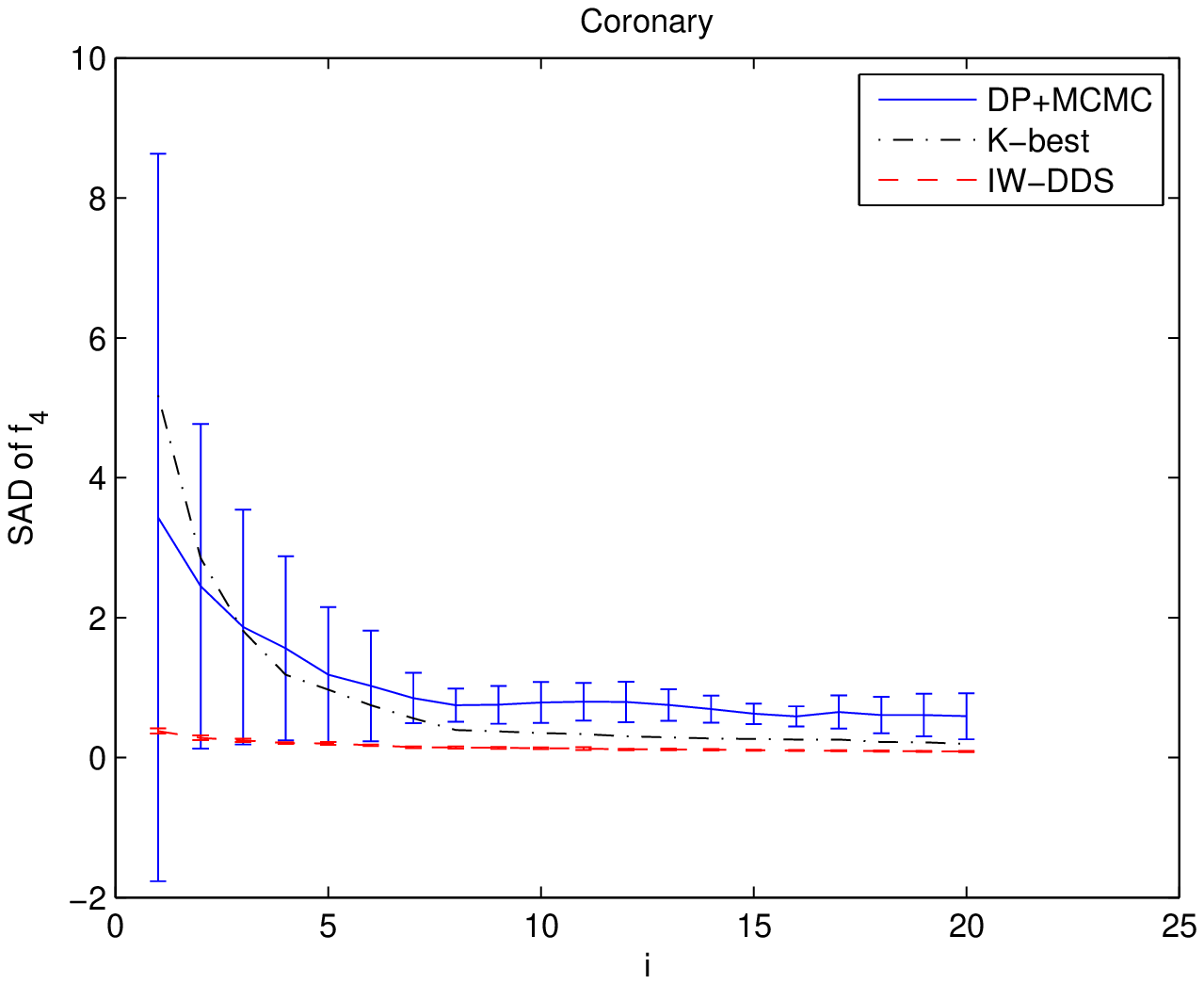}
  \caption{SAD of the Learned $f_4$ Features for Coronary}
  \label{fig-SAD_combined_two_f4_path_coronary_perm}
\end{minipage}
\qquad
\begin{minipage}[b]{0.46\linewidth}
  %\centering
  \includegraphics[width=1 \linewidth]{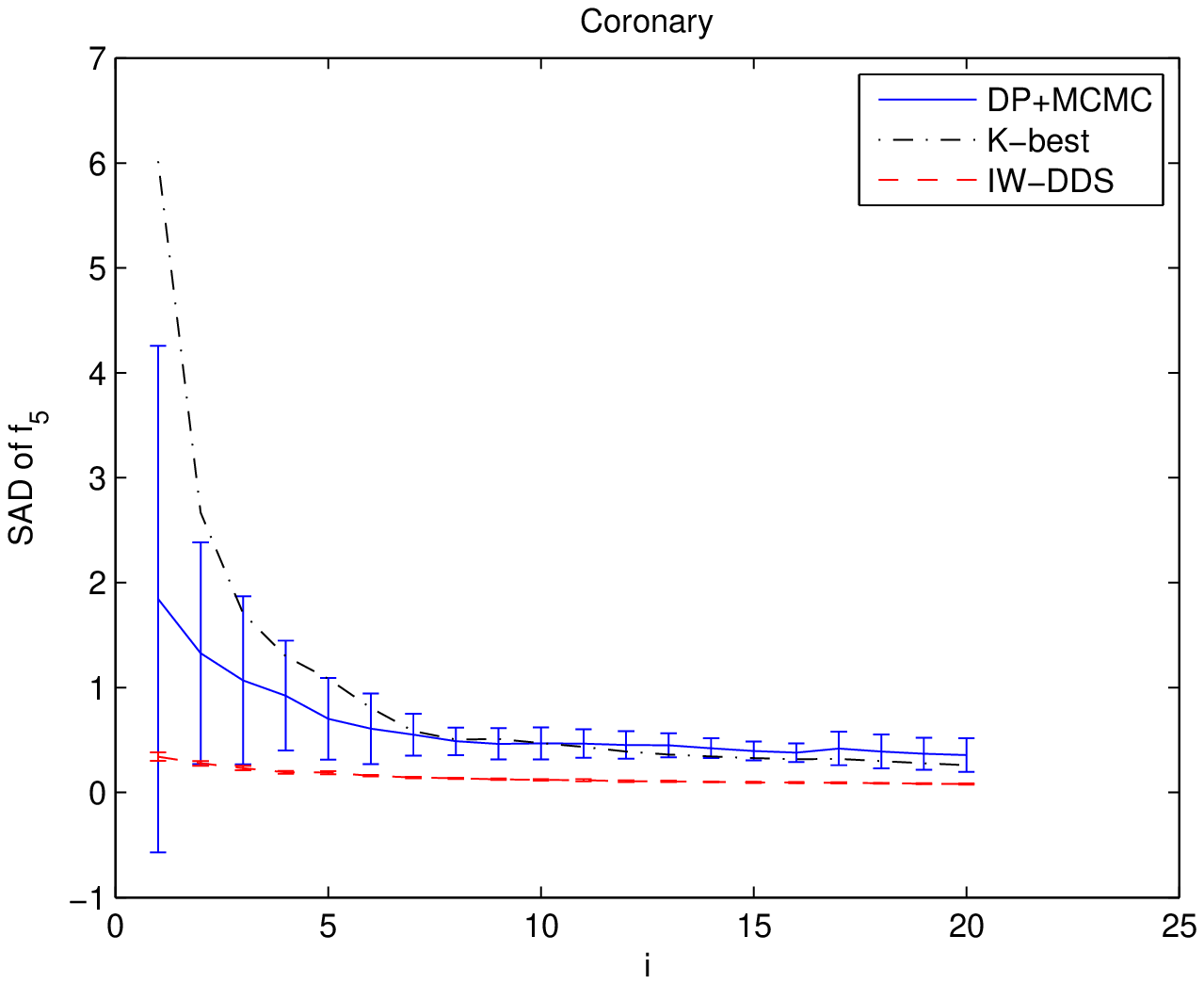}
  \caption{SAD of the Learned $f_5$ Features for Coronary}
  \label{fig-SAD_combined_two_f5_path_coronary_perm}
\end{minipage}

\end{figure}

\begin{figure}
\centering

\begin{minipage}[b]{0.46\linewidth}
  %\centering
  \includegraphics[width=1 \linewidth]{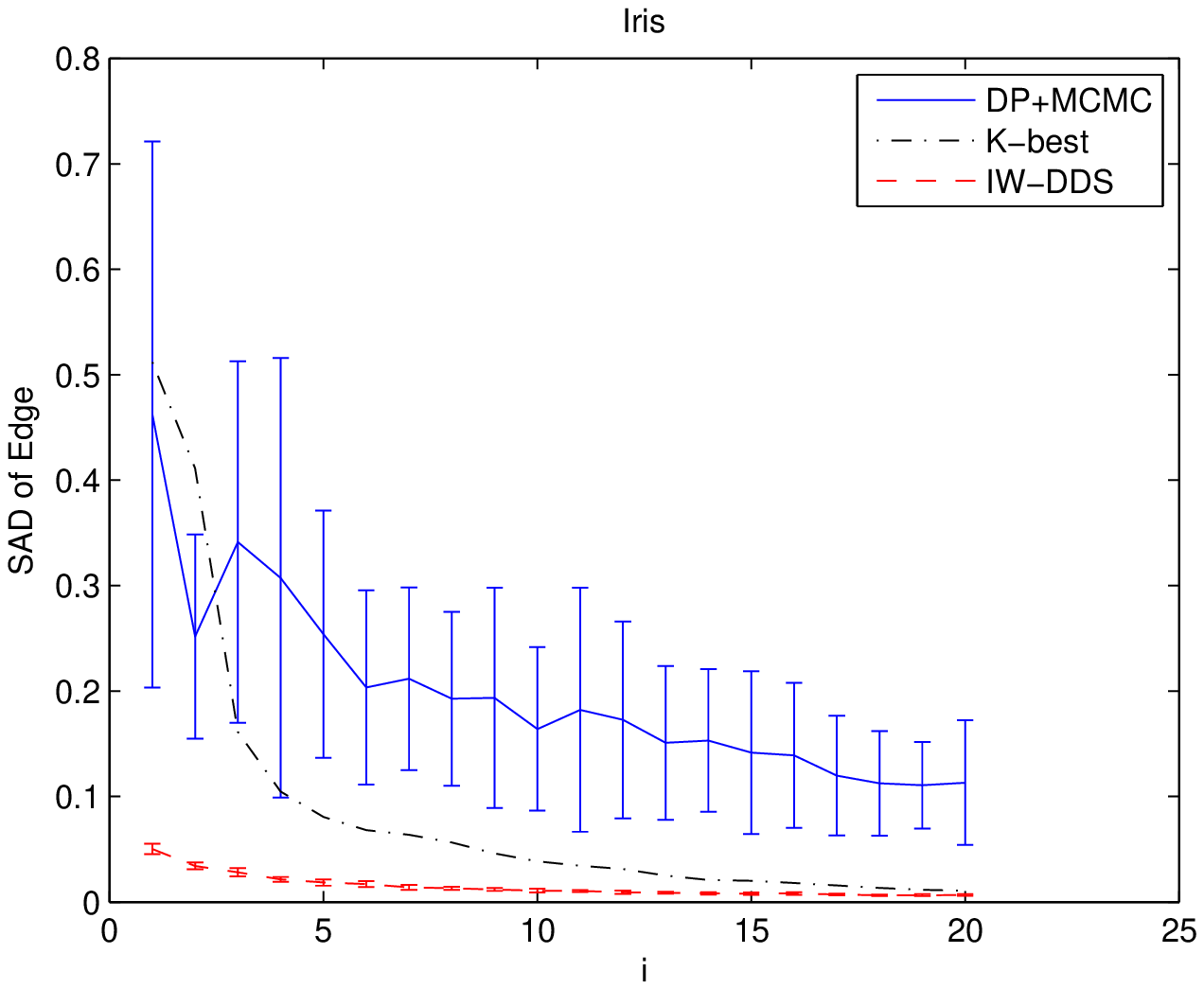}
  \caption{SAD of the Learned Edge Features for Iris}
  \label{fig-SAD_edge_iris_samp150}
\end{minipage}
\qquad
\begin{minipage}[b]{0.46\linewidth}
  %\centering
  \includegraphics[width=1 \linewidth]{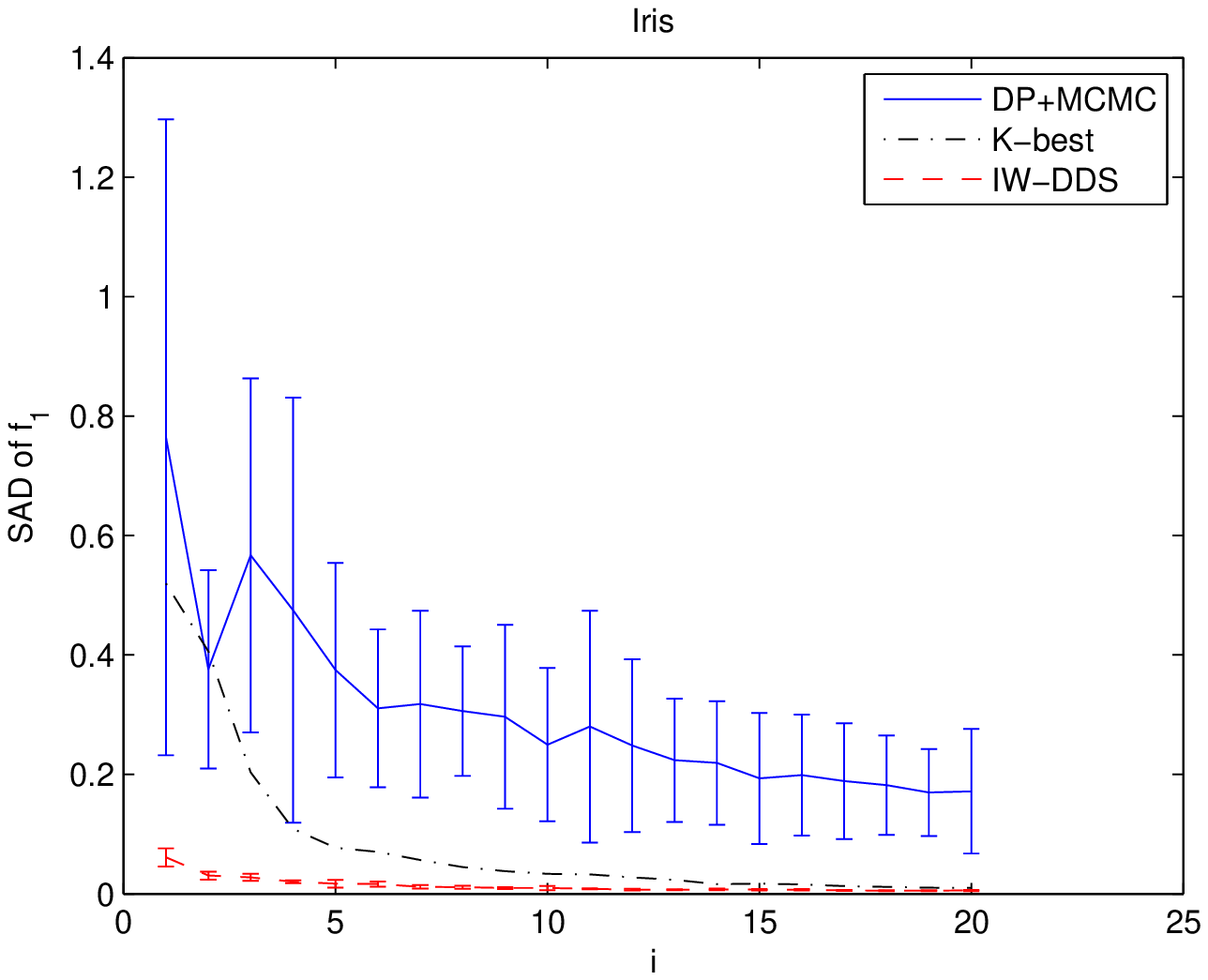}
  \caption{SAD of the Learned $f_1$ Features for Iris}
  \label{fig-SAD_path_iris_samp150}
\end{minipage}

\qquad

\begin{minipage}[b]{0.46\linewidth}
  %\centering
  \includegraphics[width=1 \linewidth]{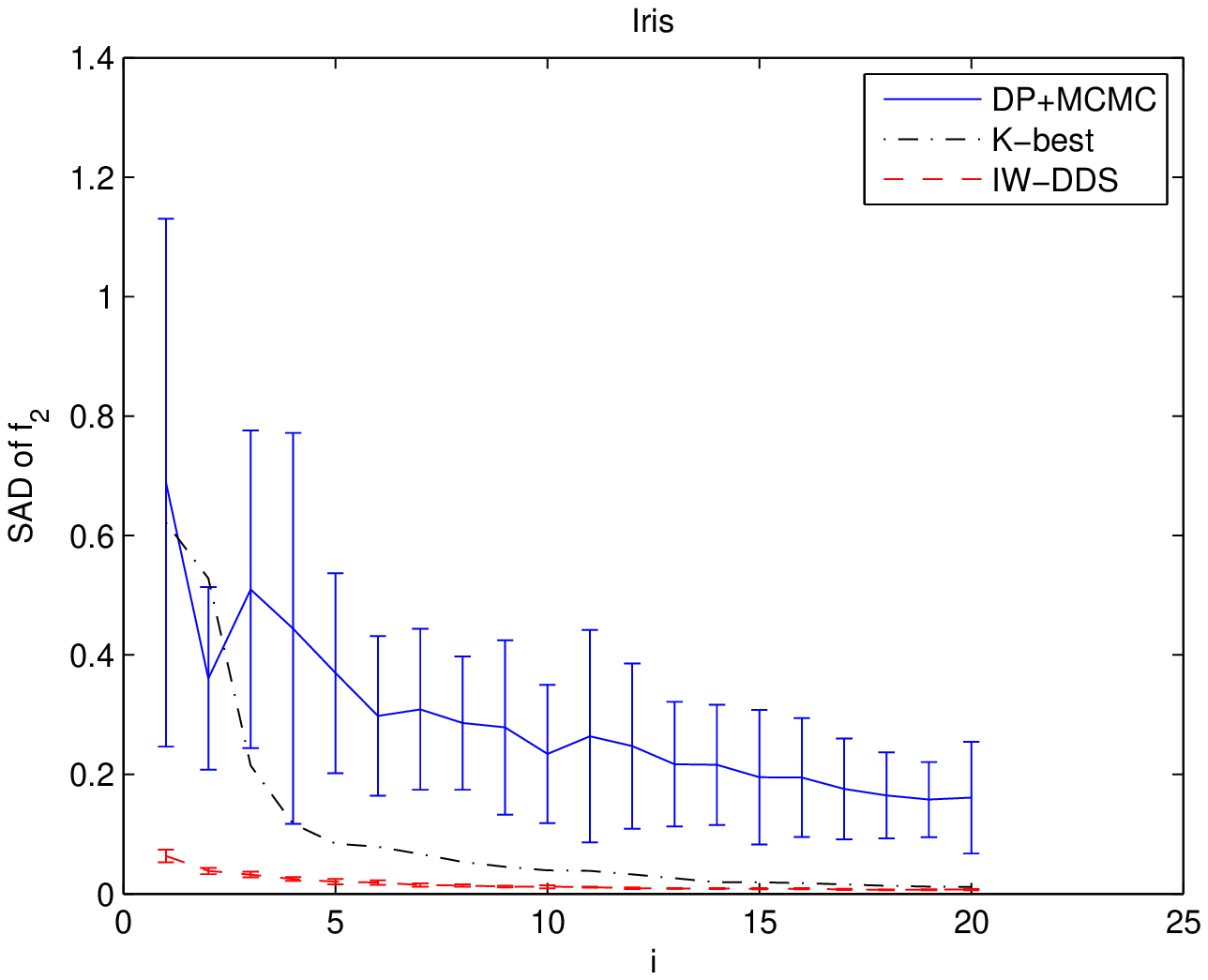}
  \caption{SAD of the Learned $f_2$ Features for Iris}
  \label{fig-SAD_limited_leng_path_iris_samp150}
\end{minipage}
\qquad
\begin{minipage}[b]{0.46\linewidth}
  %\centering
  \includegraphics[width=1 \linewidth]{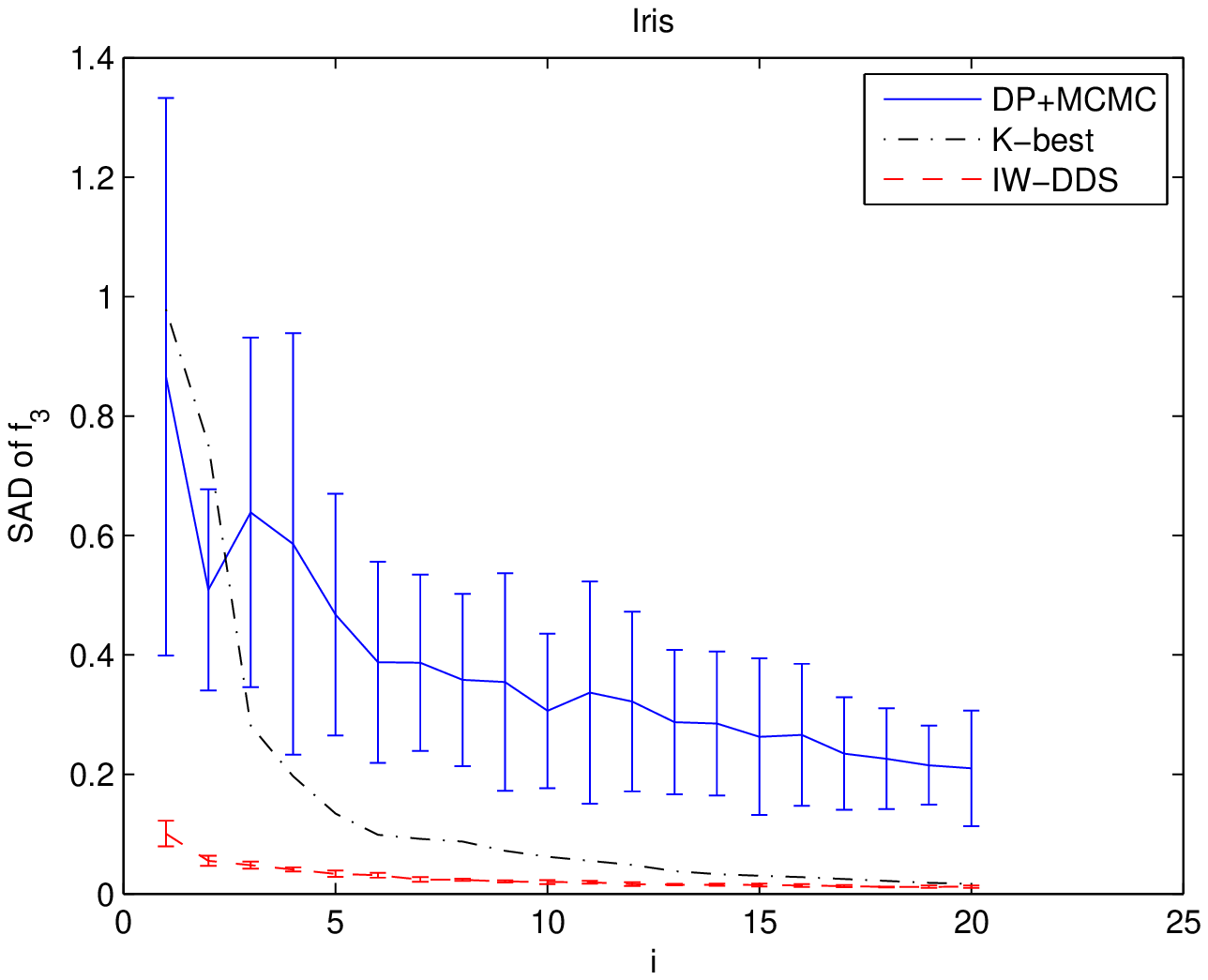}
  \caption{SAD of the Learned $f_3$ Features for Iris}
  \label{fig-SAD_combined_two_path_iris_samp150}
\end{minipage}

\qquad

\begin{minipage}[b]{0.46\linewidth}
  %\centering
  \includegraphics[width=1 \linewidth]{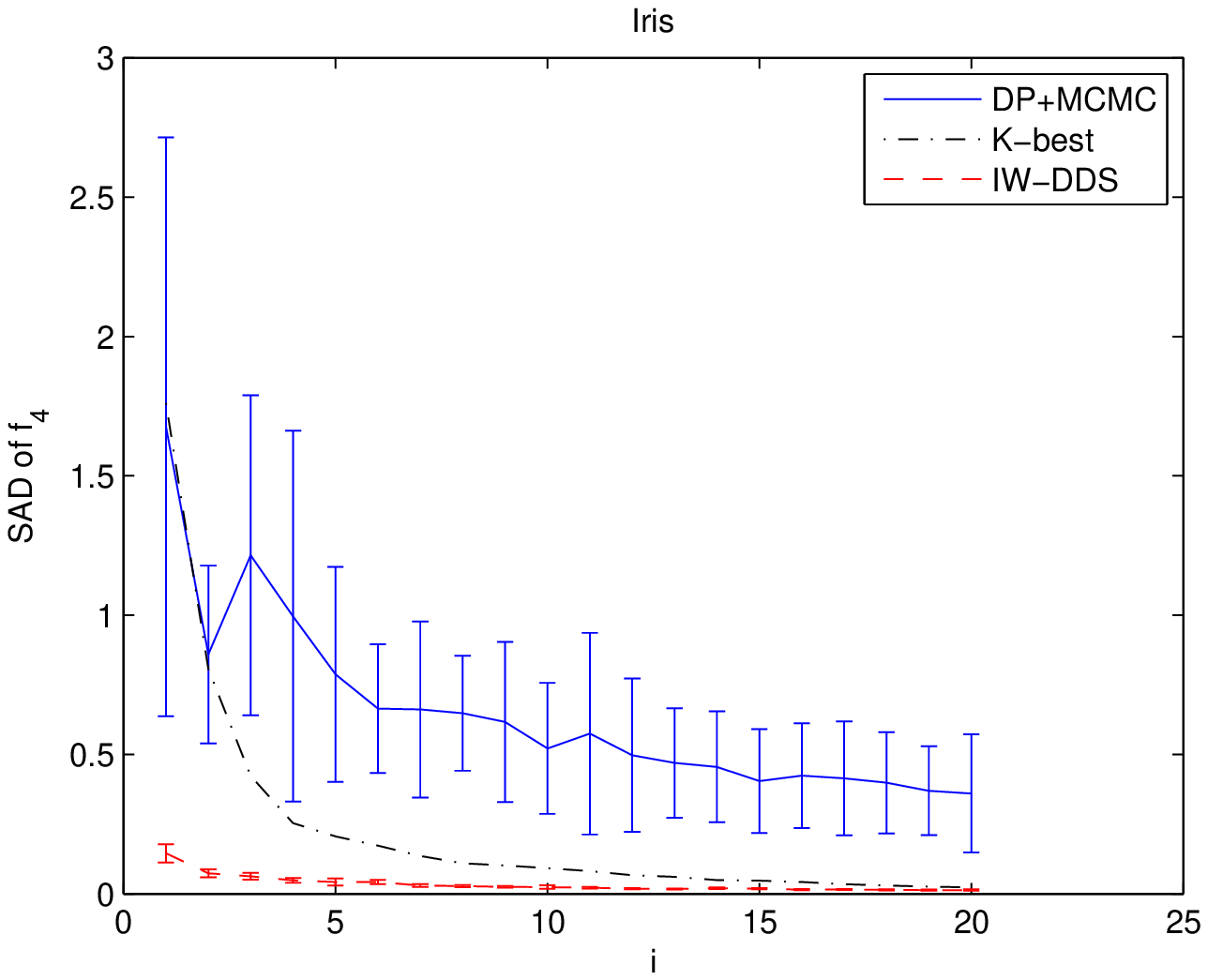}
  \caption{SAD of the Learned $f_4$ Features for Iris}
  \label{fig-SAD_combined_two_f4_path_iris_samp150}
\end{minipage}
\qquad
\begin{minipage}[b]{0.46\linewidth}
  %\centering
  \includegraphics[width=1 \linewidth]{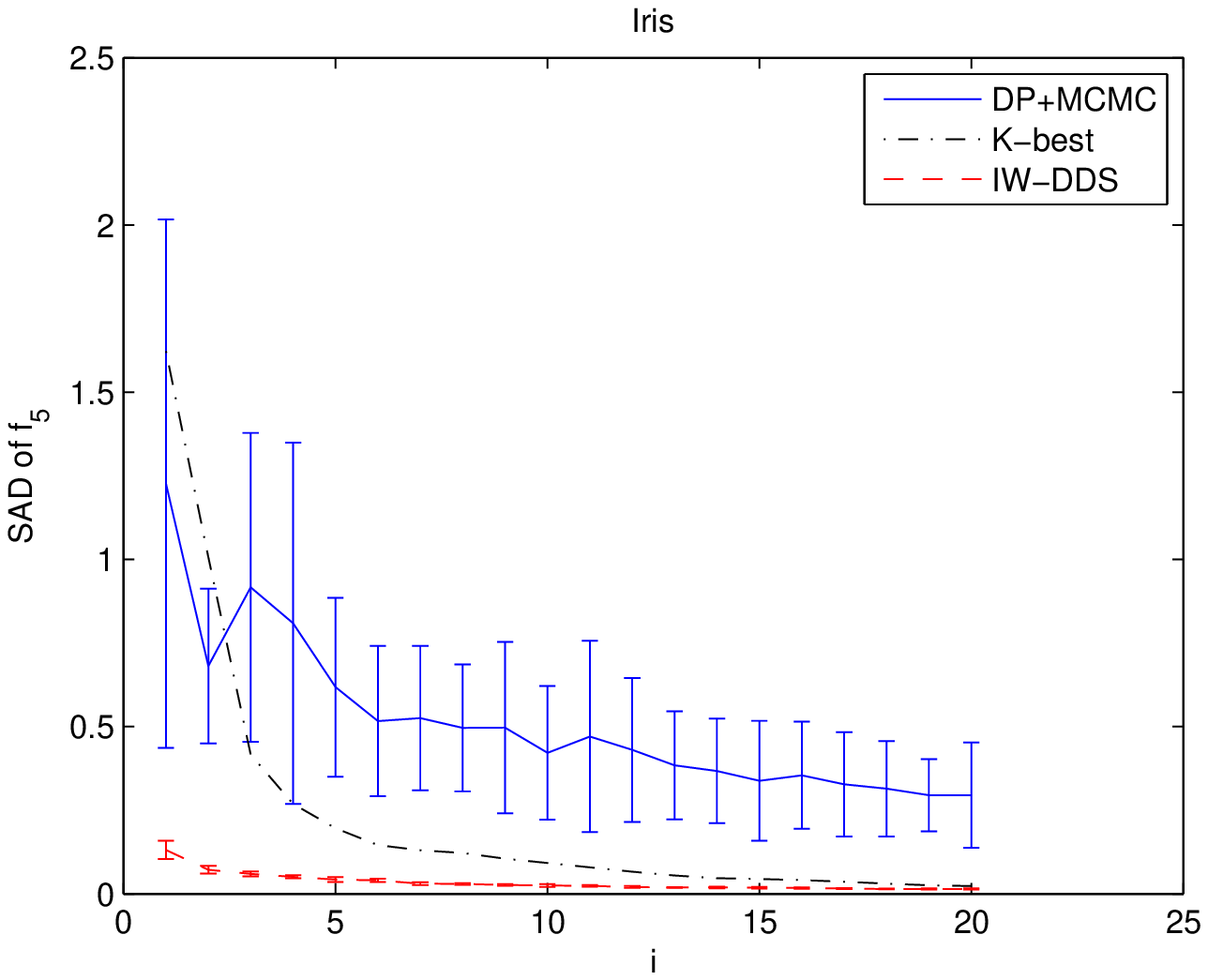}
  \caption{SAD of the Learned $f_5$ Features for Iris}
  \label{fig-SAD_combined_two_f5_path_iris_samp150}
\end{minipage}

\end{figure}

%Here we describe experiments introduced in Section \ref{sec-expr-path} in the details.
The following
is our experimental design for
the data set Coronary  with $n = 6$ and $m = 1841$.
%Iris data set with $n = 5$ and $m = 150$.
For the IW-DDS,
we tried the sample size %$N_o$ $ = 1,000 \cdot i$
$N_o$ $ = 2,500 \cdot i$,
%$1,000 \cdot i$
where $i \in \{1, 2, \ldots, 20 \}$.
For each $i$, we independently ran the IW-DDS $20$ times to get the sample mean and the sample standard deviation of SAD
for the (directed) edge feature and the five non-modular features ($f_1$, $f_2$, $f_3$, $f_4$ and $f_5$).
For the DP+MCMC, we ran $25,000 \cdot i$ MCMC iterations, where $i \in \{1, 2, \ldots, 20 \}$.
For each $i$, we discarded the first $12,500 \cdot i$ MCMC iterations for ``burn-in''
and set the thinning parameter to be $5$
so that $2,500 \cdot i$ DAGs got sampled.
%Again,
For each $i$, we independently ran the MCMC $20$ times to get the sample mean and the sample standard deviation of SAD
for the edge feature, $f_1$, $f_2$, $f_3$, $f_4$ and $f_5$.
%\footnote{
%The purpose of this experiment setting for the DP+MCMC is also just to show that
%the claimed ``if-then'' conditional statement holds.
%}
For the $K$-best,
we ran the $K$-best with $K = 10 \cdot i$, where $i \in \{1, 2, \ldots, 20 \}$.
For each $i$, we ran the $K$-best just once to get SAD
for the edge feature, $f_1$, $f_2$, $f_3$, $f_4$ and $f_5$ since
there is no randomness in
the outcome of the $K$-best algorithm.
\footnote{
%$K = 10 \cdot i$ is somehow chosen arbitrarily.
The purpose of the experimental setting for the DP+MCMC and the $K$-best is merely to testify
the claimed ``if-then'' conditional statement:
if the SAD performance from the IW-DDS is significantly better than the SAD performance from the competing method for an edge feature,
then the SAD performance from the IW-DDS will also be significantly better than the one from the competing method for each investigated non-modular feature using the same set of DAG samples.
%Our purpose is not to show that the learning performance of the IW-DDS is significantly better than
%the performance of the $K$-best in this data set.
}

The experimental results for the data set Coronary are demonstrated
from Figure \ref{fig-SAD_edge_coronary_perm} to Figure \ref{fig-SAD_combined_two_f5_path_coronary_perm}.
Figure
%\ref{fig-SAD_edge_iris_samp150}
\ref{fig-SAD_edge_coronary_perm}
shows the SAD performance of the three methods with each $i$ for the edge feature,
where an error bar represents
one sample standard deviation across $20$ runs for the DP+MCMC or the IW-DDS
at each $i$.
Correspondingly,
Figure \ref{fig-SAD_path_coronary_perm} to Figure \ref{fig-SAD_combined_two_f5_path_coronary_perm}
show the SAD performance of the three methods with each $i$ for the five investigated non-modular features ($f_1$, $f_2$, $f_3$, $f_4$ and $f_5$)
respectively.
Combining
Figure \ref{fig-SAD_edge_coronary_perm} and
each of Figures
\ref{fig-SAD_path_coronary_perm},
\ref{fig-SAD_limited_leng_path_coronary_perm},
\ref{fig-SAD_combined_two_path_coronary_perm},
\ref{fig-SAD_combined_two_f4_path_coronary_perm},
and \ref{fig-SAD_combined_two_f5_path_coronary_perm},
one can clearly see that
if the SAD of the IW-DDS is significantly smaller than
the SAD of the competing method
(the DP+MCMC or the $K$-best)
for the edge feature,
then the SAD of the IW-DDS will also be
significantly
smaller than
the SAD of the competing method
for each of the five investigated non-modular features.
%From these six figures,
%one can clearly see that
%the position of the SAD curve of IW-DDS relative to the competing method
%for the edge feature
%for each learning method,
%its SAD curve for each $f_j$ has a shape similar to its SAD curve for the edge feature
More specifically,
comparing with the DP+MCMC,
at each $i \in \{1, 2, \ldots 20 \}$,
for the edge feature,
the real mean of SAD from the IW-DDS is significantly smaller than
the corresponding one from the DP+MCMC with the p-value $< 0.01$ from the two-sample $t$ test with unequal variances.
Consistently,
at each $i \in \{1, 2, \ldots 20 \}$,
for each investigated non-modular feature $f_j$,
the real mean of SAD from the IW-DDS is also significantly smaller than
the corresponding one from the DP+MCMC with the p-value $< 0.01$ from the same $t$ test.
Comparing with the $K$-best,
at each $i \in \{1, 2, \ldots 20 \}$,
for the edge feature,
the real mean of SAD from the IW-DDS is significantly smaller than
the SAD from the $K$-best with the p-value $< 1 \times 10^{-20}$ from the one-sample $t$ test.
Consistently,
at each $i \in \{1, 2, \ldots 20 \}$,
for each investigated non-modular feature $f_j$,
the real mean of SAD from the IW-DDS is also significantly smaller than
the SAD from the $K$-best with the p-value $< 1 \times 10^{-20}$ from the same $t$ test.

We also performed the same kind of experiments for
%Coronary data set with $n = 6$ and $m = 1841$.
the data set Iris  with $n = 5$ and $m = 150$.
The results are
%shown in
demonstrated from
Figure \ref{fig-SAD_edge_iris_samp150}
%\ref{fig-SAD_path_iris_samp150},
%\ref{fig-SAD_limited_leng_path_iris_samp150},
%and \ref{fig-SAD_combined_two_path_iris_samp150}.
to Figure \ref{fig-SAD_combined_two_f5_path_iris_samp150}.
Comparing with the DP+MCMC,
at each $i \in \{1, 2, \ldots 20 \}$,
just as the comparison result for the edge feature,
for each investigated $f_j$,
the real mean of SAD from the IW-DDS is also significantly smaller than
the corresponding one from the DP+MCMC with the p-value $< 1 \times 10^{-5}$ from the two-sample $t$ test with unequal variances.
Comparing with the $K$-best,
at each $i \in \{1, 2, \ldots 20 \}$,
just as the comparison result for the edge feature,
for each investigated $f_j$,
the real mean of SAD from the IW-DDS is also significantly smaller than
the SAD from the $K$-best with the p-value $< 1 \times 10^{-9}$ from the one-sample $t$ test.
Thus, a conclusion similar to the one from Coronary can be drawn:
% to support our expectation:
if the SAD of the IW-DDS is significantly smaller than
the SAD of the competing method for the edge feature,
then the SAD of the IW-DDS will also be
significantly smaller than
the SAD of the competing method for each of the five investigated non-modular features.
%and such a SAD decrease due to the IW-DDS is usually significant.

%\subsection{Experimental Results for Performance Guarantee for DDS Algorithm}
%\subsection{EXPERIMENTAL RESULTS FOR PERFORMANCE GUARANTEE FOR DDS ALGORITHM}
%\subsection{PERFORMANCE GUARANTEE FOR THE DDS ALGORITHM}
\subsection{Performance Guarantee for the DDS Algorithm}
To testify the
%performance guarantee
quality guarantee
for the estimator based on the DDS algorithm (Corollary \ref{coro-1} (iv)),
we performed experiments based on two data cases (Letter with $m = 100$ and Tic-Tac-Toe),
which have relatively large $\hat{\mu}$(SAD) or $\hat{\mu}$(MAD) ( $= \hat{\mu}$(SAD) $ /(n(n-1))$ ) from the DDS algorithm shown in Table \ref{tb:bias_comparison_sad_1}.
Based on the hypothesis testing approach, we can conclude with very strong evidence
that the performance guarantee for our estimator holds
for both data cases.
%Please refer to the supplementary material for the details.
%The details are omitted due to the space constraint.
The details of the experiments are as follows.

%\subsection{Supplementary Experimental Results for Performance Guarantee}
%\subsection{SUPPLEMENTARY EXPERIMENTAL RESULTS FOR PERFORMANCE GUARANTEE}

%In this subsection we describe our performed experiments for
%testifying the performance guarantee for our estimator based on
%DDS algorithm.
For the first set of experiments,
we choose the data case Letter with $m = 100$,
which has the largest $\hat{\mu}$(SAD) ($= 0.2948$) from the DDS among all the 33 data cases shown in Table \ref{tb:bias_comparison_sad_1}.
%for the first set of experiments.
We first %fix
consider the setting of the parameters
specified as $\epsilon$ $= 0.02$ and $\delta$ $= 0.05$,
which serves as our performance requirement.
By setting the DAG sample size $N_o$
%$= \lceil \frac{\ln (2/\delta)}{2\epsilon^2} \rceil$ $= 4612$,
$= \lceil (\ln (2/\delta)) / (2\epsilon^2) \rceil$ $= 4612$,
we intend to show that the estimator $\hat{p}_{\prec}(f|D)$ coming from our DDS
has the performance guarantee such that
the Hoeffding inequality $p( |\hat{p}_{\prec}(f|D) - p_{\prec}(f|D)| \geq \epsilon) \leq \delta$ holds.
%For verification purpose,
%Here we instigate each directed edge feature $f$.
Each directed edge feature $f$ is investigated here since
the posterior of each edge $p_{\prec}(f|D)$ can be easily obtained by using the DP algorithm
%\citep{Koivisto:06}.
of \citet{Koivisto:06}.
%Meanwhile, the corresponding $\hat{p}_{\prec}(f|D)$ based on $N_o$ ($ = 4612$)
%DAG samples can be obtained by our DDS algorithm.
For each edge $f$,
we call the event of $ |\hat{p}_{\prec}(f|D) - p_{\prec}(f|D)| \geq \epsilon$ as the event of
%uncontrolled deviation
violation (of the pre-specified estimation error bound)
in
the learning of $f$.
Define the indication variable $W$ for the event of violation in
the learning of $f$.
Thus, $W$ is a Bernoulli random variable with the success probability
$p_{vio}$ $= p( |\hat{p}_{\prec}(f|D) - p_{\prec}(f|D)| \geq \epsilon)$.
We independently repeat the DDS algorithm (with the same $N_o$) $R = 400$ times
and use the average of $W$ as the estimator $\hat{p}_{vio}$ for each edge.
%Intuitively, $\hat{p}_{vio}$ for each edge should be smaller than $\delta = 0.05$ if the Hoeffding inequality holds.
%Therefore,
Note that the mean of $\hat{p}_{vio}$ is $p_{vio}$ and the variance of $\hat{p}_{vio}$ is $p_{vio} (1 - p_{vio}) / R$
because $R \hat{p}_{vio}$
has a binomial distribution with the trial number $R$ and success probability $p_{vio}$.
Since we expect that $p_{vio}$ will be small
so that $p_{vio} (1 - p_{vio})$ will be large
%with respect to $p_{vio}$,
relative to $p_{vio}$,
we choose large $R = 400$ to make the variance of $\hat{p}_{vio}$ relatively small with respect to the mean of $\hat{p}_{vio}$.
Figure \ref{fig-Hoeffding_Letter_No_4612} shows the histogram of $\hat{p}_{vio}$ for each of all the $n(n-1)$ $= 272$ directed edges.
For each of the $272$ edges,
it can be clearly seen that the corresponding $\hat{p}_{vio}$ is much smaller than $\delta = 0.05$ marked by the vertical bar.
For $240$ out of the $272$ edges,  the corresponding $\hat{p}_{vio}$'s are exactly equal to $0$.
Even for the largest $\hat{p}_{vio}$ $ = 0.015$, corresponding to $6$ successes among $400$ trials,
we can use the one-sided hypothesis testing to reject the null hypothesis that $p_{vio}$ $ \geq 0.05$ and
to conclude that $p_{vio}$ $ < 0.05$ with the p-value
less than $2 \times 10^{-4}$.
Therefore, the Hoeffding inequality holds for the learning of each edge in this parameter setting.

Next, we consider another setting of the parameters with $\epsilon$ $= 0.01$ and $\delta$ $= 0.02$,
which has more demanding performance requirement.
By setting the DAG sample size $N_o$
$= \lceil (\ln (2/\delta)) / (2\epsilon^2) \rceil$ $= 23026$,
we want to show that the estimator $\hat{p}_{\prec}(f|D)$ coming from our DDS
has the performance guarantee satisfying the Hoeffding inequality.
With the same logic,
we independently repeat the DDS algorithm (with the same $N_o$) $R = 1250$ times
and use the average of $W$ as the estimator $\hat{p}_{vio}$ for each edge.
(Here we choose even larger $R$ since we expect that $p_{vio}$ will be smaller in
this parameter setting.)
Figure \ref{fig-Hoeffding_Letter_No_23026} shows the histogram of $\hat{p}_{vio}$
for each of all the $272$ directed edges.
For each edge,
it can be clearly seen that the corresponding $\hat{p}_{vio}$ is much smaller than $\delta = 0.02$.
Even for the largest $\hat{p}_{vio}$ $ = 0.004$, corresponding to $5$ successes among $1250$ trials,
we can use the one-sided hypothesis testing to reject the null hypothesis that $p_{vio}$ $ \geq 0.02$ and
to conclude that $p_{vio}$ $ < 0.02$ with the p-value
less than $2 \times 10^{-6}$.
Therefore, the Hoeffding inequality also holds in this parameter setting.

For the second set of experiments,
we choose the data case Tic-Tac-Toe,
which has the largest $\hat{\mu}$(MAD)  ($= \hat{\mu}$(SAD) $/ (n(n-1)) $ $= 0.1547 / 90$) from the DDS among all the 33 data cases
shown in Table \ref{tb:bias_comparison_sad_1}.
The same kind of experiments are performed for this data case.
For the parameter setting with $\epsilon$ $= 0.02$ and $\delta$ $= 0.05$, the corresponding result is shown in Figure \ref{fig-Hoeffding_Tic_Tac_Toe_No_4612}.
For each of the $90$ edges, the corresponding $\hat{p}_{vio}$ is clearly much smaller than $\delta = 0.05$.
Even for the largest $\hat{p}_{vio}$ $ = 0.0125$, corresponding to $5$ successes among $400$ trials,
we can use the one-sided hypothesis testing to conclude that $p_{vio}$ $ < 0.05$ with the p-value
less than $6 \times 10^{-5}$.
For the parameter setting with $\epsilon$ $= 0.01$ and $\delta$ $= 0.02$, the corresponding result is shown in Figure \ref{fig-Hoeffding_Tic_Tac_Toe_No_23026}.
For each edge, the corresponding $\hat{p}_{vio}$ can be clearly seen to be much smaller than $\delta = 0.02$.
Even for the largest $\hat{p}_{vio}$ $ = 0.0056$, corresponding to $7$ successes among $1250$ trials, we can use the one-sided hypothesis testing to conclude that $p_{vio}$ $ < 0.02$ with the p-value
less than $3 \times 10^{-5}$.
Thus, the Hoeffding inequality also holds in the set of experiments for this data case.

Finally, for the data case Tic-Tac-Toe,
we fix $\epsilon$ $= 0.02$
but increase $N_o$ from $2,000$ to $10,000$ by an increment of $1,000$ each time.
%$\{2000, 4000, 6000, 8000, 10000, 120000\}$
For each $N_o$,
%$p(|\hat{p}_{\prec}(f|D) -p_{\prec}(f|D)|\geq \epsilon) \leq 2e^{-2 N_o \epsilon^2}.$
we plot the Hoeffding bound $2e^{-2 N_o \epsilon^2}$
for the probability of violation
$p_{vio}$ $= p( |\hat{p}_{\prec}(f|D) - p_{\prec}(f|D)| \geq \epsilon)$
in Figure~\ref{fig-Hoeffding_Tic_Tac_Toe_No_2000_10000}.
(Note that the Hoeffding bound decreases at an exponential rate as $N_o$ increases.)
Then for each $N_o$,
we also plot both the maximum and the mean of all the $n(n-1)$ $\hat{p}_{vio}$'s
in Figure~\ref{fig-Hoeffding_Tic_Tac_Toe_No_2000_10000}.
Again $\hat{p}_{vio}$ for each edge is the average of W
by independently running the DDS algorithm (with the same $N_o$) $R$ times.
We set $R = max\{400, \lceil 10 / (e^{-2 N_o \epsilon^2}) \rceil \}$
and use the larger $R$ for the larger $N_o$
since we expect that $\hat{p}_{vio}$ will be smaller for the larger $N_o$.
%Then we can check whether Hoeffding bound holds for our estimator from DDS for each $N_o$.
From Figure~\ref{fig-Hoeffding_Tic_Tac_Toe_No_2000_10000},
we can clearly see that
$\hat{p}_{vio}^*$,
the maximum of all the $n(n-1)$ $\hat{p}_{vio}$'s, is always far below the Hoeffding bound for each $N_o$.
Furthermore,
for each $N_o$
we can use the one-sided hypothesis testing to reject
the null hypothesis that $p_{vio}$ $ \geq 2e^{-2 N_o \epsilon^2}$ and
to conclude that $p_{vio}$ $ < 2e^{-2 N_o \epsilon^2}$ with the p-value
less than $3 \times 10^{-4}$.
Therefore, the Hoeffding inequality holds %in the learning
for each $N_o$.

\begin{figure}
\centering
\begin{minipage}[b]{0.46\linewidth}
  %\centering
  \includegraphics[width=1 \linewidth]{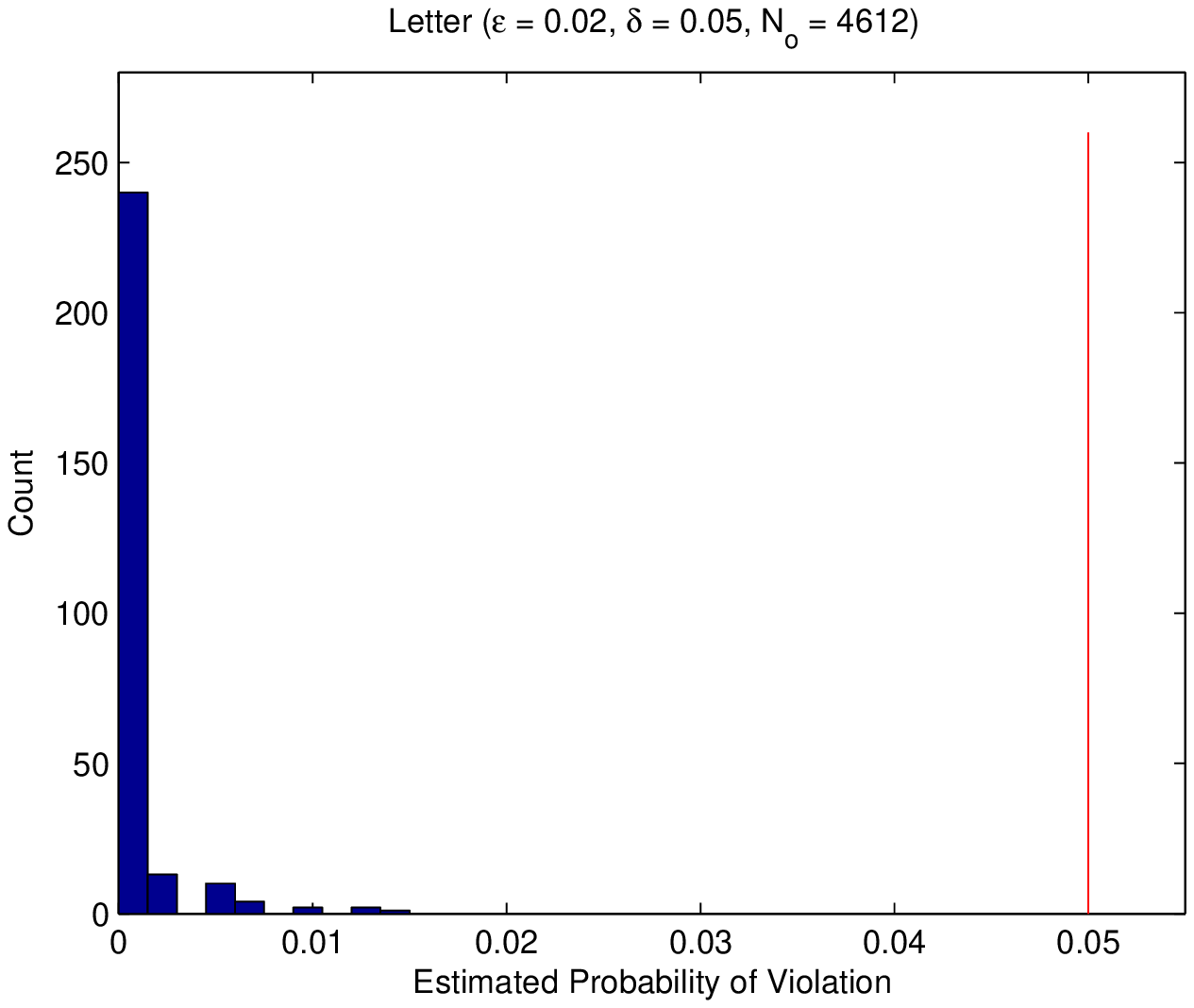}
  \caption{Histogram of Estimated Probabilities of Violation in Edge Learning for Letter
($m = 100$) with $\epsilon = 0.02$, $\delta = 0.05$ and $N_o = 4612$}
  \label{fig-Hoeffding_Letter_No_4612}
\end{minipage}
\qquad
\begin{minipage}[b]{0.46\linewidth}
  %\centering
  \includegraphics[width=1 \linewidth]{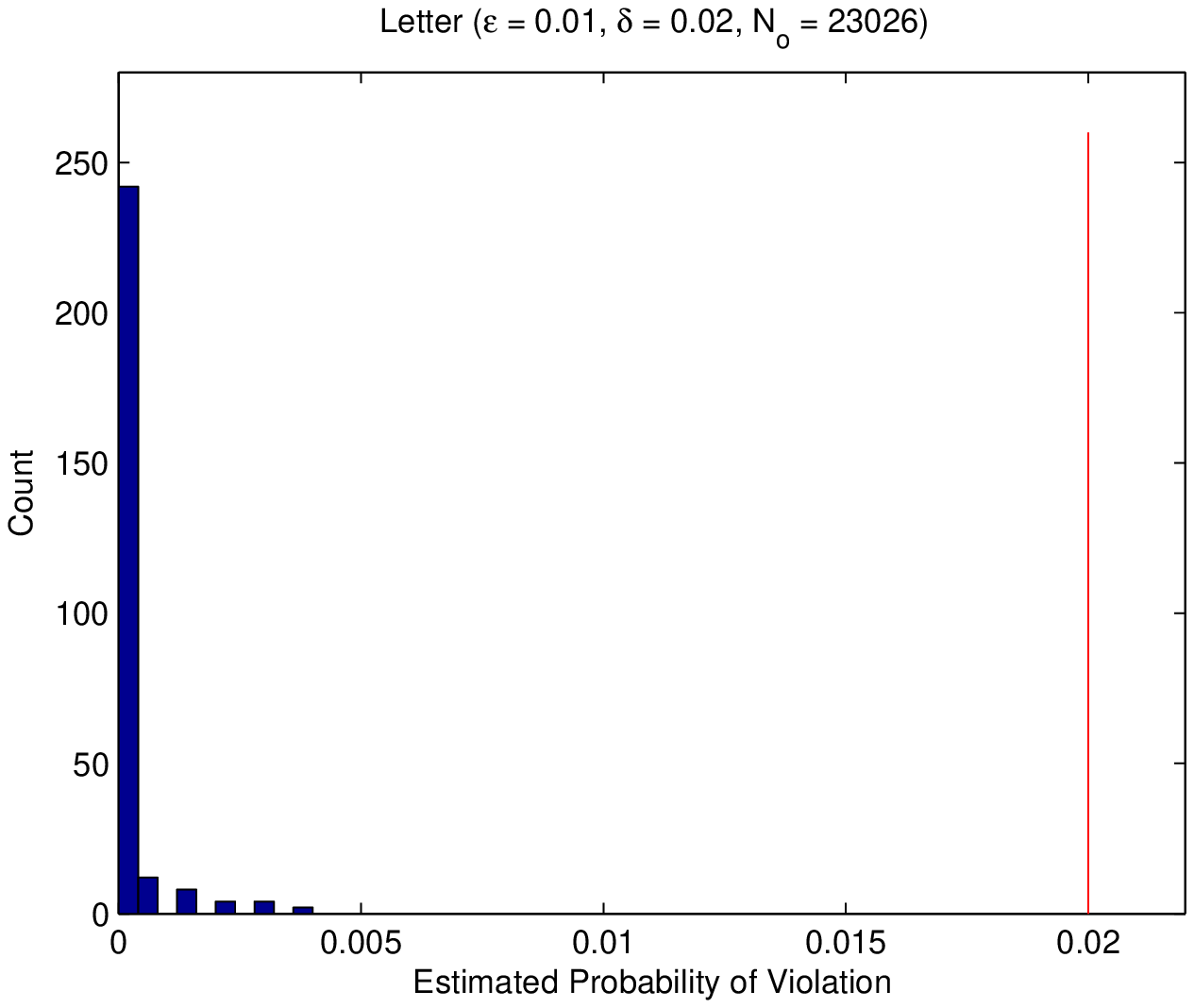}
  \caption{Histogram of Estimated Probabilities of Violation in Edge Learning for Letter
($m = 100$) with $\epsilon = 0.01$, $\delta = 0.02$ and $N_o = 23026$}
  \label{fig-Hoeffding_Letter_No_23026}
\end{minipage}
\end{figure}

\begin{comment}
\begin{figure}
\centering
\begin{minipage}[b]{0.46\linewidth}
  %\centering
  \includegraphics[width=1 \linewidth]{Hoeffding_Syn15_No_4612_new}
  \caption{Histogram of Estimated Probabilities of Violation in Edge Learning for Syn15
($m = 100$) with $\epsilon = 0.02$, $\delta = 0.05$ and $N_o = 4612$}
  \label{fig-Hoeffding_Syn15_No_4612}
\end{minipage}
\qquad
\begin{minipage}[b]{0.46\linewidth}
  %\centering
  \includegraphics[width=1 \linewidth]{Hoeffding_Syn15_No_23026_new}
  \caption{Histogram of Estimated Probabilities of Violation in Edge Learning for Syn15
($m = 100$) with $\epsilon = 0.01$, $\delta = 0.02$ and $N_o = 23026$}
  \label{fig-Hoeffding_Syn15_No_23026}
\end{minipage}
\end{figure}
\end{comment}

\begin{figure}
\centering
\begin{minipage}[b]{0.46\linewidth}
  %\centering
  \includegraphics[width=1 \linewidth]{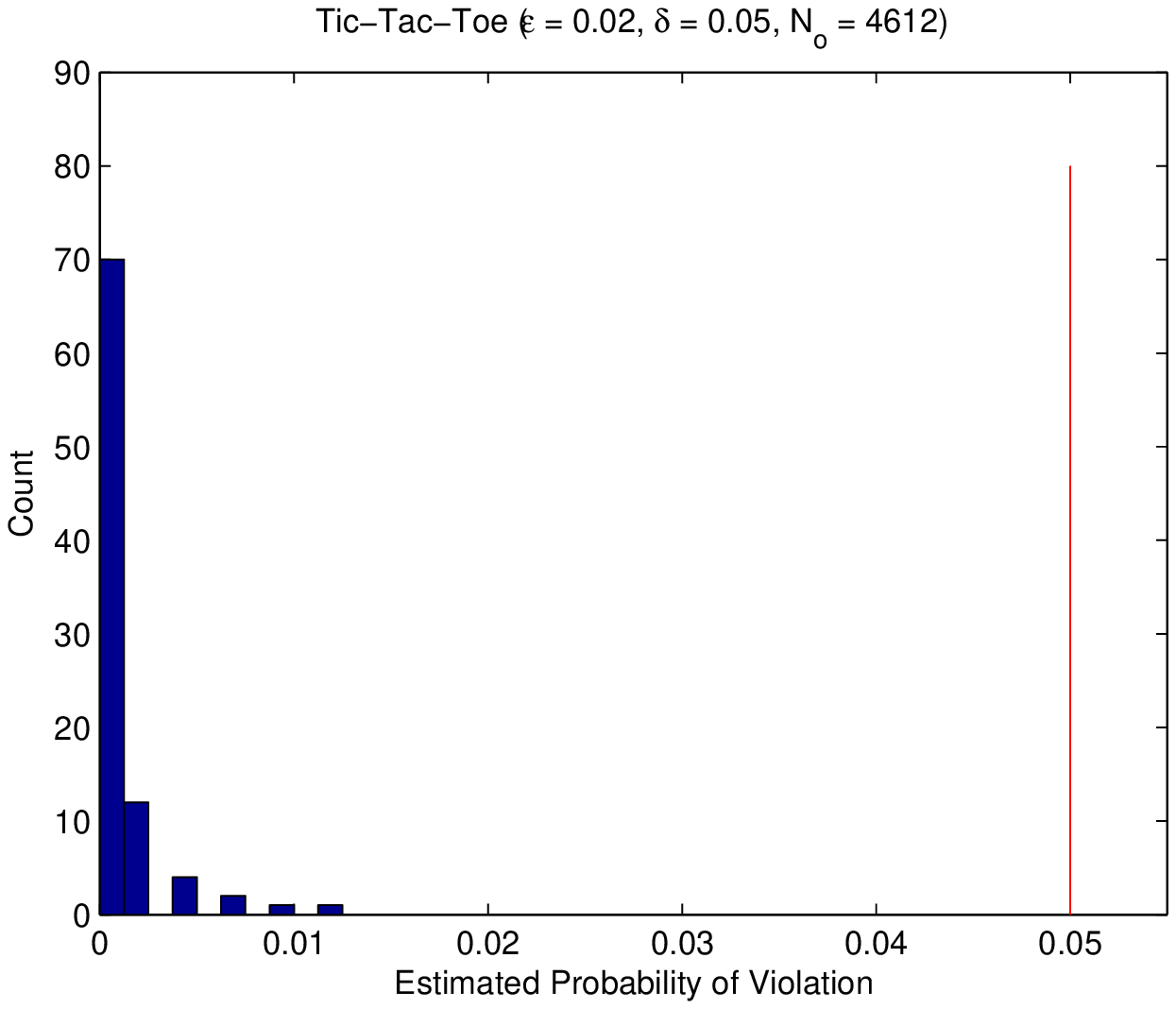}
  \caption{Histogram of Estimated Probabilities of Violation in Edge Learning for Tic-Tac-Toe
  with $\epsilon = 0.02$, $\delta = 0.05$ and $N_o = 4612$}
  \label{fig-Hoeffding_Tic_Tac_Toe_No_4612}
\end{minipage}
\qquad
\begin{minipage}[b]{0.46\linewidth}
  %\centering
  \includegraphics[width=1 \linewidth]{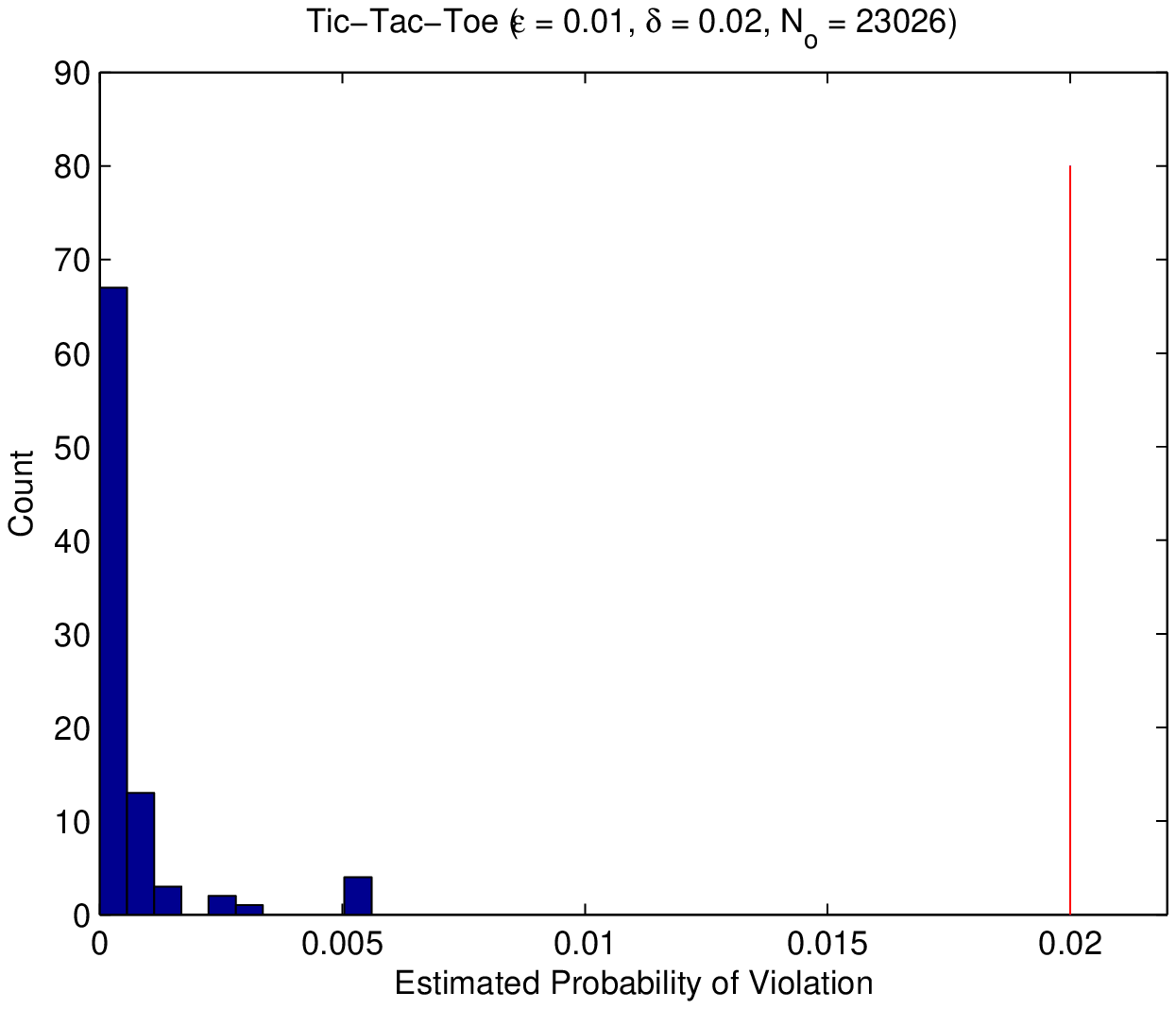}
  \caption{Histogram of Estimated Probabilities of Violation in Edge Learning for Tic-Tac-Toe
  with $\epsilon = 0.01$, $\delta = 0.02$ and $N_o = 23026$}
  \label{fig-Hoeffding_Tic_Tac_Toe_No_23026}
\end{minipage}
\end{figure}

\begin{figure}[ht]
\vskip 0.2in
\begin{center}
%\centerline{\includegraphics[width=0.43\textwidth, height=0.36\textwidth]{Hoeffding_Syn15_No_2000_12000}}
\centerline{\includegraphics[width=1.1 \textwidth, height=0.5 \textwidth]{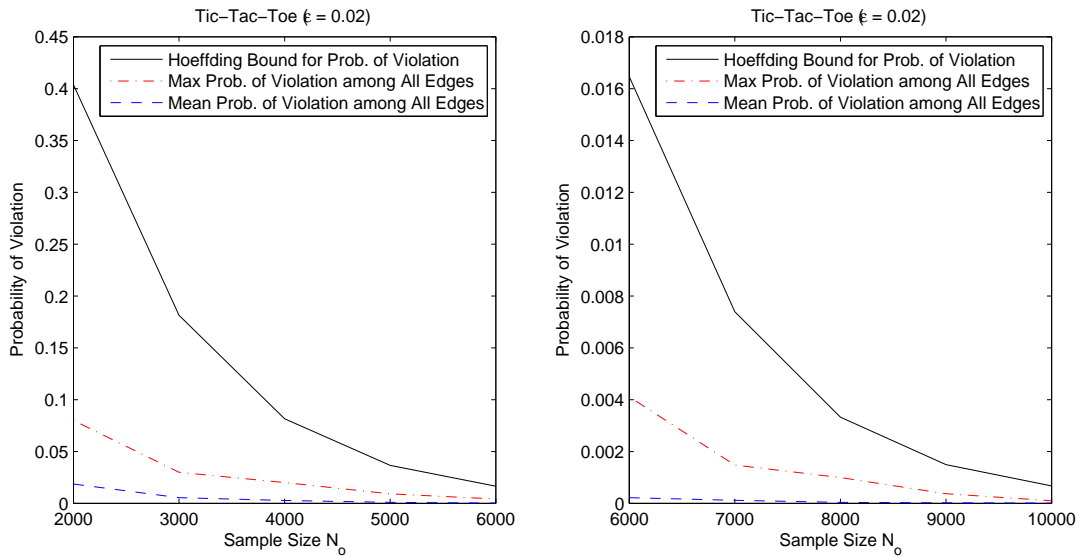}}
\caption{Plot of Probability of Violation versus $N_o$ for Tic-Tac-Toe
with $\epsilon = 0.02$}
\label{fig-Hoeffding_Tic_Tac_Toe_No_2000_10000}
\end{center}
\vskip -0.2in
\end{figure}

%\section{Conclusion \label{sec-con}}
%\section{CONCLUSION \label{sec-con}}
\section{Conclusion \label{sec-con}}
\begin{comment}
We develop new algorithms for efficiently sampling
Bayesian network structures (DAGs) of moderate size.
The sampled DAGs can then be used to
build estimators for
the posteriors of any features.
%Our algorithms serve as the complements to the exact algorithms
%\citep{Koivisto:06,tian:he2009}
%\citep{Koivisto:soo04,Parviainen:ECML2011,tian:he2009}
%to estimate the posteriors of arbitrary non-modular features.
Our algorithms address the limitations of the exact algorithms
\citep{Koivisto:soo04,Parviainen:ECML2011,tian:he2009}
so that the posteriors of arbitrary non-modular features can be estimated
when the size of the Bayesian network is moderate.
We theoretically prove good properties of our estimators.
Finally we
empirically show that the performance of our estimators is considerably better than
the performance of the estimators from previous state-of-the-art methods
with or without assuming the order-modular prior.
%We only show the experimental results in computing the posteriors of edge features since their true values can be computed using the exact algorithms. We expect our algorithms will also be superior in computing non-modular features.
\end{comment}

We develop new algorithms for efficiently sampling Bayesian network structures (DAGs). The sampled DAGs can then be used to build estimators for the posteriors of any features of interests.
Theoretically we show that our estimators have several desirable properties.
For example, unlike the existing MCMC algorithms,
the estimators based on the DDS algorithm satisfy the Hoeffding bound
and therefore enjoy the quality guarantee of the estimation with the given number of samples.
Empirically we show that our estimators considerably outperform the previous state-of-the-art with or without assuming the order-modular prior.

Our algorithms are capable of estimating the posteriors of arbitrary (non-modular) features
(the DDS under the order-modular prior, the IW-DDS under the structure-modular prior);
while the exact algorithms are available for computing modular features under the order-modular prior
with time $O(n^{k+1} C(m) + k n 2^n)$ and space $O(n 2^n)$
%(Koivisto and Sood, 2004; Koivisto, 2006);
\citep{Koivisto:soo04,Koivisto:06};
computing path features under the order-modular prior with time $O(n^{k+1} C(m) + n 3^n)$ and space $O(3^n)$
%(Parviainen and Koivisto, 2011);
\citep{Parviainen:ECML2011};
and computing modular features under the structure-modular prior with time $O(n^{k+1} C(m) + kn2^n + 3^n)$ and space $O(n2^n)$
%(Tian and He, 2009).
\citep{tian:he2009}.
%The computational costs of our algorithms are dominated by their first step,
The bottleneck of our algorithms is their first computation step,
the DP algorithm
of \citet{Koivisto:soo04}
%with time $O(n^{k+1} C(m) + k n 2^n)$ and space $O(n 2^n)$.
whose space cost is $O(n 2^n)$.
Therefore the application of our algorithms is limited to the data sets on which the DP algorithm
of \citet{Koivisto:soo04}
is able to run -- up to around 25 variables in current desktops,
while a parallel implementation of the DP algorithm has been demonstrated on a data set with 33 variables
using a cluster including totally 2,048 processors and 8,192 GB memory
\citep{chen:etal14}.
%256 nodes
%16 CPU/ node  (actually 8 CPU/node)
%32G/node

%\newpage
\clearpage

\appendix{}

%\section{Supplementary Material}
%\section{SUPPLEMENTARY MATERIAL}
%\section{Appendix}

%\subsection{Proofs of Propositions, Theorems and Corollary}
%\subsection{PROOFS OF PROPOSITIONS, THEOREMS AND COROLLARY \label{sec-appendix-proofs}}
%\subsection{Proofs of Propositions, Theorems and Corollary \label{sec-appendix-proofs}}
\section{Proofs of Propositions, Theorems and Corollary \label{sec-appendix-proofs}}
This appendix provides the proofs of the propositions, theorems and corollary
%referenced
in the paper.

%\subsubsection{Proof of Proposition \ref{prop-1} }
\subsection{Proof of Proposition \ref{prop-1} \label{sec-appendix-proof-prop1}}

We first prove a lemma for Proposition 1.

Let an order $\prec$ be represented as $(\sigma_1, \ldots, \sigma_n)$ % $\sigma_1 < \ldots < \sigma_n$
where $\sigma_i$ is the $i$th element in the order.
\begin{lemma}\label{lm-1}

The probability that
%the last element is $\sigma_n$,
%the one before the last element %second to last element %
%is $\sigma_{n-1}$,
%and so on
the last $n-k+1$ elements along the order are $\sigma_k, \sigma_{k+1}, \ldots, \sigma_n$ respectively
is given by
%is as follows:
\begin{align*}
p(\sigma_k, \sigma_{k+1}, \ldots, \sigma_n, D) = L'(U^{\prec}_{\sigma_k})\prod_{i=k}^n\alpha'_{\sigma_i}(U^{\prec}_{\sigma_i}),
\end{align*}
where $U^{\prec}_{\sigma_i}= V-\{\sigma_i, \sigma_{i+1}, \ldots, \sigma_n\}$.
\end{lemma}
\begin{proof}
\begin{align*}
&   p(\sigma_k, \sigma_{k+1}, \ldots, \sigma_n, D) \nonumber \\
&=  \sum_{\prec'\in\mathcal{L}(U^{\prec}_{\sigma_k})} p(\prec',\sigma_k, \sigma_{k+1}, \ldots, \sigma_n, D)\\
&=  \prod_{i=k}^n\alpha'_{\sigma_i}(U^{\prec}_{\sigma_i})\sum_{\prec'\in\mathcal{L}(U^{\prec}_{\sigma_k})}\prod_{i\in U^{\prec}_{\sigma_k} } \alpha'_i(U^{\prec}_i) \quad (\textrm{from (\ref{eq-p_prec_D})})\\
&=  L'(U^{\prec}_{\sigma_k})\prod_{i=k}^n\alpha'_{\sigma_i}(U^{\prec}_{\sigma_i}).\quad (\textrm{from (\ref{eq-L_S})})
\end{align*}
\end{proof}

%The proof of Proposition 1 is a direct corollary of Lemma~\ref{lm-1} based on the definition of the conditional probability.
Proposition~\ref{prop-1} can be directly proved from the conclusion of Lemma~\ref{lm-1} according to the definition of conditional probability.

%\subsubsection{Proof of Proposition \ref{prop-2}  }
\subsection{Proof of Proposition \ref{prop-2}  }

\begin{proof}

From Proposition~\ref{prop-1} and $L'(U^{\prec}_{\sigma_1}=\emptyset)=1$, we have
\begin{align*}
& \prod_{k=1}^n p(\sigma_k| \sigma_{k+1}, \ldots, \sigma_n, D) \\
& = \frac{\prod_{i=1}^n\alpha'_{\sigma_i}(U^{\prec}_{\sigma_i})}{L'(V)} \\
& = \frac{p(\prec, D)}{p_{\prec}(D)} \quad (\textrm{from (\ref{eq-p_prec_D}) \& (\ref{eq-120330237}) }) \\
& = p(\prec | D),
\end{align*}
which proves Eq.~(\ref{eq-120330235}).  %from Eq.~(\ref{eq-p_prec_D}) and (\ref{eq-120330237}).
\end{proof}

%\subsubsection{Proof of Theorem \ref{thm-DDS}  }
\subsection{Proof of Theorem \ref{thm-DDS}  }

\begin{proof}

First, we show that each DAG $G$ sampled according to our DDS algorithm has its pmf $p_s(G)$ equal to
the exact posterior $p_{\prec}(G|D)$ by assuming the order-modular prior.

On one hand, from the following derivation, we can get the exact form for $p_{\prec}(G|D)$:
\begin{align}
&   p_{\prec}(f | D)    \nonumber \\
&=  \sum_{\prec} p(\prec | D) p(f | \prec, D)   \nonumber \\
%&=  \sum_{\prec} p(\prec | D) \sum_{G \in \mathcal{G}_{\prec}} f(G) p(G | \prec, D) \\
&=  \sum_{\prec} p(\prec | D) \sum_{G \subseteq  \prec} f(G) p(G | \prec, D) \nonumber \\
&=  \sum_{\prec} \sum_{G \subseteq  \prec} f(G) p(\prec, G | D) \nonumber \\
&=  \sum_{G} f(G) \sum_{\prec s.t. G \subseteq  \prec}  p(\prec, G | D).    \label{eq-prop-3-1}
%&=  \sum_{G} f(G) p(G | D)                                                  \label{eq-prop-3-2} \\
%&=  \textbf{E}_{G | D} [f(G)]
\end{align}
Thus,
for each possible DAG $G_i$,
by setting $f(G)$ to be the indication function $I[G = G_i]$
and
then relating Eq.~(\ref{eq-prop-3-1}) with Eq.~(\ref{eq-HR-pre1}),
we know that
$p_{\prec}(G_i|D) = \sum_{\prec s.t. G_i \subseteq  \prec}  p(\prec, G_i  | D)$ for each $G_i$.

On the other hand, the event that a DAG $G$ gets sampled according to our DDS algorithm occurs if and only if
one of the orders that $G$ is consistent with gets sampled in Step 2 of the DDS algorithm and
then $G$ gets sampled from the sampled order in Step 3 of the DDS algorithm.
Therefore, based on Total Probability Formula,
$p_s(G)$
$= \sum_{\prec s.t. G \subseteq  \prec}  p(\prec | D) p(G | \prec, D) $
$= \sum_{\prec s.t. G \subseteq  \prec}  p(\prec, G  | D) $
$= p_{\prec}(G|D)$.

Second, since $N_o$ orders are sampled independently and each DAG per sampled order is sampled independently,
$p_s(G_1, G_2, \ldots, G_{N_o} | D)$
$= \prod_{i=1}^{N_o} [ \sum_{\prec s.t. G \subseteq  \prec}  p(\prec | D) p(G_i | \prec, D)  ] $
$= \prod_{i=1}^{N_o} $  $p_{\prec}(G_i|D)$.

Therefore, Theorem \ref{thm-DDS} is proved.

\end{proof}

%\subsubsection{Proof of Corollary \ref{coro-1}}
\subsection{Proof of Corollary \ref{coro-1}}

\begin{proof}

For each $G_i$ in the DAG set $\{G_1, $  $G_2, \ldots, G_{N_o} \}$ sampled from the DDS algorithm,
$f(G_i) \in \{0, 1\}$.
Since  $G_1, $  $G_2, \ldots, G_{N_o} $
are $iid$ with pmf $p_{\prec}(G|D)$ (by Theorem \ref{thm-DDS}),
$f(G_1), $  $f(G_2), \ldots, $  $f(G_{N_o}) $ are $iid$ with Bernoulli pmf.

For each $G_i$, the following is true for $E (f(G_i))$, the expectation of $f(G_i)$.
\begin{align*}
%& E^{ p_{\prec}(G | D) } (f(G_i)) \\
& E  (f(G_i)) \\
& = \sum_{G} f(G) p_{\prec}(G | D) \\
& = p_{\prec}(f|D).
\end{align*}

Thus, $f(G_1), $  $f(G_2), \ldots, f(G_{N_o}) $ are $iid$ with Bernoulli($p_{\prec}(f|D)$).
In other words, $f(G_1), $  $f(G_2), \ldots, $ $f(G_{N_o}) $ are independent; and for each $G_i$,
$P(f(G_i) = 1) = p_{\prec}(f|D)$ and $P(f(G_i) = 0) = 1 - p_{\prec}(f|D)$.

(i) Proof that $\hat{p}_{\prec}(f|D)$ is an unbiased estimator for $p_{\prec}(f|D)$,
that is, $E ( \hat{p}_{\prec}(f|D) )$ $ = p_{\prec}(f|D)$.

\begin{align*}
& E \left( \hat{p}_{\prec}(f|D) \right) \\
& = E \left( \frac{1}{N_o} \sum_{i = 1}^{N_o}  f(G_i) \right) \\
& =   \frac{1}{N_o} \sum_{i = 1}^{N_o}  E \left( f(G_i) \right) \\
& =  \frac{1}{N_o} N_o  p_{\prec}(f|D) \\
& =  p_{\prec}(f|D).
\end{align*}

(ii) Proof that $\hat{p}_{\prec}(f|D)$ converges almost surely to $p_{\prec}(f|D)$.

Since
$f(G_1), $  $f(G_2), \ldots, f(G_{N_o}) $ are $iid$
with $E(f(G_i))$ %the expectation of each $f(G_i)$,  is
$= p_{\prec}(f|D) < \infty$,
and since
\begin{align*}
\hat{p}_{\prec}(f|D) = \frac{1}{N_o} \sum_{i = 1}^{N_o}  f(G_i),
\end{align*}
based on the Strong Law of Large Numbers (Theorem 5.5.9)
\citep{CB:StatisticalI},
$\hat{p}_{\prec}(f|D)$ converges almost surely to $p_{\prec}(f|D)$ (as $N_o \rightarrow \infty$).

Note that, by Theorem 2.5.1
\citep{athreya:measure2006},
the property that $\hat{p}_{\prec}(f|D)$ converges almost surely to $p_{\prec}(f|D)$
implies
that
$\hat{p}_{\prec}(f|D)$ converges in probability to $p_{\prec}(f|D)$, that is,
$\hat{p}_{\prec}(f|D)$ is a consistent estimator for $p_{\prec}(f|D)$.

(iii) Proof that
%if $0 < \hat{p}_{\prec}(f|D) < 1$,
if $0 < p_{\prec}(f|D) < 1$,
then the random variable
$$
\frac{\sqrt{N_o}(\hat{p}_{\prec}(f|D) - p_{\prec}(f|D))} { \sqrt{\hat{p}_{\prec}(f|D)(1 - \hat{p}_{\prec}(f|D))} }
$$
has a limiting standard normal distribution,
denoted by

$$
\frac{\sqrt{N_o}(\hat{p}_{\prec}(f|D) - p_{\prec}(f|D))} { \sqrt{\hat{p}_{\prec}(f|D)(1 - \hat{p}_{\prec}(f|D))} }
\longrightarrow^{d} \mathcal{N}(0, 1).
$$

Since $f(G_1), $  $f(G_2), \ldots, f(G_{N_o}) $ are $iid$ with Bernoulli($p_{\prec}(f|D)$),
for each $G_i$, the following is true for $Var(f(G_i))$, the variance of $f(G_i)$.
\begin{align*}
& Var(f(G_i)) \\
& =  E(f(G_i)) (1 - E(f(G_i)))\\
& = p_{\prec}(f|D) ( 1- p_{\prec}(f|D) ) \\
& < \infty.
\end{align*}

%If $0 < \hat{p}_{\prec}(f|D) < 1$, by Eq.~(\ref{eq-2-HR})
%there exists $G_i$ and $G_j$ in the DAG samples such that $f(G_i) = 1$
%and $f(G_j) = 0$.
%Thus, $ 0 < p_{\prec}(f|D) < 1$ so that $Var(f(G_i))  > 0$.
Since $ 0 < p_{\prec}(f|D) < 1$,
%$Var(f(G_i))  > 0$.
$Var(f(G_i))$ is also strictly greater than 0.

Again, we have already known that $E(f(G_i))$
$= p_{\prec}(f|D) < \infty$.

Thus, by the Central Limit Theorem (Theorem 5.5.15) \citep{CB:StatisticalI},

$$
\frac{\sqrt{N_o}(\hat{p}_{\prec}(f|D) - E(f(G)) )} { \sqrt{ Var(f(G))  } }
\longrightarrow^{d} \mathcal{N}(0, 1),
$$

that is,

$$
\frac{\sqrt{N_o}(\hat{p}_{\prec}(f|D) - p_{\prec}(f|D))} { \sqrt{p_{\prec}(f|D)(1 - p_{\prec}(f|D))} }
\longrightarrow^{d} \mathcal{N}(0, 1).
$$

Since      $\hat{p}_{\prec}(f|D)$ converges in probability to $p_{\prec}(f|D)$,
denoted by $\hat{p}_{\prec}(f|D)$ $\longrightarrow^{p}$       $p_{\prec}(f|D)$,
by the Continuous Mapping Theorem (Theorem 5.5.4) \citep{CB:StatisticalI},
$$
\sqrt{\hat{p}_{\prec}(f|D)(1 - \hat{p}_{\prec}(f|D))} \longrightarrow^{p}
\sqrt{p_{\prec}(f|D)(1 - p_{\prec}(f|D))}.
$$

Finally, by Slutsky's Theorem (Theorem 5.5.17) \citep{CB:StatisticalI},

$$
\frac{\sqrt{N_o}(\hat{p}_{\prec}(f|D) - p_{\prec}(f|D))} { \sqrt{\hat{p}_{\prec}(f|D)(1 - \hat{p}_{\prec}(f|D))} }
\longrightarrow^{d} \mathcal{N}(0, 1).
$$

(iv) %$\forall \epsilon > 0, \forall 0 < \delta < 1$,
Proof that for any $\epsilon > 0$, any $0 < \delta < 1$,
if $N_o \geq (\ln (2/\delta))/(2\epsilon^2)$,
then $p(|\hat{p}_{\prec}(f|D) -p_{\prec}(f|D)| < \epsilon)$ $ \geq 1 - \delta$.

Since $f(G_1), $  $f(G_2), \ldots, f(G_{N_o}) $ are $iid$ with Bernoulli($p_{\prec}(f|D)$),
the Hoeffding bound \citep{Hoeffding:1963,Koller:PGM} holds:
\begin{align*}
p(|\hat{p}_{\prec}(f|D) -p_{\prec}(f|D)| \geq \epsilon) \leq 2e^{-2 N_o \epsilon^2}.
\end{align*}

This is equivalent to

\begin{align*}
p(|\hat{p}_{\prec}(f|D) -p_{\prec}(f|D)| < \epsilon) \geq 1 - 2e^{-2 N_o \epsilon^2},
\end{align*}

and the conclusion is implied straightforward.

\end{proof}

%\subsubsection{Proof of Equation~(\ref{eq-relation_p-sigma_j-U_sigma_j})}
\subsection{Proof of Equation~(\ref{eq-relation_p-sigma_j-U_sigma_j})}
\begin{proof}
For any $j$ $\in \{1, \ldots, n\}$:
\begin{align*}
& P((\sigma_j, U_{\sigma_j}) | D) \\
& \propto P((\sigma_j, U_{\sigma_j}) , D) \\
& = \sum_{\substack{ (\sigma_1,     \ldots, \sigma_{j-1}) \\ \in \mathcal{L}(U_{\sigma_j}) }}
    \sum_{\substack{ (\sigma_{j+1}, \ldots, \sigma_n)     \\ \in \mathcal{L}(V - U_{\sigma_j} - \{\sigma_j\}) }}
    P(\sigma_1, \ldots, \sigma_{j-1}, \sigma_j, \sigma_{j+1}, \ldots, \sigma_n, D)\\
& = \sum_{\substack{ (\sigma_1,     \ldots, \sigma_{j-1}) \\ \in \mathcal{L}(U_{\sigma_j}) }}
    \sum_{\substack{ (\sigma_{j+1}, \ldots, \sigma_n)     \\ \in \mathcal{L}(V - U_{\sigma_j} - \{\sigma_j\}) }}
    \prod_{i=1}^n \alpha'_{\sigma_i}(U_{\sigma_i})\\
%& = \left[ \sum_{\substack{ {\prec}_1 = (\sigma_1,     \ldots, \sigma_{j-1}) \\ \in \mathcal{L}(U_{\sigma_j}) }}
%        \prod_{i=1}^{j-1} \alpha'_{\sigma_i}(U_{\sigma_i})
%    \right]
%    \left[ \sum_{ {\prec}_2 =\substack{ (\sigma_{j+1}, \ldots, \sigma_n)     \\ \in \mathcal{L}(V - U_{\sigma_j} - \{\sigma_j\}) }}
%        \prod_{i=j+1}^{n} \alpha'_{\sigma_i}( [V - (V - U_{\sigma_j} - \{\sigma_j\})] \cup U_{\sigma_i}^{{\prec}_2})
%    \right]
%    \alpha'_{\sigma_j}(U_{\sigma_j}) \\
& = \alpha'_{\sigma_j}(U_{\sigma_j}) L'(U_{\sigma_j}) R'(V - U_{\sigma_j} - \{ \sigma_j \}).
\end{align*}
\end{proof}

%\subsubsection{Proof of Equation~(\ref{eq-p_prec_vs_p_nprec})}
%\subsubsection{Proof that $p_{\prec}(G|D) \propto |\prec_{G}| \cdot  p_{\nprec}(G|D)$}
%\subsubsection{Proof of Equation~(\ref{eq-relation_p_prec_p_nprec})}
\subsection{Proof of Equation~(\ref{eq-relation_p_prec_p_nprec})}
\begin{proof}
\begin{align*}
& p_{\prec}(G|D)\\
& = \sum_{\prec s.t. G \subseteq  \prec}  p(\prec, G | D) \\
& = \frac{1}{p_{\prec}(D)} \sum_{\prec s.t. G \subseteq  \prec}  p(\prec, G, D) \\
& = \frac{1}{p_{\prec}(D)} \sum_{\prec s.t. G \subseteq  \prec}  p(\prec, G) p(D | G) \\
& = \frac{1}{p_{\prec}(D)} \sum_{\prec s.t. G \subseteq  \prec}  (\prod_{i=1}^n q_i(U_i) \rho_i(Pa_i)) p(D | G) \\
& = \frac{1}{p_{\prec}(D)} \sum_{\prec s.t. G \subseteq  \prec}  (\prod_{i=1}^n p_i(Pa_i)) p(D | G) \\
& = \frac{1}{p_{\prec}(D)} \sum_{\prec s.t. G \subseteq  \prec}  p(G) p(D | G) \\
& = \frac{1}{p_{\prec}(D)} \cdot |\prec_{G}| \cdot  p_{\nprec}(G,D) \\
& = \frac{p_{\nprec}(D)}{p_{\prec}(D)} \cdot |\prec_{G}| \cdot p_{\nprec}(G | D).
\end{align*}

Since
%$$
%\frac{p_{\nprec}(D)}{p_{\prec}(D)} > 0,
%$$
both $p_{\nprec}(D) > 0$  and $p_{\prec}(D) > 0$,
$p_{\prec}(G|D) \propto $  %\\
$|\prec_{G}| \cdot  p_{\nprec}(G|D)$, which also has been shown
%in \citep{ellis:won08}.
by \citet{ellis:won08}.

\end{proof}

%\subsubsection{Proof of Theorem~\ref{thm-IW-DDS}}
\subsection{Proof of Theorem~\ref{thm-IW-DDS}}
\begin{proof}

Let $\Omega$ denote the set of all the DAGs.
%$\mathcal{\Omega}^+$
Define $\mathcal{U}^+$ $ = $ $\{G \in \Omega: p_{\nprec}(G, D) > 0 \}$.
$\mathcal{U}^+$ will equal $\Omega$
if the user chooses a prior such that $p(G) > 0 $ for every DAG $G$ (such as a uniform DAG prior $p(G) \equiv 1$).
However, if the user has some additional domain knowledge
so that he/she sets some prior to exclude some DAGs a priori,
$\mathcal{U}^+$ will be a proper subset of $\Omega$.
Note that $p_{\nprec}(f|D)$ $ = p_{\nprec}(f, D) / p_{\nprec}(D) $,
where $p_{\nprec}(f, D)$ $ = \sum_{G} f(G) p_{\nprec}(G, D)$ $ = \sum_{G \in \mathcal{U}^+} f(G) p_{\nprec}(G, D) $,
and $p_{\nprec}(D)$ $ = p_{\nprec}(f \equiv 1, D)$ $ = \sum_{G} p_{\nprec}(G, D)$ $ = \sum_{G \in \mathcal{U}^+} p_{\nprec}(G, D) $.

Let $I[]$ denote the indicator function.
Rewrite $\hat{p}_{\nprec}(f|D)$ in Eq.~(\ref{eq-esti-f-Given-D}) as $\hat{p}_{\nprec}(f, D) / \hat{p}_{\nprec}(D) $,
where
%\begin{align}
%& \hat{p}_{\nprec}(f, D) \\
%& = \sum_{G \in \mathcal{G}} f(G) p_{\nprec}(G, D) \\
%& = \sum_{G} f(G) p_{\nprec}(G, D) I[G\in \mathcal{G}] \\
%& = \sum_{G \in \mathcal{U}^+} f(G) p_{\nprec}(G, D) I[G\in \mathcal{G}]
%\end{align}
$\hat{p}_{\nprec}(f, D) $
$= \sum_{G \in \mathcal{G}} f(G) p_{\nprec}(G, D)$
$= \sum_{G} f(G)$ $ p_{\nprec}(G, D) I[G\in \mathcal{G}] $
$= \sum_{G \in \mathcal{U}^+} f(G) p_{\nprec}(G, D) I[ $ $G \in \mathcal{G}] $,
and
%\begin{align}
%& \hat{p}_{\nprec}(D) \\
%& = \hat{p}_{\nprec}(f \equiv 1, D) \\
%& = \sum_{G \in \mathcal{U}^+ } p_{\nprec}(G, D) I[G\in \mathcal{G}]
%\end{align}
$\hat{p}_{\nprec}(D)$
$ = \hat{p}_{\nprec}(f \equiv 1, D)$
$= \sum_{G \in \mathcal{G}} p_{\nprec}(G, D)$
$= \sum_{G} p_{\nprec}(G, D) I[G \in \mathcal{G}] $
$= \sum_{G \in \mathcal{U}^+} $ $p_{\nprec}(G, D) I[G \in \mathcal{G}] $.

Note that by Theorem \ref{thm-DDS},
$P(G\in \mathcal{G})$ = $1 - (1-p_{\prec}(G|D))^{N_o}$,
where $p_{\prec}(G|D)$ is the exact posterior of $G$ under the order-modular prior assumption.
Also note that by Eq.~(\ref{eq-relation_p_prec_p_nprec}), $\forall G: (p_{\nprec}(G, D) > 0) \Rightarrow (p_{\prec}(G|D) > 0)$.

(i) Proof that $\hat{p}_{\nprec}(f|D)$ is an asymptotically unbiased estimator for $p_{\nprec}(f|D)$,
that is,
%$$ \lim_{N_o \to \infty} E(\hat{p}_{\nprec}(f|D)) = p_{\nprec}(f|D).$$
\begin{align}
\lim_{N_o \to \infty} E(\hat{p}_{\nprec}(f|D)) = p_{\nprec}(f|D). \label{eq-thm-2-proof-target}
\end{align}

For notational convenience,
let $\gamma$ denote $\hat{p}_{\nprec}(f, D)$,
and let $\tau$ denote $\hat{p}_{\nprec}(D)$.
Define $g(\gamma, \tau)$ $ = \gamma / \tau$ so that $g(\gamma, \tau)$ is essentially $\hat{p}_{\nprec}(f | D)$.

Note that
\begin{align*}
& E(\gamma) \\
& = E(\hat{p}_{\nprec}(f, D)) \\
& = E \left( \sum_{G \in \mathcal{U}^+} f(G) p_{\nprec}(G, D) I[G\in \mathcal{G}] \right) \\
& = \sum_{G \in \mathcal{U}^+} E( f(G) p_{\nprec}(G, D) I[G\in \mathcal{G}] ) \\
& = \sum_{G \in \mathcal{U}^+} f(G) p_{\nprec}(G, D) P(G\in \mathcal{G}) \\
& = \sum_{G \in \mathcal{U}^+} f(G) p_{\nprec}(G, D) (1 - (1-p_{\prec}(G|D))^{N_o}).
\end{align*}
Thus,
\begin{align*}
& \lim_{N_o \to \infty} E(\gamma) \\
%& \lim_{N_o \to \infty} E(\hat{p}_{\nprec}(f, D)) \\
& = \lim_{N_o \to \infty} \left( \sum_{G \in \mathcal{U}^+} f(G) p_{\nprec}(G, D) (1 - (1-p_{\prec}(G|D))^{N_o}) \right) \\
& =  \sum_{G \in \mathcal{U}^+} f(G) p_{\nprec}(G, D) \lim_{N_o \to \infty}(1 - (1-p_{\prec}(G|D))^{N_o}) \\
& =  \sum_{G \in \mathcal{U}^+} f(G) p_{\nprec}(G, D) \\
& = p_{\nprec}(f, D).
\end{align*}

Similarly, by setting $f \equiv 1$, we have
\begin{align*}
& E(\tau) \\
& = E(\hat{p}_{\nprec}(D)) \\
& = \sum_{G \in \mathcal{U}^+} p_{\nprec}(G, D) (1 - (1-p_{\prec}(G|D))^{N_o}),
\end{align*}
and
$$
\lim_{N_o \to \infty} E(\tau) = p_{\nprec}(D).
$$

%For the notational convenience,
%let $\gamma$ denote $\hat{p}_{\nprec}(f, D)$,
%and let $\tau$ denote $\hat{p}_{\nprec}(D)$.
%Define $g(\gamma, \tau)$ $ = \gamma / \tau$.

Next, by Taylor's Theorem (with the Lagrange form of the remainder),
\begin{align*}
& g(\gamma, \tau) \\
& = g(E(\gamma), E(\tau))
%\\
%& \;\;\;\; + \left[ \frac{\partial g(\gamma, \tau)}{\partial \gamma} \right]_{\gamma = E(\gamma), \tau = E(\tau)} (\gamma^* - E(\gamma)) \\
+ \left[ \frac{\partial g(\gamma, \tau)}{\partial \gamma} \right]_{\gamma = E(\gamma), \tau = E(\tau)} (\gamma^* - E(\gamma)) \\
& \;\;\;\; + \left[ \frac{\partial g(\gamma, \tau)}{\partial \tau}   \right]_{\gamma = E(\gamma), \tau = E(\tau)} (\tau^* - E(\tau)),
\end{align*}

where
$\gamma^*$ $= E(\gamma) + \theta (\gamma - E(\gamma))$,
$\tau^*$ $= E(\tau) + \theta (\tau - E(\tau))$,
and $\theta$ is a random variable
%with the restriction that
such that $0 < \theta < 1$.

Examine the components separately:
\begin{align*}
& g(E(\gamma), E(\tau))
%\\
%& = E(\gamma) / E(\tau).
= E(\gamma) / E(\tau).
\end{align*}

\begin{align*}
& \left[ \frac{\partial g(\gamma, \tau)}{\partial \gamma} \right]_{\gamma = E(\gamma), \tau = E(\tau)} \\
& = \left[ \tau^{-1} \right]_{\gamma = E(\gamma), \tau = E(\tau)} \\
& = (E(\tau))^{-1}. \\
%& = (E(\hat{p}_{\nprec}(D)))^{-1}.
%& = \left( \sum_{G \in \mathcal{U}^+} p_{\nprec}(G, D) (1 - (1-p_{\prec}(G|D))^{N_o}) \right)^{-1}.
\end{align*}
%May need to consider limit
%Since $(E(\tau))^{-1}$ $= (E(\hat{p}_{\nprec}(D)))^{-1}$ is a positive real number,
%$(E(\tau))^{-1}$ is bounded.

%Similarly,
\begin{align*}
& \left[ \frac{\partial g(\gamma, \tau)}{\partial \tau}   \right]_{\gamma = E(\gamma), \tau = E(\tau)} \\
& = \left[ - \gamma (\tau)^{-2} \right]_{\gamma = E(\gamma), \tau = E(\tau)} \\
& = - E(\gamma) (E(\tau))^{-2}. \\
%& = - E(\hat{p}_{\nprec}(f, D))(E(\hat{p}_{\nprec}(D)))^{-2}.
\end{align*}
%Since $- E(\gamma) (E(\tau))^{-2}$ is a non-positive real number,
%$- E(\gamma) (E(\tau))^{-2}$ is bounded.

Since neither $E(\gamma)$ nor $E(\tau)$ is random,
\begin{align*}
& E \left( g(\gamma, \tau) \right) \\
& = E(\gamma) / E(\tau)
%\\
%& \;\;\;\; +  (E(\tau))^{-1}  E(\gamma^* - E(\gamma))  \\
+  (E(\tau))^{-1}  E(\gamma^* - E(\gamma))
%\\
%& \;\;\;\; - E(\gamma) (E(\tau))^{-2}  E(\tau^* - E(\tau))\\
- E(\gamma) (E(\tau))^{-2}  E(\tau^* - E(\tau))\\
& = E(\gamma) / E(\tau)
%\\
%& \;\;\;\; +  (E(\tau))^{-1}  E(\theta (\gamma - E(\gamma)) )  \\
+  (E(\tau))^{-1}  E(\theta (\gamma - E(\gamma)) )
%& \;\;\;\; - E(\gamma) (E(\tau))^{-2}  E(\theta (\tau - E(\tau)) ).
- E(\gamma) (E(\tau))^{-2}  E(\theta (\tau - E(\tau)) ).
\end{align*}

Consider the limit of each component separately:
\begin{align*}
& \lim_{N_o \to \infty} \left( E(\gamma) / E(\tau) \right) \\
& = \left( \lim_{N_o \to \infty} E(\gamma) \right) / \left( \lim_{N_o \to \infty} E(\tau) \right) \\
& = p_{\nprec}(f, D) / p_{\nprec}(D) \\
& = p_{\nprec}(f | D).
\end{align*}

\begin{align*}
%& \lim_{N_o \to \infty} \left[ \frac{\partial g(\gamma, \tau)}{\partial \gamma} \right]_{\gamma = E(\gamma), \tau = E(\tau)} \\
& \lim_{N_o \to \infty} (E(\tau))^{-1} \\
%& = \lim_{N_o \to \infty} (E(\hat{p}_{\nprec}(D)))^{-1}\\
%& = \left( \lim_{N_o \to \infty}E(\hat{p}_{\nprec}(D)) \right)^{-1}\\
& = \left( \lim_{N_o \to \infty}E(\tau ) \right)^{-1}\\
& = (p_{\nprec}(D) )^{-1}.
\end{align*}
%Since $(p_{\nprec}(D) )^{-1}$ is a positive constant, %real number,
%$(p_{\nprec}(D) )^{-1}$ is bounded.

%Similarly,
\begin{align*}
%& \lim_{N_o \to \infty} \left[ \frac{\partial g(\gamma, \tau)}{\partial \tau}   \right]_{\gamma = E(\gamma), \tau = E(\tau)} \\
& \lim_{N_o \to \infty} - E(\gamma) (E(\tau))^{-2} \\
%& = - \lim_{N_o \to \infty} E(\hat{p}_{\nprec}(f, D))(E(\hat{p}_{\nprec}(D)))^{-2} \\
%& = - \left(\lim_{N_o \to \infty} E(\hat{p}_{\nprec}(f, D)) \right)  \left(\lim_{N_o \to \infty} (E(\hat{p}_{\nprec}(D)))^{-2} \right)\\
& = - \left(\lim_{N_o \to \infty} E(\gamma) \right)  \left(\lim_{N_o \to \infty} E(\tau) \right)^{-2}\\
& = - p_{\nprec}(f, D) (p_{\nprec}(D))^{-2}.
\end{align*}
%Since $- p_{\nprec}(f, D) (p_{\nprec}(D))^{-2}$ is a non-positive constant, %real number,
%$- p_{\nprec}(f, D) (p_{\nprec}(D))^{-2}$ is bounded.

Note that both $(p_{\nprec}(D) )^{-1}$ and $- p_{\nprec}(f, D) (p_{\nprec}(D))^{-2}$ are constant real numbers.

Finally, we intend to prove the following two equalities:
\begin{align}
\lim_{N_o \to \infty}  E(\theta (\gamma - E(\gamma)) ) = 0,  \label{eq-thm-2-main-proof1} \\
\lim_{N_o \to \infty}  E(\theta (\tau - E(\tau)) )     = 0.  \label{eq-thm-2-main-proof2}
\end{align}

Once this is done,
\begin{align*}
& \lim_{N_o \to \infty} E \left( g(\gamma, \tau) \right) \\
& = \lim_{N_o \to \infty} \left( E(\gamma) / E(\tau) \right)
%\\
%& \;\;\;\; +  \lim_{N_o \to \infty} (E(\tau))^{-1}  \cdot 0   \\
+  \lim_{N_o \to \infty} (E(\tau))^{-1}  \cdot 0
%\\
%& \;\;\;\; + \lim_{N_o \to \infty} \left( - E(\gamma) (E(\tau))^{-2} \right) \cdot 0 \\
+ \lim_{N_o \to \infty} \left( - E(\gamma) (E(\tau))^{-2} \right) \cdot 0 \\
& = p_{\nprec}(f | D).
\end{align*}

The whole proof for Eq.~(\ref{eq-thm-2-proof-target}) will then be done.

The proof for Eq.~(\ref{eq-thm-2-main-proof1}) is as follows.

Based on the definition of limit,
proving Eq.~(\ref{eq-thm-2-main-proof1})
is equivalent to proving

%\begin{align*}
%\forall \epsilon > 0: \exists N^* \in \mathcal{Z}^+ : \forall N_o \geq N^*:  \\
%\;\;\;\;\;\;\;\;\;\;  | E(\theta (\gamma - E(\gamma)) ) - 0 | < \epsilon.
%\end{align*}

%Since  $| E(\theta (\gamma - E(\gamma)) ) - 0 |$ $= | | E(\theta (\gamma - E(\gamma)) )| - 0 |$,
%by the definition of limit,
%to prove Eq.~(\ref{eq-thm-2-main-proof1}),
%it is also equivalent to proving

\begin{align}
\lim_{N_o \to \infty}  \left| E(\theta (\gamma - E(\gamma)) ) \right| = 0  \label{eq-thm-2-main-proof1-abs}.
\end{align}

Since
\begin{align*}
& \left| E(\theta (\gamma - E(\gamma)) ) \right| \\
& \leq E \left( | \theta (\gamma - E(\gamma)) | \right) \\
& = E \left( | \theta | \cdot | \gamma - E(\gamma) | \right)   \\
& \leq  E \left(  | \gamma - E(\gamma) | \right) \quad (\textrm{due to $0 < \theta < 1$  }),
\end{align*}

and $\left| E(\theta (\gamma - E(\gamma)) ) \right| \geq 0$,
to prove Eq.~(\ref{eq-thm-2-main-proof1-abs}), it is sufficient to prove
\begin{align}
\lim_{N_o \to \infty}  E \left(  | \gamma - E(\gamma) | \right)  = 0  \label{eq-thm-2-main-proof1-abs-del-theta}.
\end{align}

\begin{align*}
& E \left(  | \gamma - E(\gamma) | \right) \\
& = E \left( \left|  \sum_{G \in \mathcal{U}^+} f(G) p_{\nprec}(G, D) I[G\in \mathcal{G}]  \right. \right.
%\\
%& \;\;\;\;\;\;\;\;\;\;  \left. \left. - \sum_{G \in \mathcal{U}^+} f(G) p_{\nprec}(G, D) P(G\in \mathcal{G}) \right| \right)\\
\left. \left. - \sum_{G \in \mathcal{U}^+} f(G) p_{\nprec}(G, D) P(G\in \mathcal{G}) \right| \right)\\
& = E \left( \left|  \sum_{G \in \mathcal{U}^+} f(G) p_{\nprec}(G, D) ( I[G\in \mathcal{G}] - P(G\in \mathcal{G}) ) \right| \right)\\
& \leq E \left(   \sum_{G \in \mathcal{U}^+} f(G) p_{\nprec}(G, D) \left| I[G\in \mathcal{G}] - P(G\in \mathcal{G})  \right| \right) \\
& =   \sum_{G \in \mathcal{U}^+} f(G) p_{\nprec}(G, D) E \left(  \left|  I[G\in \mathcal{G}] - P(G\in \mathcal{G})  \right| \right) \\
& =   \sum_{G \in \mathcal{U}^+} f(G) p_{\nprec}(G, D) [  (1 - P(G\in \mathcal{G})) P(G\in \mathcal{G}) \\
& \;\;\;\;\;\;\;\;\;\;      +  (P(G\in \mathcal{G}) - 0) (1 - P(G\in \mathcal{G}))   ] \\
& =   \sum_{G \in \mathcal{U}^+} f(G) p_{\nprec}(G, D) [  2 P(G\in \mathcal{G}) (1 - P(G\in \mathcal{G})) ] \\
& =   \sum_{G \in \mathcal{U}^+} f(G) p_{\nprec}(G, D) [  2 (1 - (1-p_{\prec}(G|D))^{N_o})
%\\
%& \;\;\;\;\;\;\;\;\;\;  \times (1-p_{\prec}(G|D))^{N_o}  ].
\times (1-p_{\prec}(G|D))^{N_o}  ].
\end{align*}

Since $\forall G \in  \mathcal{U}^+: p_{\prec}(G|D) > 0$,
\begin{align*}
& \lim_{N_o \to \infty}   \sum_{G \in \mathcal{U}^+} f(G) p_{\nprec}(G, D) [  2 (1 - (1-p_{\prec}(G|D))^{N_o})
%\\
%& \;\;\;\;\;\;\;\;\;\;    \times (1-p_{\prec}(G|D))^{N_o}   ]  \\
\times (1-p_{\prec}(G|D))^{N_o}   ]  \\
& =  \sum_{G \in \mathcal{U}^+} f(G) p_{\nprec}(G, D) \lim_{N_o \to \infty} [  2 (1 - (1-p_{\prec}(G|D))^{N_o})
%\\
%& \;\;\;\;\;\;\;\;\;\;    \times (1-p_{\prec}(G|D))^{N_o}   ]  \\
\times (1-p_{\prec}(G|D))^{N_o}   ]  \\
& =  \sum_{G \in \mathcal{U}^+} f(G) p_{\nprec}(G, D) \cdot 0 \\
& = 0,
\end{align*}

Eq.~(\ref{eq-thm-2-main-proof1-abs-del-theta}) is proved, and the proof for Eq.~(\ref{eq-thm-2-main-proof1}) is done.

%By setting $f \equiv 1$, using the exactly same logic, we can also prove Eq.~(\ref{eq-thm-2-main-proof2}).
Setting $f \equiv 1$ in Eq.~(\ref{eq-thm-2-main-proof1}) leads to Eq.~(\ref{eq-thm-2-main-proof2}).

Thus, the whole proof for Eq.~(\ref{eq-thm-2-proof-target}) is done.

(ii) Proof that $\hat{p}_{\nprec}(f|D)$ converges almost surely to $p_{\nprec}(f|D)$,
denoted by $\hat{p}_{\nprec}(f|D) \longrightarrow^{a.s.} p_{\nprec}(f|D)$.

Remember that $\Omega$ denotes the set of all the DAGs,
that is, $\Omega$ $ = \{G_1, G_2, \ldots, G_{W^*} \}$,
where $W^*$ is $| \Omega |$, the number of all the DAGs.
Note that $W^*$ is a finite positive integer though it is super-exponential in the number of variables $n$.

Let $\mathcal{F}$ be $\mathcal{P}(\Omega)$, the power set of $\Omega$.
Thus, $\mathcal{F}$ is a $\sigma-$algebra on $\Omega$ \citep{athreya:measure2006}.
Define for any $A \in \mathcal{F}$,
$\mu(A) $
%$ = \sum_{j = 1}^W p_{\prec}(G_j|D) I[G_j \in A]$.
$ = \sum_{G_j \in A} p_{\prec}(G_j|D)$.
It is well-known that $\mu$ is a probability measure on $\mathcal{F}$
so that $(\Omega, \mathcal{F}, \mu)$ is a probability space \citep{athreya:measure2006}.

For each $i \geq 1$,
let $\Omega_i = \Omega$, $\mathcal{F}_i = \mathcal{F}$ and $\mu_i = \mu$.
Let $\Omega^{\infty} $ $ = \{ (G^{(1)}, G^{(2)}, \ldots): G^{(i)} \in \Omega_i, i \geq 1 \}$.
Let a cylinder set $\uwave{A} = $ $A_1 \times A_2 \times \ldots \times A_k \times \Omega_{k+1} \times \Omega_{k+2} \times \ldots$,
where there exists $1 \leq k < \infty$
such that $A_i \in \mathcal{F}_i$ for $1 \leq i \leq k$ and $A_i = \Omega_i$ for $i > k$.
Let $\mathcal{F}^{\infty}$ $ = \sigma< \{ \uwave{A}: \uwave{A} $ is a cylinder set $\} > $,
that is,  $\mathcal{F}^{\infty}$ is a $\sigma-$algebra generated by the set of all the $\uwave{A}$'s.
$\mathcal{F}^{\infty}$ is a a $\sigma-$algebra on $\Omega^{\infty} $
and is called a product $\sigma-$algebra.

Define, for each $(A_1, A_2, \ldots, A_k, \Omega_{k+1}, \Omega_{k+2}, \ldots) \in \mathcal{F}^{\infty}$,
$\mu^{\infty}(A_1, A_2, \ldots, A_k, \Omega_{k+1}, \Omega_{k+2}, $ $\ldots)$
$ = \mu_1(A_1) \times \mu_2(A_2) \times \ldots \times \mu_k(A_k)$.
By Kolmogorov's consistency theorem,
$(\Omega^{\infty}, \mathcal{F}^{\infty}, \mu^{\infty})$ is a probability space \citep{athreya:measure2006}.

Let
$\Omega^{\infty, 0}$
$ = \{ (G^{(1)}, G^{(2)}, \ldots) \in \Omega^{\infty} : \exists G \in \mathcal{U}^+: \forall i \geq 1:  G^{(i)} \neq G \}$.
Let $\Omega^{\infty, 1}$
$ = \Omega^{\infty} - \Omega^{\infty, 0} $.
Thus,
$\Omega^{\infty, 1}$
$ = \{ (G^{(1)}, G^{(2)}, \ldots) \in \Omega^{\infty} : \forall G \in \mathcal{U}^+: \exists i \geq 1:  G^{(i)} = G \}$.

Define, for each $\omega^{\infty} \in \Omega^{\infty}$,
\begin{align*}
\hat{p}_{\nprec}^{N_o}(\omega^{\infty})
= \frac{ \sum_{G \in \mathcal{U}^+} f(G) p_{\nprec}(G, D) I[G\in \mathcal{G}^{N_o}(\omega^{\infty})] }
{\sum_{G \in \mathcal{U}^+}  p_{\nprec}(G, D) I[G\in \mathcal{G}^{N_o}(\omega^{\infty})]},
\end{align*}
where $\mathcal{G}^{N_o}(\omega^{\infty})$ is the DAG set including the first $N_o$ coordinates %components
of $\omega^{\infty}$.
By the definition, we know that $\hat{p}_{\nprec}^{N_o}(\omega^{\infty})$ $= \hat{p}_{\nprec}(f|D)$.

For each $\omega^{\infty} \in \Omega^{\infty, 1}$,
for each $G \in \mathcal{U}^+$,
let $N(G, \omega^{\infty})$ be the smallest integer such that $G^{(N(G, \omega^{\infty}))} = G$.
Let $N(\omega^{\infty})$ $ = max_{G \in \mathcal{U}^+} N(G, \omega^{\infty})$.
Then for each $N_o \geq N(\omega^{\infty})$,
for each $G \in \mathcal{U}^+$,
$I[G\in \mathcal{G}^{N_o}(\omega^{\infty})] $ $ = 1$.
%Thus, $\lim_{N_o \to \infty} I[G\in \mathcal{G}^{N_o}(\omega^{\infty})]$ $ = 1$.

Accordingly,
for each $\omega^{\infty} \in \Omega^{\infty, 1}$,
there exists $N(\omega^{\infty})$ such that for each $N_o \geq  N(\omega^{\infty})$:
\begin{align*}
&  \hat{p}_{\nprec}^{N_o}(\omega^{\infty}) \\
& = \frac{ \sum_{G \in \mathcal{U}^+} f(G) p_{\nprec}(G, D)  }
{\sum_{G \in \mathcal{U}^+}  p_{\nprec}(G, D) } \\
& = p_{\nprec}(f|D).
\end{align*}
This implies that
$\lim_{N_o \to \infty} \hat{p}_{\nprec}^{N_o}(\omega^{\infty}) = p_{\nprec}(f|D)$ for each $\omega^{\infty} \in \Omega^{\infty, 1}$.

\begin{comment}
\begin{align*}
& \lim_{N_o \to \infty} \hat{p}_{\nprec}^{N_o}(\omega^{\infty}) \\
& = \frac{ \sum_{G \in \mathcal{U}^+} f(G) p_{\nprec}(G, D) \lim_{N_o \to \infty} I[G\in \mathcal{G}^{N_o}(\omega^{\infty})] }
{\sum_{G \in \mathcal{U}^+}  p_{\nprec}(G, D) \lim_{N_o \to \infty} I[G\in \mathcal{G}^{N_o}(\omega^{\infty})]} \\
& = \frac{ \sum_{G \in \mathcal{U}^+} f(G) p_{\nprec}(G, D)  }
{\sum_{G \in \mathcal{U}^+}  p_{\nprec}(G, D) } \\
& = \hat{p}_{\nprec}(f|D).
\end{align*}
\end{comment}

Finally, we intend to prove the following equality:
\begin{align}
\mu^{\infty}(\Omega^{\infty, 1}) = 1  \label{eq-thm-2-almost_surely-proof1}.
\end{align}
Once this is done, the whole proof for $\hat{p}_{\nprec}(f|D) $ $ \longrightarrow^{a.s.} p_{\nprec}(f|D)$ is done.

%The proof for Eq.~(\ref{eq-thm-2-almost_surely-proof1}) is as follows:
Proving Eq.~(\ref{eq-thm-2-almost_surely-proof1}) is equivalent to proving
\begin{align}
\mu^{\infty}(\Omega^{\infty, 0}) = 0  \label{eq-thm-2-almost_surely-proof2}.
\end{align}

For any $G \in \Omega$,
let $\Omega^{\infty, 0, G} $
$= \{ (G^{(1)}, G^{(2)}, \ldots) \in \Omega^{\infty} : \forall i \geq 1:  G^{(i)} \neq G \}$.
Thus,
%\begin{align*}
$\Omega^{\infty, 0}  $
$= \bigcup_{G \in \mathcal{U}^+} \Omega^{\infty, 0, G}$.
Accordingly, $\mu^{\infty}(\Omega^{\infty, 0}) $
$\leq \sum_{G \in \mathcal{U}^+} \mu^{\infty}(\Omega^{\infty, 0, G})$.
%\end{align*}

For each $j \geq 1$, let $\Omega^{\infty, 0, G, j}$
$= \{ (G^{(1)}, G^{(2)}, \ldots) \in \Omega^{\infty} : \forall i \in \{1, \ldots, j \}:  G^{(i)} \neq G \}$.
Note
$\Omega^{\infty, 0, G, 1}$
$\supseteq \Omega^{\infty, 0, G, 2}$
$\supseteq \ldots$
$\supseteq \Omega^{\infty, 0, G, j-1}$
$\supseteq \Omega^{\infty, 0, G, j}$ for each $j \geq 1$.
Thus, $\Omega^{\infty, 0, G}$
$= \bigcap_{j=1}^{\infty} \Omega^{\infty, 0, G, j}$.

Finally, for any $G \in \mathcal{U}^+$,
\begin{align*}
& \mu^{\infty}(\Omega^{\infty, 0, G}) \\
&= \mu^{\infty}(\bigcap_{j=1}^{\infty} \Omega^{\infty, 0, G, j}) \\
&= \lim_{j \to \infty} \mu^{\infty}(\Omega^{\infty, 0, G, j}) \\
&= \lim_{j \to \infty} \mu_1(\Omega - \{G \}) \times \mu_2(\Omega - \{G \}) \times \ldots \times \mu_j(\Omega - \{G \})\\
&= \lim_{j \to \infty} [1 - \mu(\{G \})]^j \\
&= \lim_{j \to \infty} [1 - p_{\prec}(G | D)]^j \\
&= 0.
\end{align*}

Thus,
$\sum_{G \in \mathcal{U}^+} \mu^{\infty}(\Omega^{\infty, 0, G}) = 0$,
so that
%$\mu^{\infty}(\Omega^{\infty, 0}) = 0$,
Eq.~(\ref{eq-thm-2-almost_surely-proof2}) is proved.
The whole proof is done.

Note that, by Theorem 2.5.1
\citep{athreya:measure2006},
the property that
$\hat{p}_{\nprec}(f|D)$ converges almost surely to $p_{\nprec}(f|D)$
implies
that
$\hat{p}_{\nprec}(f|D)$ converges in probability to $p_{\nprec}(f|D)$, that is,
$\hat{p}_{\nprec}(f|D)$ is a consistent estimator for $p_{\nprec}(f|D)$.

(iii) Proof that the convergence rate of $\hat{p}_{\nprec}(f|D)$ is $o(a^{N_o})$ for any $0 < a < 1$.

In the proof for (ii), we have shown that
for each $\omega^{\infty} \in \Omega^{\infty, 1}$, %where $\mu^{\infty}(\Omega^{\infty, 1}) = 1$,
there exists $N(\omega^{\infty})$ such that for each $N_o \geq N(\omega^{\infty})$,
$\hat{p}_{\nprec}^{N_o}(\omega^{\infty})  = p_{\nprec}(f|D)$.
This means that for any $0 < a < 1$,
$(a^{N_o})^{-1}  [\hat{p}_{\nprec}^{N_o}(\omega^{\infty})  - p_{\nprec}(f|D)] = 0$.
Thus, $ \lim_{N_o \to \infty} (a^{N_o})^{-1}  [\hat{p}_{\nprec}^{N_o}(\omega^{\infty})  - p_{\nprec}(f|D)] = 0$
so that the proof is done.

(iv) Proof that if the quantity
$\Delta  = \sum_{G\in \mathcal{G}} p_{\nprec}(G|D)$,
then
%\begin{align}\label{eq-p_nprec_f_given_D_interval}
$\Delta \cdot \hat{p}_{\nprec}(f|D) \leq p_{\nprec}(f|D) \leq \Delta \cdot \hat{p}_{\nprec}(f|D) + 1 - \Delta$.
%\end{align}

%Note that $p_{\nprec}(f|D)$ $ = \sum_{G} f(G) p_{\nprec}(G, D) / p_{\nprec}(D) $.
The proof is essentially the same as the proof for Proposition 1 of \citet{tian:he2010}
which proves Eq.~(\ref{eq-p_nprec_f_given_D_interval_equiv}), an equivalent form of Eq.~(\ref{eq-p_nprec_f_given_D_interval}).
The direct proof for Eq.~(\ref{eq-p_nprec_f_given_D_interval}) is also provided in the supplementary material.

\end{proof}

\clearpage

\bibliographystyle{hapalike}

\bibliography{./csl_cites_HR}

\end{document}